\documentclass{article}

\PassOptionsToPackage{numbers, compress}{natbib}

\usepackage[final]{neurips_2021}

\usepackage[utf8]{inputenc} 
\usepackage[T1]{fontenc}    
\usepackage[colorlinks,
            linkcolor=red,
            anchorcolor=blue,
            citecolor=blue
            ]{hyperref}
\usepackage{url}            
\usepackage{booktabs}       
\usepackage{amsfonts}       
\usepackage{nicefrac}       
\usepackage{microtype}      
\usepackage{xcolor}
\usepackage{subfig}

\usepackage{macros}
\usepackage[numbers]{natbib}
\def \la {\langle}
\def \ra {\rangle}

\renewcommand{\epsilon}{\varepsilon}
\def\rage{{\tt RAGE }}

\def\AE{{\tt ActionElim }}
\def\NE{{\tt NeuralEmbedding }}
\def\KE{{\tt KernelEmbedding }}
\def\LE{{\tt LinearEmbedding }}
\newcommand{\round}{{\tt ROUND }}

\title{Pure Exploration in Kernel and Neural Bandits}

\author{%
  Yinglun Zhu\thanks{Equal contribution} \\
  Department of Computer Sciences\\
  University of Wisconsin-Madison\\
  Madison, WI 53706 \\
  \texttt{yinglun@cs.wisc.edu} \\
  \And
  Dongruo Zhou$^*$ \\
  Department of Computer Science\\
  University of California, Los Angeles\\
  Los Angeles, CA 90095 \\
  \texttt{drzhou@cs.ucla.edu} \\
  \And
  Ruoxi Jiang$^*$ \\
  Department of Computer Science\\
  University of Chicago\\
  Chicago, IL 60637 \\
  \texttt{roxie62@uchicago.edu} \\
  \And
  Quanquan Gu \\
  Department of Computer Science\\
  University of California, Los Angeles\\
  Los Angeles, CA 90095 \\
  \texttt{qgu@cs.ucla.edu} \\
  \And
  Rebecca Willett \\
  Department of Statistics and Computer Science\\
  University of Chicago\\
  Chicago, IL 60637 \\
  \texttt{willett@uchicago.edu} \\
  \And
  Robert Nowak \\
  Department of Electrical and Computer Engineering\\
  University of Wisconsin-Madison\\
  Madison, WI 53706 \\
  \texttt{rdnowak@wisc.edu} \\
}

\begin{document}

\maketitle

\begin{abstract}
We study pure exploration in bandits, where the dimension of the feature representation can be much larger than the number of arms. To overcome the curse of dimensionality, we propose to adaptively embed the feature representation of each arm into a lower-dimensional space and carefully deal with the induced model misspecification. Our approach is conceptually very different from existing works that can either only handle low-dimensional linear bandits or passively deal with model misspecification. We showcase the application of our approach to two pure exploration settings that were previously under-studied: (1) the reward function belongs to a possibly infinite-dimensional Reproducing Kernel Hilbert Space, and (2) the reward function is nonlinear and can be approximated by neural networks. Our main results provide sample complexity guarantees that only depend on the effective dimension of the feature spaces in the kernel or neural representations. Extensive experiments conducted on both synthetic and real-world datasets demonstrate the efficacy of our methods.
\end{abstract}

\section{Introduction}

Pure exploration in bandits \citep{even2002pac, even2006action, bubeck2009pure} has been extensively studied in machine learning. Consider a set of arms, where each arm is associated with an unknown reward distribution. The goal is to approximately identify the optimal arm using as few samples as possible. Applications of bandit pure exploration range from medical domains \citep{aziz2019multi} to online content recommendation \citep{tanczos2017kl}. 

Despite the popularity of bandit pure exploration, it was previously mainly studied in two relatively restrictive settings: (1) the standard multi-armed bandit setting \citep{karnin2013almost, jamieson2014lil, chen2015optimal, jamieson2014best}, where the expected rewards among arms are completely unrelated to each other, and (2) the (generalized) linear bandit setting \citep{soare2014best, fiez2019sequential, degenne2020gamification, kazerouni2021best}, where the expected rewards are assumed to be linearly parameterized by some unknown weight vector. The standard multi-armed bandit setting fails to deal with large arm sets, and the linear bandit setting suffers from both model misspecification (due to its simplified linear form) and the curse of dimensionality in the high-dimensional setting. Pure exploration is also studied in continuous spaces. However, guarantees therein scale exponentially with dimension \citep{munos2014bandits, bartlett2019simple}.

In this paper, we generalize bandit pure exploration to the nonlinear and high-dimensional settings. More specifically, we study the following two settings: (1) the rewards of arms are parameterized by a function belonging to a Reproducing Kernel Hilbert Space (RKHS),
and (2) the rewards of arms are nonlinear functions that can be approximated by an overparameterized neural network. Problems in these two settings are often high-dimensional in nature. To overcome the curse of dimensionality, we propose to adaptively embed each arm's feature representation in a lower-dimensional space and carefully deal with the induced misspecification. Note that our approach is conceptually very different from all existing work dealing with model misspecification: they assume the existence of misspecification and address it in the original space (thus dealing with model misspecification in a passive way) \citep{lattimore2020learning, camilleri2021high}. On the other hand, we deliberately induce (acceptable) misspecification to embed arms into lower-dimensional spaces and thus overcome the curse of dimensionality.

\subsection{Contribution and Outline}

We make the following main contributions:
\begin{itemize}[leftmargin = *]
    \item In \cref{sec:active_compression}, we introduce the idea of adaptive embedding to avoid the curse of dimensionality. The induced model misspecification are carefully handled, which is novel in the bandit pure exploration setting. The sample complexity is theoretically analyzed and we relate the instance-dependent sample complexity to the complexity of a closely-related linear bandit problem without model misspecification. As a by-product, our algorithm can also be applied to constrained high-dimensional linear bandit pure exploration to reduce sample complexity. 
    \item In \cref{sec:kernel}, we specialize the adaptive embedding scheme to pure exploration in an RKHS. We construct feature mappings from eigenfunctions and eigenvalues of the associated kernel. The effective dimension of the kernel is analyzed, and we provide sample complexity guarantees in terms of the eigenvalue decay of the associated kernel. We rely on a \emph{known} kernel in this setting. 
    \item In \cref{sec:neural}, we further extend our adaptive embedding scheme to pure exploration with a general nonlinear reward function and model the reward function with an over-parameterized neural network. Sample complexity guarantees are provided with respect to the eigenvalue decay of the associated Neural Tangent Kernel. 
    To the best of our knowledge, this provides the first theoretically founded pure exploration algorithm with a neural network approximation.

    \item In \cref{sec:experiment}, we conduct extensive experiments on both synthetic and real-world datesets to confirm the efficacy of our proposed algorithms. 
    We conclude our paper in \cref{sec:conclusion} with open problems. 
    \end{itemize}

\subsection{Related Work}
The bandit pure exploration problem has a long history, dating back to the seminal work by \citet{bechhofer1958sequential, paulson1964sequential}. One classical objective of pure exploration is  Best Arm Identification (BAI), where the goal is to identify the best arm using as few samples as possible \citep{karnin2013almost, jamieson2014lil, chen2015optimal, garivier2016optimal}. To make it applicable to a large action space, the BAI problem is also extensively studied as the good arm identification problem, where the goal is to identify an $\epsilon$-optimal arm \citep{even2002pac, even2006action, kalyanakrishnan2012pac, kano2019good, sabato2019epsilon, katz2020true, mason2020finding}.

The pure exploration problem in linear bandits is initially analyzed in \citet{soare2014best}, where optimal experimental design \citep{kiefer1960equivalence} is applied to guide the allocation of samples. Other approaches dealing with linear bandits, with various sample complexity guarantees, include adaptive sampling \citep{xu2018fully} and an approach called track-and-stop \citep{degenne2020gamification}. Constrained linear bandit pure exploration is also commonly studied with additional assumptions on the reward parameters \citep{tao2018best, degenne2020gamification}. We note that the track-and-stop approach only achieves optimal instance-dependent sample complexity in the regime where the confidence parameter approaches $0$, but fails to do so in the moderate confidence regime. \citet{fiez2019sequential} propose an elimination-based algorithm (with optimal design) that achieves (nearly) instance-dependent sample complexity. such algorithm is further generalized to the combinatorial bandit setting \citep{katz2020empirical}. 

Learning with model misspecification was recently introduced in bandit learning, with the primary emphasis placed on the regret minimization problem \citep{ghosh2017misspecified, lattimore2020learning, foster2020adapting}. A very recent independent work studies pure exploration in kernel bandits with misspecification \citep{camilleri2021high}; both their and our algorithms follow the framework of \rage \citep{fiez2019sequential} and draw inspiration from \citep{lattimore2020learning}. \citet{camilleri2021high} propose a robust estimator that works in high-dimensional spaces and also explore the project-then-round idea through \emph{regularized least squares}. Our algorithms adaptively embed actions into lower dimensional spaces according to some error tolerances (different embeddings from round to round); our rounding and elimination steps are thus computed only with respect to lower-dimensional embeddings. We additionally study the pure exploration problem with an overparameterized neural network. As mentioned before, our approach is also conceptually different from existing ones: rather than passively dealing with model misspecification in its original representation, we deliberately and adaptively embed arms into a lower-dimensional space to avoid the curse of dimensionality; the induced model misspecification are also carefully dealt with in our algorithms.

\section{Problem setting}
\label{sec:preliminary}

We introduce the general setting and notations for pure exploration in bandits. Consider a set of arms $\cX \subseteq \R^D$ where the number of arms $\abs*{\cX} = K$ is possibly very large. We use an unknown function $h: \cX \rightarrow [-1,1]$ to represent the true reward of each arm. A noisy feedback $h(\bx) + \xi$ is observed after each sample arm $\bx$, where the noise $\xi$ is assumed to be $1$-sub-Gaussian. The learner is allowed to allocate her samples based on previously collected information, and the goal is to approximately identify an approximately optimal arm using as few samples as possible. Let $\bx_\star = \argmax_{\bx \in \cX} h(\bx)$ denote the optimal arm among $\cX$. We aim at developing $(\epsilon, \delta)$-PAC guarantees: for any $\delta \in (0, 1)$, with probability at least $1 - \delta$, the algorithm outputs an $\epsilon$-optimal arm $\widehat \bx$ such that $h(\widehat \bx) \geq h(\bx_\star) - \epsilon$ using a finite number of samples. The performance of the algorithm is measured by its sample complexity, i.e., the number of samples pulled before it stops and recommends a candidate arm.

\textbf{Notations.} We define $\Delta_{\bx} = h(\bx_\star) - h(\bx)$ as the sub-optimality gap of arm $\bx$. We use the notations $\cS_k := \curly*{\bx \in \cX: \Delta_{\bx} < 4 \cdot 2^{-k}}$ (with $\cS_1 = \cX$). We consider feature mappings of the form $\bpsi_d(\cdot): \cX \rightarrow \R^{d}$, and define $\bpsi_d(\cX) = \curly*{\bpsi_d(\bx): \bx \in \cX}$. 
We use 
$\bLambda_{\cX} = \curly*{\lambda \in \R^{\abs{\cX}}: \sum_{\bx \in \cX} \lambda_{\bx} = 1, \lambda_{\bx} \geq 0 }$ to denote the $(\abs{\cX} -1)$-dimensional probability simplex over arms in $\cX$; and set $\bA_{\bpsi_d}(\lambda) = \sum_{\bx \in \cX} \lambda_{\bx} \bpsi_{d}(\bx) \bpsi_d (\bx)^\top$.\footnote{A generalized inversion is used for singular matrices. We refer to \cref{app:rounding} for detailed discussion.} We use $\norm{\bx}_{\bA} = \sqrt{\bx^\top \bA \bx}$ to represent the Mahalanobis norm. We also define $\cY(\cV) = \curly*{\bv - \bv^\prime: \bv, \bv^\prime \in \cV}$ for any set $\cV$. For a matrix $\Hb \in \RR^{|\cX| \times |\cX|}$, we use $\Hb(\xb, \xb')$ to denote the entry of $\Hb$ which locates at row $\xb$ and column $\xb'$.

\section{Bandit pure exploration with adaptive embedding}
\label{sec:active_compression}

We introduce the idea of bandit pure exploration with adaptive embedding, which can be viewed as an approach that actively trades off sample complexity with accuracy guarantees: we adaptively embed the feature representation into lower-dimensional spaces to avoid the curse of dimensionality, and conduct pure exploration with \emph{misspecified} linear bandits. The embedding dimensions are carefully selected so that we can identify an $\epsilon$-optimal arm.

We formalize the idea as follows. For any $d \in \N$, we assume the existence of a feature mapping $\bpsi_d: \cX \rightarrow \R^d$ and a unknown reward vector $\btheta_d \in \R^d$ such that, for any $\bx \in \cX$,
\begin{align*}
    h(\bx) = \ang*{\bpsi_d(\bx), \btheta_d} + \eta_d(\bx),
\end{align*}
where $\eta_d(\bx)$ represents the induced approximation error on arm $\bx$ with respect to the low-dimensional embedding $\bpsi_d(\cdot)$. Without loss of generality, we assume that the action set $\cX$ is rich enough so that $\bpsi_d(\cX)$ spans $\R^d$ for $d$ considered in this paper. Otherwise, one can always project feature representations $\bpsi_d(\cX)$ into an even lower-dimensional space without losing information in the linear component.

We use $\widetilde \gamma: \N \rightarrow \R$ to represent the misspecification level: an upper bound of the induced approximation error across all arms, i.e., $\max_{\bx \in \cX} \abs*{\eta_d(\bx)} \leq \widetilde \gamma (d)$. 
We define $g(d, \zeta) \coloneqq (1 + \zeta) \, \inf_{\lambda \in \bLambda_{\mathcal{X}}} \sup_{\by \in \cY( \bpsi_d(\mathcal{X}))} \norm*{\by}^2_{\bA_{\bpsi_d}(\lambda)^{-1}}$, which represents the optimal value of a transductive design among embeddings in $\R^d$. We define $\gamma(d) \coloneqq  \paren*{16 + 8\sqrt{ g(d , \zeta)}} \, \widetilde  \gamma(d)$, which quantifies the sub-optimality gap of the identified arm in the worst case. One can easily show that $\gamma(d) = O(\widetilde \gamma (d)\sqrt{d} )$ through Kiefer-Wolfowitz theorem \citep{kiefer1960equivalence}. 

\begin{remark}
We believe such optimality guarantees are un-improvable in general. In fact, a hard instance is constructed in \cite{lattimore2020learning} showing that, even with deterministic feedback, identifying a $o(\widetilde \gamma (d)\sqrt{d})$-optimal arm requires sample complexity exponential on $d$. On the other side, identifying a $\Omega(\widetilde \gamma (d )\sqrt{d})$-optimal only requires sample complexity polynomially in $d$. Such a sharp trade-off between optimality and sample complexity motivates our definition of $\gamma(d)$ (and our sample complexities are polynomially in $d$).
\end{remark}

We assume the knowledge of both the feature mapping $\bpsi_d(\cdot)$ and the error function $\widetilde \gamma (\cdot)$. This assumption is mild since one can explicitly construct/analyze $\bpsi_d(\cdot)$ and $\widetilde \gamma (\cdot)$ in many cases (as discussed in \cref{sec:high_dim_linear}, \cref{sec:kernel} and \cref{sec:neural}). We further assume that $\gamma(d)$ can be made arbitrarily small for large enough $d$. Such assumption holds true if the reward can be explained by the linear component for $d$ large enough, i.e., $\widetilde \gamma(d) = 0$. We now define the \emph{effective dimension} with respect to $\gamma(d)$ (induced from feature mapping $\bpsi_d(\cdot)$) as follows.
\begin{definition}
\label{def:d_eff}
For any $\epsilon > 0$, we define the effective dimension as $d_{\eff}(\epsilon) \coloneqq \min \curly*{ d \geq 1: \gamma(d) \leq \epsilon }$.
\end{definition}
In general, the effective dimension $d_{\eff}(\epsilon)$ captures the smallest dimension one needs to explore in order to identify an $\epsilon$-optimal arm. Similar notions have been previously used in regret minimization settings \citep{valko2013finite, valko2014spectral}. One can easily see that $d_{\eff}(\epsilon_1) \leq d_{\eff}(\epsilon_2)$ as long as $\epsilon_1 \geq \epsilon_2$.

\subsection{Algorithm and analysis}
\label{sec:alg_analysis}

\cref{alg:active_elim_adaptive} follows the framework of {\tt RAGE} \citep{fiez2019sequential} to eliminate arms with sub-optimality gap $\geq O(2^{-k})$ at the $k$-th iteration. It runs for $n = O(\log(1/\epsilon))$ iterations and identifies an $\epsilon$-optimal arm. We use optimal experimental design to select arms for faster elimination. For any fixed design $\lambda \in \bLambda_{\cX}$, with $N \geq r_d(\zeta)$ samples and an approximation factor $\zeta$ (with default value $\zeta \in [1/10, 1/4]$), the rounding procedure in $\R^d$, i.e., $\round(\lambda, N, d, \zeta)$, outputs a discrete allocation $\curly{\bx_1, \bx_2, \dots, \bx_N}$ satisfying
\begin{align}
    \max_{\by \in \cY(\bpsi_d(\cX))} \norm{\by}^2_{\paren{\sum_{i=1}^N \bpsi_d(\bx_i) \bpsi_d(\bx_i)^{\top}}^{-1}} \leq (1+\zeta) \max_{\by \in \cY(\bpsi_d(\cX))} \norm{\by}^2_{ \bA_{\bpsi_d}(\lambda)^{-1} } / N. \label{eq:rounding}
\end{align}
Efficient rounding procedures exist with $r_d(\zeta) = \frac{d^2 + d+2}{\zeta}$ \citep{pukelsheim2006optimal} or $r_d(\zeta) = \frac{180d}{\zeta^2}$ \citep{allen2020near, fiez2019sequential}. We refer reads to \citep{fiez2019sequential, pukelsheim2006optimal, allen2020near} for detailed rounding algorithms and the associated computational complexities.

\begin{algorithm}[]
	\caption{Arm Elimination with Adaptive Embedding and Induced Misspecification}
	\label{alg:active_elim_adaptive} 
	\renewcommand{\algorithmicrequire}{\textbf{Input:}}
	\renewcommand{\algorithmicensure}{\textbf{Output:}}
	\begin{algorithmic}[1]
		\REQUIRE Action set $\cX$, confidence parameter $\delta$, accuracy parameter $\epsilon$ and rounding approximation factor $\zeta$.
		\STATE Set $n = \ceil*{\log_2(2/\epsilon)}$ and $\widehat \cS_1 = \cX$.
		\FOR {$k = 1, 2, \dots, n$}
		\STATE Set $\delta_k = \delta/k^2$, $d_k = d_{\eff}(4\cdot 2^{-k})$. 
		\STATE Select feature representation $\bpsi_{d_k}(\cdot)$, and calculate the induced misspecification level $\widetilde \gamma(d_k)$. Set $r_{d_k}(\zeta) = O(d_k / \zeta^2)$ as the number of samples needed for \round in $\R^{d_k}$.
		\STATE Set $\lambda_k$ and $\tau_k$ be the design and the value of the following optimization problem 
		\vspace{-6 pt}
		$$\inf_{\lambda \in \bLambda_{\cX}} \sup_{\by \in \cY( \bpsi_{d_k}(\widehat \cS_k))} \norm*{\by}^2_{\bA_{\bpsi_{d_k}}(\lambda)^{-1}}.$$
		\vspace{-8 pt}
		\STATE Set $\epsilon_k = 2 \widetilde \gamma(d_k) + \widetilde \gamma(d_k) \sqrt{(1+\zeta) \, \tau_k }$,\\
		and $ N_k = \max \curly*{ \ceil*{ (2^{-k} - \epsilon_k)^{-2} 2 (1+\zeta)\,\tau_k  \log(\abs*{\widehat \cS_k}^2/\delta_k) }  ,r_{d_k}(\zeta)}.$
		\STATE Get $\curly{\bx_1, \bx_2, \dots, \bx_{N_k}} = \round (\lambda_k,N_k,d_k, \zeta)$.
		\STATE Pull arms $\curly{\bx_1, \bx_2, \dots, \bx_{N_k}}$ and receive rewards $\curly*{y_1, \ldots, y_{N_k}}$.
        \STATE Set $\widehat{\btheta}_k = \bA_k^{-1} \bb_k$, where $\bA_k = \sum_{i=1}^{N_k} \bpsi_{d_k}(\bx_i) \bpsi_{d_k}(\bx_i)^{\top}$ and $\bb_k = \sum_{i=1}^{N_k} \bpsi_{d_k}(\bx_i) y_i$.
        \STATE 
        Eliminate arms with respect to criteria 
        \vspace{-4 pt}
        $$\widehat \cS_{k+1} = \widehat \cS_k \setminus \{ \bx \in \widehat \cS_k : \exists \bx^\prime \text{ such that }(\bpsi_{d_k}(\bx^\prime) -\bpsi_{d_k}(\bx))^\top \widehat{\theta}_k \geq \omega_k(\bpsi_{d_k}(\bx^\prime) -\bpsi_{d_k}(\bx))\},$$
        \vspace{-4 pt}
        where $\omega_k(\by) = \epsilon_k + \norm{\by}_{\bA_k^{-1}} \sqrt{2 \log \paren*{ { \abs*{\widehat \cS_k}^2}/{\delta_k}}}$.
		\ENDFOR 
		\ENSURE Output any arm in $\widehat \cS_{n+1}$.
	\end{algorithmic}
\end{algorithm}

Unlike \rage that directly works in the original high-dimensional space, \cref{alg:active_elim_adaptive} adaptively embeds arms into lower-dimensional spaces and carefully deals with the induced misspecification. More specifically, the embedding dimension $d_k$ is selected as the smallest dimension such that the induced error term $\epsilon_k$ is well controlled, i.e., $\epsilon_k \leq O(2^{-k})$. The embedding is more aggressive at initial iterations due to larger error tolerance; The embedding dimension selected at the last iteration is (roughly) $d_{\eff}(\epsilon)$ to identify an $\epsilon$-optimal arm. The number of samples required for each iteration $N_k$ is with respect to an experimental design in the lower-dimensional space \emph{after embedding}. The \round procedure also becomes more efficient due to the embedding. Before stating our main theorem, we introduce the following complexity measure \citep{soare2014best, fiez2019sequential, degenne2020gamification}, which quantifies the hardness of the pure exploration problem (with respect to mapping $\bpsi_d(\cdot)$).
\begin{align*}
    \rho_d^\star(\epsilon) \coloneqq \inf_{\lambda \in \bLambda_{\cX}} \sup_{ \bx \in \cX \setminus \curly*{ \bx_\star }} \frac{\norm*{\bpsi_d(\bx_{\star})-\bpsi_d(\bx)}^2_{\bA_{\bpsi_d}(\lambda)^{-1}}}{ \max \curly{ h(\bx_\star) - h(\bx), \epsilon}^2}.
\end{align*}

\begin{restatable}{theorem}{thmActiveElimAdaptive}
\label{thm:active_elim_adaptive}
With probability of at least $1-\delta$, \cref{alg:active_elim_adaptive} correctly outputs an $\epsilon$-optimal arm with sample complexity upper bounded by 
\begin{align*}
    640 \sum_{k=1}^{\ceil*{\log_2(2/\epsilon)}} \paren{ \paren{ k \, \rho_{d_k}^\star(2^{2-k})  \log(k^2 \abs{\cX}^2/\delta)}  + \paren{r_{d_k}(\zeta)+1} }= \widetilde{O} \paren{ d_{\eff}(\epsilon) \cdot \max \curly*{\Delta_{\min}, \epsilon}^{-2} } ,
\end{align*}
where $d_k = d_{\eff}(4\cdot 2^{-k}) \leq d_{\eff}(\epsilon)$ since $4 \cdot 2^{-k} \geq \epsilon$ when $k \leq \ceil*{\log_2(2/\epsilon)}$.
\end{restatable}

The rounding term $r_d(\zeta)$ commonly appears in the sample complexity of linear bandits \citep{fiez2019sequential, katz2020empirical}; and our rounding term is with respect to the lower-dimensional space after embedding, which only scales with $d_k$ rather than the ambient dimension. To further interpret the complexity, we define another complexity measure of a closely related linear bandit problem in the low-dimensional space and without model misspecification. 
\begin{align*}
    \widetilde \rho_d^\star(\epsilon) \coloneqq \inf_{\lambda \in \bLambda_{\cX}} \sup_{ \bx \in \cX \setminus \curly*{ \bx_\star }} \frac{\norm*{\bpsi_d(\bx_{\star})-\bpsi_d(\bx)}^2_{\bA_{\bpsi_d}(\lambda)^{-1}}}{ \max \curly{ \ang{\bpsi_d(\bx_\star) - \bpsi_d(\bx), \btheta_d}, \epsilon}^2},
\end{align*}
where $\ang{\bpsi_d(\bx_\star) - \bpsi_d(\bx), \btheta_d}$ on the denominator represents the sub-optimality gap characterized by the linear component rather than the true sub-optimality gap $h(\bx_\star) - h(\bx)$. The relation between $\rho^\star(\epsilon)$ and $\widetilde \rho^\star(\epsilon)$ is discussed as follows.
\begin{restatable}{proposition}{propComplexityLinear}
\label{prop:complexity_linear}
Suppose $\max_{\bx \in \cX} \abs*{h(\bx) - \ang*{\bpsi_d(\bx), \btheta_d}} \leq \widetilde \gamma(d)$. For any $\epsilon \geq \widetilde \gamma(d)$, we have $\rho_d^\star(\epsilon) \leq 9\widetilde \rho_d^\star(\epsilon)$. When $\widetilde \gamma(d)  < \Delta_{\min}/2$, $\widetilde \rho_d^\star(0)$ represents the sample complexity of a closely-related linear bandit problem without model misspecification, i.e., $\widetilde \bh(\bx) = \ang*{\bpsi_d(\bx), \btheta_d}$.
\end{restatable}

\begin{remark}
When $\widetilde \gamma(d)  < \Delta_{\min}/2$, our sample complexity upper bound is relevant to the sample complexity of closely-related linear bandit problems without model misspecification in lower-dimensional spaces. In fact, $\widetilde \rho_d^\star(0) \log(1/2.4\delta)$ is the lower bound of the corresponding linear bandit problem in $\R^d$ \citep{soare2014best, fiez2019sequential, degenne2020gamification}.
\end{remark}

\begin{remark}
Although the misspecification levels are generally known for situations considered in this paper, we also provide an algorithm that deals with unknown misspecification levels in \cref{app:unknown_misspecification}. Similar sample complexity guarantees are provided, but only in an unverifiably way (due to unknown misspecification levels): the algorithm starts to output $\epsilon$-good arms after $N$ samples, yet it doesn't know when to stop. We refer readers to \cite{katz2020true} for details on the unverifiable sample complexity.
\end{remark}

\subsection{Application to high-dimensional linear bandits}
\label{sec:high_dim_linear}
We apply the idea of adaptive embedding to high-dimensional linear bandits. We consider linear bandit problem of the form $\bh = \bX \btheta_\star$ where $\bX \in \R^{K \times D}$ and the $i$-th row of $\bX$ represents the feature vector of arm $\bx_i$. We assume that $\norm{\btheta_\star}_{2} \leq C$, which is commonly studied as the constrained linear bandit problem \citep{tao2018best, degenne2020gamification}.

Let $\bX = \bU \bSigma \bV^\top$ be the singular value decomposition (SVD) of $\bX$, with singular values $\sigma_1 \geq \sigma_2 \geq \dots \geq \sigma_{r}>0$ for some $r \leq \min \curly*{K,D}$. Let $u_{i,j}$ denote the $(i,j)$-th entry of matrix $\bU$ and $\bu_{:,i}$ denote the $i$-th column of $\bU$ (similar notations for $\bV$). We have 
\begin{align*}
    \bh = \bX \btheta_\star = \bU \bSigma \bV^\top \btheta_\star = \sum_{i=1}^d \sigma_i \bu_{:,i} \bv_{:,i}^\top \btheta_\star + \sum_{i=d+1}^D \sigma_i \bu_{:,i} \bv_{:,i}^\top \btheta_\star =: \sum_{i=1}^d \sigma_i \bu_{:,i} \bv_{:,i}^\top \btheta_\star + \bmeta,
\end{align*}
where $\norm{\bmeta}_{\infty} \leq C \sum_{i=d+1}^D \sigma_i$. As a result, for any $d \leq r$, we can construct the feature mapping $\bpsi_d(\bx_i) = \sq{\sigma_1 u_{i,1}, \dots, \sigma_{d} u_{i,d} }^\top \in \R^{d}$ such that $h(\bx_i) = \ang*{\bpsi_d(\bx_i), \widetilde \btheta_\star} + \bmeta(\bx_i)$, where $\widetilde \btheta_\star = \sq{\bV^\top \btheta}_{[1:d]} \in \R^d$ is the associated reward parameter.\footnote{We note that the embeddings and associated quantities can also be constructed on the fly with respect to the set of uneliminated arms.} The upper bound of the induced misspecification can be expressed as $\widetilde \gamma(d) = C \sum_{i=d+1}^D \sigma_i$, which allows us to calculate $\gamma(d)$. We can then apply \cref{alg:active_elim_adaptive} to identify an $\epsilon$-optimal arm. A high-dimensional linear bandit instance is provided in \cref{app:high_dim_linear} showing that: \cref{alg:active_elim_adaptive} takes $\widetilde O(1/\epsilon^2)$ samples to identify an $\epsilon$-optimal arm, while the sample complexity upper bound of \rage scales as $\widetilde O (D/\epsilon^2)$.

\section{Pure exploration in RKHS}
\label{sec:kernel}

We consider a kernel function $\cK: \cZ \times \cZ \rightarrow \R$ over a compact set $\cZ$; we assume the kernel function satisfies condition stated in the Mercer's Theorem (see \cref{app:mercer}) and has eigenvalues decay fast enough (see \cref{asm:kernel_eigen_decay}). Let $\cH$ be the Reproducing Kernel Hilbert Space (RKHS) induced from $\cK$. We assume $\cX \subseteq \cZ$ and the true reward of any arm $\bx \in \cX$ is given by an unknown function $h \in \cH$ such that $\norm{h}_{\cH} \leq 1$.

Let $\curly*{\phi_j}_{j=1}^\infty$ and $\curly*{\mu_j}_{j=1}^\infty$ be sequences of eigenfunctions and non-negative eigenvalues associated with kernel $\cK$.\footnote{With a known kernel, the sequence of eigenfunctions and eigenvalues can be analytically calculated or numeriaclly approximated \citep{schlesinger1957approximating, santin2016approximation}. We assume the knowledge of eigenfunctions and eigenvalues in this paper.} A corollary of Mercer's theorem shows that any $h \in \cH$ can be written in the form of $h(\cdot) = \sum_{j=1}^\infty \theta_j \phi_j(\cdot)$ for some $\curly*{\theta_j}_{j=1}^\infty \in \ell^2(\N)$ such that $\sum_{j=1}^\infty \theta_j^2/\mu_j < \infty$. We also have $\norm{h}_{\cH}^2 = \sum_{j=1}^\infty \theta_j^2/\mu_j$. Although functions in RKHS are non-linear in nature, we now can represent them in terms of an infinite-dimensional linear function. We construct feature mappings for the embedding next.

For any $\bx \in \cX$, we have $h(\bx) = \sum_{j=1}^\infty \theta_j \phi_j(\bx) = \sum_{j=1}^\infty \frac{\theta_j}{\sqrt{\mu_j}} \sqrt{\mu_j} \phi_j(\bx)$. Let $C_\phi \coloneqq \sup_{\bx \in \widetilde \cX, j \geq 1} \abs{ \phi_j(\bx)}$. Since $\sum_{j=1}^\infty \theta_j^2/\mu_j = \norm{h}_{\cH}^2 \leq 1$ is bounded, for any $d \in \N$, we define feature mapping $\bpsi_d(\bx) = [\sqrt{\mu_1} \phi_1(\bx), \dots, \sqrt{\mu_d} \phi_d(\bx)]^\top \in \R^d$ such that 
\begin{align*}
    h(\bx) = \ang{\btheta_d , \bpsi_d(\bx)} + \eta_d(\bx),
\end{align*}
where $\btheta_d = \sq*{ \theta_1/\sqrt{\mu_1}, \dots, \theta_d/\sqrt{\mu_d} }^\top \in \R^d$ and $\abs{\eta_d(\bx)} \leq \widetilde \gamma(d) \coloneqq C_{\phi}  \sqrt{\sum_{j > d} \mu_j}$. We remark here that the constant $C_\phi$ is calculable and usually mild, e.g., $C_\phi =1$ for $\phi_j(x) = \sin \paren{(2j-1)\pi x / 2}$.

We can then construct $\gamma(d)$ and $d_{\eff}(\epsilon)$ as in \cref{sec:active_compression} and specialize \cref{alg:active_elim_adaptive} to the kernel setting. Both $\gamma(d)$ and $d_{\eff}(\epsilon)$ depend on eigenvalues of the associated kernel. Fortunately, fast eigenvalue decay are satisfied by most kernel functions, e.g., Gaussian kernel. We quantify such properties through the following assumption.

\begin{assumption}
\label{asm:kernel_eigen_decay}
We consider kernels with the following eigenvalue decay with some absolute constants $C_k$ and $\beta$.
\begin{enumerate}[leftmargin = *]
    \item Kernel $\cK$ is said to have $(C_k, \beta)$-polynomial eigenvalue decay (with $\beta > 3/2$) if $\mu_j \leq C_k j^{-\beta}$ for all $j \geq 1$.
    \item Kernel $\cK$ is said to have $(C_k, \beta)$-exponential eigenvalue decay (with $\beta > 0$) if $\mu_j \leq C_k e^{-\beta j}$ for all $j \geq 1$.
\end{enumerate}
\end{assumption}

\begin{restatable}{theorem}{thmKernelElim}
\label{thm:kernel_elim}
Suppose \cref{asm:kernel_eigen_decay} holds. For any $\epsilon >0$, the following statements hold when we specialize \cref{alg:active_elim_adaptive} to the kernel setting.
\begin{enumerate}[leftmargin = *]
    \item Suppose $\cK$ has $(C_k, \beta)$-polynomial eigenvalue decay. We have $d_{\eff}(\epsilon) = O \paren*{{\epsilon}^{-2/(2\beta -3)}}$, and the sample complexity of identifying an $\epsilon$-optimal arm is upper bounded by $\widetilde O ( \epsilon^{-2/(2\beta -3)} \max \curly*{\Delta_{\min}, \epsilon}^{-2} )$. 
    \item Suppose $\cK$ has $(C_k, \beta)$-exponential eigenvalue decay. We have $d_{\eff}(\epsilon) = O \paren*{\log(1/\epsilon)}$, and the sample complexity of identifying an $\epsilon$-optimal arm is upper bounded by $\widetilde O (\max \curly*{\Delta_{\min}, \epsilon}^{-2})$.
\end{enumerate}
\end{restatable}

\begin{remark}
Our sample complexity guarantees are directly related to the eigenvalue decay of the underlying kernel function, rather than the empirical kernel matrix as studied in previous works \citep{camilleri2021high, valko2013finite}. 
Although one can also provide an instance dependent bound as in \cref{thm:active_elim_adaptive}, the worst-case sample complexity bound in \cref{thm:kernel_elim} provides insightful characterizations of the sample complexity in terms of eigenvalue decay. One should notice that with exponential eigenvalue decay, the sample complexity $\widetilde O (\epsilon^{-2})$ essentially matches, up to logarithmic factors, the complexity of distinguishing a two-armed bandit up to accuracy $\epsilon$ \citep{kaufmann2016complexity}. 
\end{remark}

\section{Pure exploration with neural networks}
\label{sec:neural}

In this section we present a neural network-based pure exploration algorithm in Algorithm \ref{alg:neural}. Our algorithm is inspired by the recently proposed neural bandits algorithms for regret minimization \citep{zhou2019neural, zhang2020neural}. At the core of our algorithm is to use a neural network $f(\xb; \btheta)$ to learn the unknown reward function $h$. Specifically, following \citep{cao2019generalization2, zhou2019neural}, we consider a fully connected neural network $f(\xb; \btheta)$ with depth $L \geq 2$

\begin{align}
    &f(\xb; \btheta) = \sqrt{m}\Wb_L \sigma\Big(\Wb_{L-1}\sigma\big(\cdots \sigma (\Wb_1\xb)\big)\Big),\label{def:network}
\end{align}

where $\sigma(x): = \max(x,0)$ is the ReLU activation function, $\Wb_1 \in \RR^{m \times d}, \Wb_L \in \RR^{1 \times m}$, and for $2 \leq i\leq L-1$, $\Wb_i \in \RR^{m \times m}$. Moreover, we denote $\btheta = [\text{vec}(\Wb_1)^\top,\dots,\text{vec}(\Wb_L)^\top]^\top \in \RR^{p}$, where $p = m+md+m^2 (L-2)$ is the number of all the network parameters. We use $\gb(\xb; \btheta) = \nabla_{\btheta}f(\bx;\btheta)$ to denote the gradient of the neural network output with respect to the weights.

\begin{algorithm}[]
	\caption{Neural Arm Elimination}
	\label{alg:neural} 
	\renewcommand{\algorithmicrequire}{\textbf{Input:}}
	\renewcommand{\algorithmicensure}{\textbf{Output:}}
	\begin{algorithmic}[1]
		\REQUIRE Action set $\cX$, initial parameter $\btheta_0$, neural network $f(\xb; \btheta)$, gradient mapping $\gb(\xb, \btheta)$, width of the matrix $m$, parameter of the number of allocations $A$, approximation parameter $\zeta$, regularization parameter $\alpha$, error parameter $\bar \epsilon, \epsilon$, confidence level $\delta_k = \delta/(8k^2)$
		\STATE Set $\widehat \cS_1 = \cX$.
		\FOR {$k = 1, 2, \dots, n$ }
		\STATE Construct the truncated feature representation $\bpsi_k(\cX)$ based on gradient mapping $\gb(\xb; \btheta_{k-1})$. In detail, let $\Gb \in \RR^{|\cX| \times p}$ be the collection of gradients such that
		\vspace{-6 pt}
		\begin{align}
		    \Gb = [\gb(\xb_1; \btheta_{k-1})^\top;\dots; \gb(\xb_{|\cX|}; \btheta_{k-1})^\top]/\sqrt{m} \in \RR^{|\cX| \times p}
		\end{align}
		Let $[\Ub, \bSigma, \Vb]$ be the SVD of $\Gb$, where $\Ub  = (u_{i,j})\in \RR^{|\cX| \times |\cX|}$, $\bSigma = [\text{diag}(e_1, \dots, e_{|\cX|}), 0] \in \RR^{|\cX|\times p}$, $\Vb \in \RR^{p \times p}$. Get $d_k = \min \curly*{d \in [\abs*{\cX}] :\sum_{i=d+1}^{|\cX|} e_i\leq \bar \epsilon}$, and set $\bpsi_{d_k}(\xb_i) = (e_1 u_{i,1},\dots, e_{d_k} u_{i,d_k}) \in \RR^{d_k}$.
		\STATE Set $\lambda_k$ and $\tau_k$ be the experimental design and the value of the following optimization problem 
		\vspace{-6 pt}
		\begin{align}
		    \inf_{\lambda \in \bLambda_{\cX}} \sup_{\yb \in \cY( \bpsi_{d_k}(\widehat \cS_k))} \norm*{\yb}^2_{\bA_{\bpsi_{d_k}}(\lambda)^{-1}}.
		\end{align}
		\vspace{-8 pt}
		\STATE Set $N_k = \max\big\{2^{2k}\,A (1+\zeta)\log \paren{ { |\cX|^2}/{\delta_k}}, r_{d_k}(\zeta)\}$.
		\STATE Get $\curly{\xb_1, \xb_2, \dots, \xb_{N_k}} = \round (\lambda_k,N_k,d_k,\zeta)$. 
		\STATE Pull arms $\curly{\xb_1, \xb_2, \dots, \xb_{N_k}}$ and receive rewards $\curly*{y_1, \ldots, y_{N_k}}$.
        \STATE Using $J_k$ step $\eta_k$-step size gradient descent to optimize the following loss function to obtain $\btheta_k$,  
        \vspace{-15 pt}
		\begin{align}
		    \btheta_k = \argmin L(\btheta): = \sum_{j=1}^{N_k} (f(\xb_j; \btheta) - y_j)^2 + \frac{m\alpha }{2}\|\btheta - \btheta_0\|_2^2.
		\end{align}
		\vspace{-8 pt}
        \STATE Set $\Ab_k = \alpha \Ib + \sum_{i=1}^{N_k}\bpsi_{d_k}(\xb)\bpsi_{d_k}(\xb)^\top$ and eliminate arms with respect to criteria 
        \vspace{-6 pt}
        \begin{align*}
            \widehat \cS_{k+1} &= \widehat \cS_k \setminus \{ \xb \in \widehat \cS_k : \exists \bx^\prime \text{ such that }f(\xb'; \btheta_k)  - f(\xb; \btheta_k) \geq 2^{-k}/8 + 3\epsilon/8\}.
        \end{align*}
        \vspace{-20 pt}
		\ENDFOR 
		\ENSURE Output any arm in $\widehat \cS_{n+1}$.
	\end{algorithmic}
\end{algorithm}

In detail, at $k$-th iteration, Algorithm \ref{alg:neural} firstly applies its current gradient mapping $\gb(\xb; \btheta_{k-1})$ over the whole action set $\cX$, and obtains the collection of gradients $\Gb\in \RR^{|\cX|\times p}$. Then Algorithm \ref{alg:neural} does SVD over $\Gb $ and constructs a $d_k$-dimensional feature mapping $\bpsi_{d_k}$, which can be regarded as the projection of the gradient feature mapping $\gb(\xb; \btheta_{k-1})$ to the most informative $d_k$-dimensional eigenspace. Here we choose $d_k$ such that the summation of the eigenvalues of the remaining eigenspace be upper bounded by some error $\bar \epsilon$.  
Algorithm \ref{alg:neural} then computes the optimal design $\lambda_k$ over $\bpsi_k(\cX)$ and pulls arms $\curly*{ \xb_1, \dots, \xb_{N_k}}$ based on both the design $\lambda_k$ and the total number of allocations $N_k$. Finally, Algorithm \ref{alg:neural} trains a new neural network $f(\xb; \btheta_k)$ using gradient descent starting from the initial parameter $\btheta_0$ (details are deferred to Appendix~\ref{app:neural}), then eliminates the arms $\xb$ in the current arm set $\widehat\cS_k$ which are sub-optimal with respect to the neural network function value $f(\xb; \btheta_k)$.

The main difference between \cref{alg:neural} and its RKHS counterpart is as follows. Unlike Algorithm \ref{alg:active_elim_adaptive} which works on known feature mappings $\bpsi_d$ (derived from a known kernel $\cK$), \cref{alg:neural} does not have information about the feature mapping, and thus it constructs the feature mapping from the raw high-dimensional up-to-date gradient mapping $\gb(\xb; \btheta_{k-1})$. The feature mapping is constructed with respect to a \emph{trained} neural work, which leverages the great representation power of neural networks. This makes \cref{alg:neural} a more general and flexible algorithm than \cref{alg:active_elim_adaptive}.

Now we present the main theorem of Algorithm \ref{alg:neural}. Let $\Hb^{|\cX| \times |\cX|}$ be the Neural Tangent Kernel (NTK)\citep{jacot2018neural} gram matrix over all arms $\cX$ (the detailed definition of $\Hb$ is deferred to Appendix~\ref{app:neural}). 
We define the effective dimension for the neural version as below. The definition is similar to \cref{def:d_eff}.
\begin{definition}[Neural version]
For any $\epsilon > 0$, we define the effective dimension as $d_{\eff}(\epsilon):= \min \curly*{d \in [\abs*{\cX}]: \sum_{i=d+1}^{|\cX|} \lambda_i(\Hb) \leq \epsilon }$, where $\lambda_i(\Hb)$ is the $i$-th eigenvalue of $\Hb$.
\label{def:d_eff_neural}
\end{definition}

Next, we make standard assumptions for the initialization of neural networks and the arms $\xb \in \cX$.  
\begin{assumption}[\citep{zhou2019neural}]\label{assumption:context}
There exists $\lambda_0>0$ such that $\Hb \succeq \lambda_0 \Ib$. For any $\xb \in \cX$, the arm $\xb$ satisfies $\|\xb\|_2 = 1$ and that its $j$-th coordinate is identical to its $j+d/2$-th coordinate. Meanwhile, the initial parameter $\btheta_0 = [\text{vec}(\Wb_1)^\top,\dots,\text{vec}(\Wb_L)^\top]^\top$ is initialized as follows: for $1 \leq l\leq L-1$, $\Wb_l$ is set to $\begin{pmatrix}
    \Wb & 0 \\
    0 & \Wb
    \end{pmatrix}$,
where each entry of $\Wb$ is generated independently from $\cN(0, 4/m)$; $\Wb_L$ is set to $(\wb^\top, -\wb^\top)$, where each entry of $\wb$ is generated independently from $\cN(0, 2/m)$.
\end{assumption}
We now present our main theorem for pure exploration with neural network approximation. The formal version of our theorem is deferred to Appendix~\ref{app:neural}. 

\begin{theorem}[Informal]\label{thm:neural_inf}
Under Assumption \ref{assumption:context}, with proper selection of parameters $\alpha, n, \bar\epsilon, A, \eta_k, J_k$, 
then when $m = \text{poly}(|\cX|, L, \lambda_0^{-1}, \log(|\cX|/\delta_k), N_k, \alpha, \bar \epsilon^{-1})$, with probability at least $1-\delta$, $\widehat{\cS}_{K+1}$ only includes arm $\xb$ satisfying $\Delta_{\xb} \leq \epsilon$, and the total sample complexity of Algorithm \ref{alg:neural} is bounded by 
\begin{align}
    N = \widetilde O\bigg((1+\zeta)d_{\eff}(\bar \epsilon^2/|\cX|)/\epsilon^2 + r_{d_{\eff}(\bar \epsilon^2/|\cX|)}(\zeta) \bigg) = \widetilde O\Big(d_{\eff}(\bar \epsilon^2/|\cX|) \epsilon^{-2}\Big).\notag
\end{align}
\end{theorem}
\begin{remark}
For the case where the effective dimension can be well bounded, e.g., $d_{\eff}(\bar \epsilon^2/|\cX|) = O(\log(|\cX|/\bar \epsilon^2))$, \cref{thm:neural} suggests that Algorithm \ref{alg:neural} is able to identify an $\epsilon$-optimal arm within $\tilde O(\epsilon^{-2})$ samples. That suggests that our neural network-based algorithm is efficient without constructing a low-dimensional linear approximation of $h$ in prior, like the previous linear or RKHS approaches.
\end{remark}

\section{Experiments}
\label{sec:experiment}

We conduct four experiments on synthetic and real-world datasets. We specialize our embedding idea to the neural, kernel, and linear regimes, and denote the algorithms as \NE (\cref{alg:neural}),
\KE (\cref{alg:active_elim_adaptive} with Gaussian kernel), and {\tt LinearEmbedding} (\cref{alg:active_elim_adaptive} with linear representation), respectively. We compare our algorithms with two baselines: \rage and {\tt ActionElim}. \rage\citep{fiez2019sequential} conducts pure exploration in linear bandits and \AE\citep{jamieson2014best, even2002pac} ignores all feature representations. The (empirical) sample complexity of each algorithm is calculated as the number of samples needed so that the uneliminated set contains only $\epsilon$-optimal arms. Unsuccessful runs, i.e., those terminate with non-$\epsilon$-optimal arms, are reported as failures. 
In our experiments, we set $\epsilon=0.1$ and $\delta=0.05$. 
All results are averaged over 50 runs.\footnote{All algorithms are elimination-styled for fair comparison. Other implementation details are deferred to \cref{app:experiment}.}

\begin{figure}[h]
    \centering
    \subfloat[][]{\includegraphics[width=0.5\linewidth]{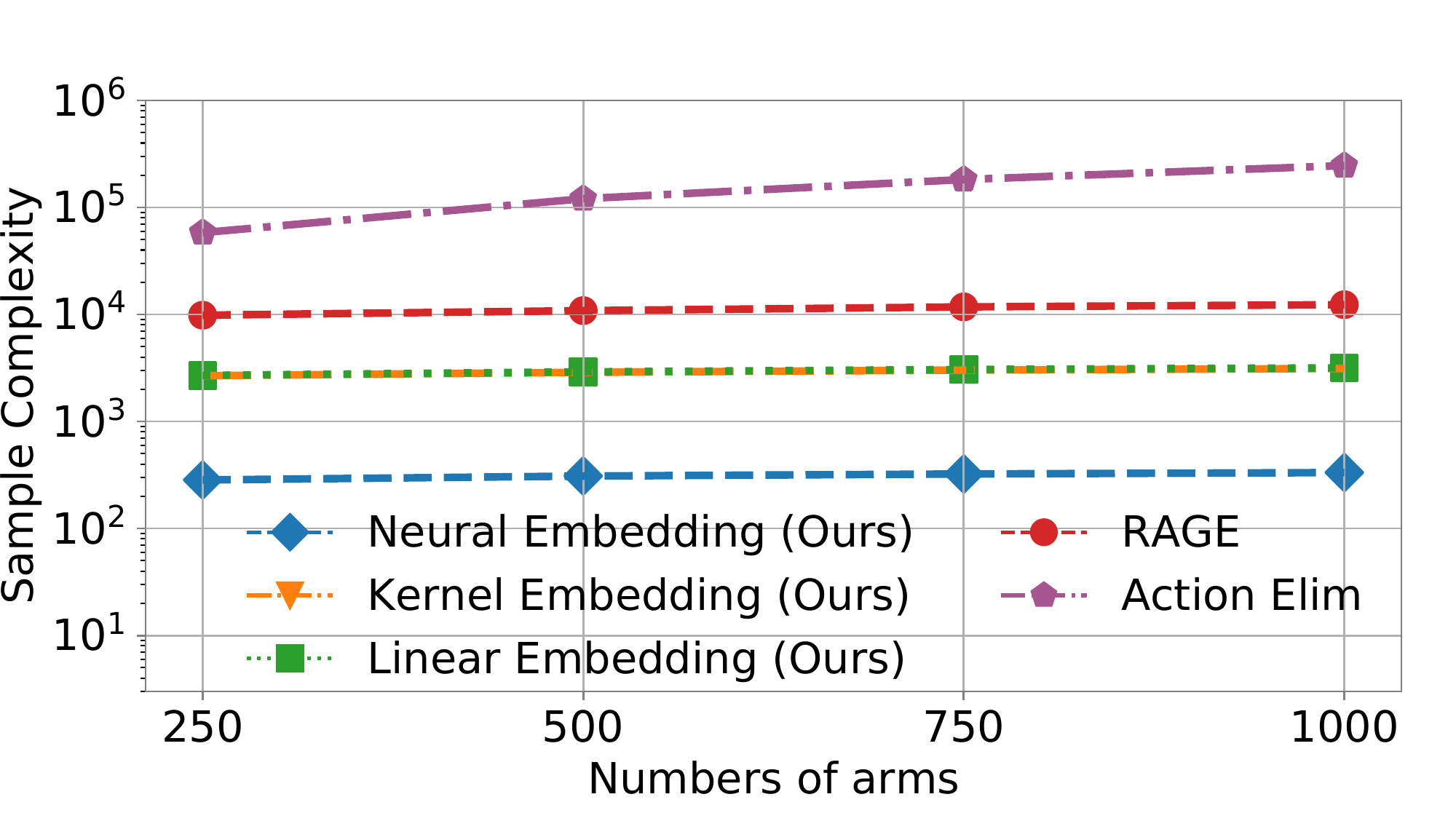}
    \label{fig:Linear}}
    \subfloat[][]{\includegraphics[width=0.5\linewidth]{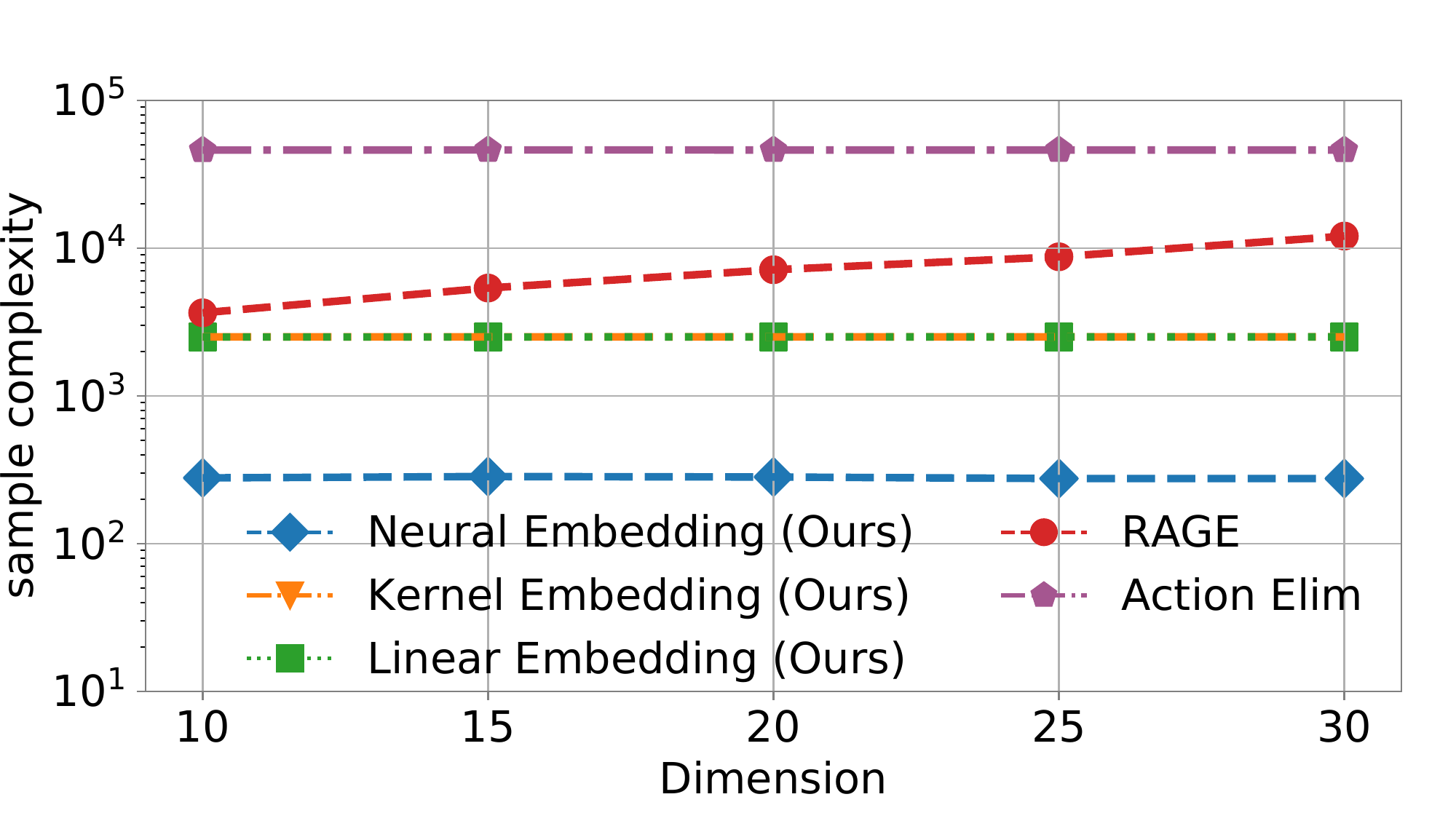}
    \label{fig:Nonlinear}}\\
    \caption{Experiments on synthetic datasets. (a) linear reward function with $D=20$; (b) nonlinear reward function with $K=200$.  }
    \label{fig:synthetic}
\end{figure}

\textbf{Synthetic datasets.} We first generate the feature matrix $\widetilde \bX = \bX + \bE \in \R^{K \times D}$ where $\bX$ is constructed as a rank-2 matrix 
and $\bE$ is a perturbation matrix with tiny spectral norm (See \cref{app:experiment} for details). Each row of $\widetilde \bX$ represents the feature representation of an arm, and those features can be grouped into two clusters with equal size. In the case with linear rewards, we let $\btheta_\star$ equal to the first row of $\widetilde \bX$. For nonlinear rewards, the reward of each arm is set as the 2-norm of its feature representation. We vary the number of arms $K$ and the ambient dimension $D$ in our experiments.

\cref{fig:synthetic} shows experimental results on synthetic datasets. All algorithms successfully identify $\epsilon$-optimal arms with zero empirical failure probability (due to the simplicity of the datasets). In terms of sample complexity, \NE outperforms all other algorithms in most cases, and \KE and \LE significantly outperform \rage and {\tt Action Elim}. The sample complexities of {\tt NeuralEmbedding}, \KE and \LE are not affected when increasing number of arms or dimensions since they first identify the important subspace and then conduct elimination. On the other side, the sample complexity of \AE gets larger with increasing number of arms and the sample complexity of \rage gets larger with increasing dimensions.

\textbf{MNIST dataset.} The MNIST dataset \cite{lecun1998mnist} contains hand-written digits from 0 to 9. We view each digit as an arm, and set their rewards according to the confusion matrix of a trained classifier.
Digit $7$ is chosen as the optimal arm with reward $1$; the reward of digits $1, 2$ and $9$ are set to be $0.8$, and all other digits have reward $0.5$. 
In each experiment, we randomly draw $200$ samples ($20$ samples each digit) from the dataset. 
We project the raw feature matrix $\bX \in \R^{200 \times 784}$ into a lower-dimensional space $\widetilde{\bX} \in \R^{200 \times 200}$ so that it becomes full rank (but without losing any information) and feasible for {\tt RAGE}.
Our goal is to correctly identify a digit $7$. 

\textbf{Yahoo dataset.} The Yahoo! User Click Log Dataset R6A\footnote{\url{https://webscope.sandbox.yahoo.com}} contains users' click-through records. Each record consists of a $36$-dimensional feature representation (obtained from an outer product), and a binary outcome stating whether or not a user clicked on the article. We view each record as an arm, and set the reward as $0.8$ (if clicked) or $0.3$ (if not clicked) to makes the problem harder.
In each experiment, we randomly draw $200$ arms from the dataset, where $5$ of them having rewards $0.8$ (proportional to true click-through ratio), Our goal is to identify an arm with rewards $0.8$. Our experimental setup is similar to the one used in \citet{fiez2019sequential}. However, their true rewards are obtained from a least square regression. We do not enforce linearity in rewards in our experiment.

\begin{figure}[t]
\vspace{-15 pt}
    \centering
    \subfloat[][MNIST]{\includegraphics[width=0.5\linewidth]{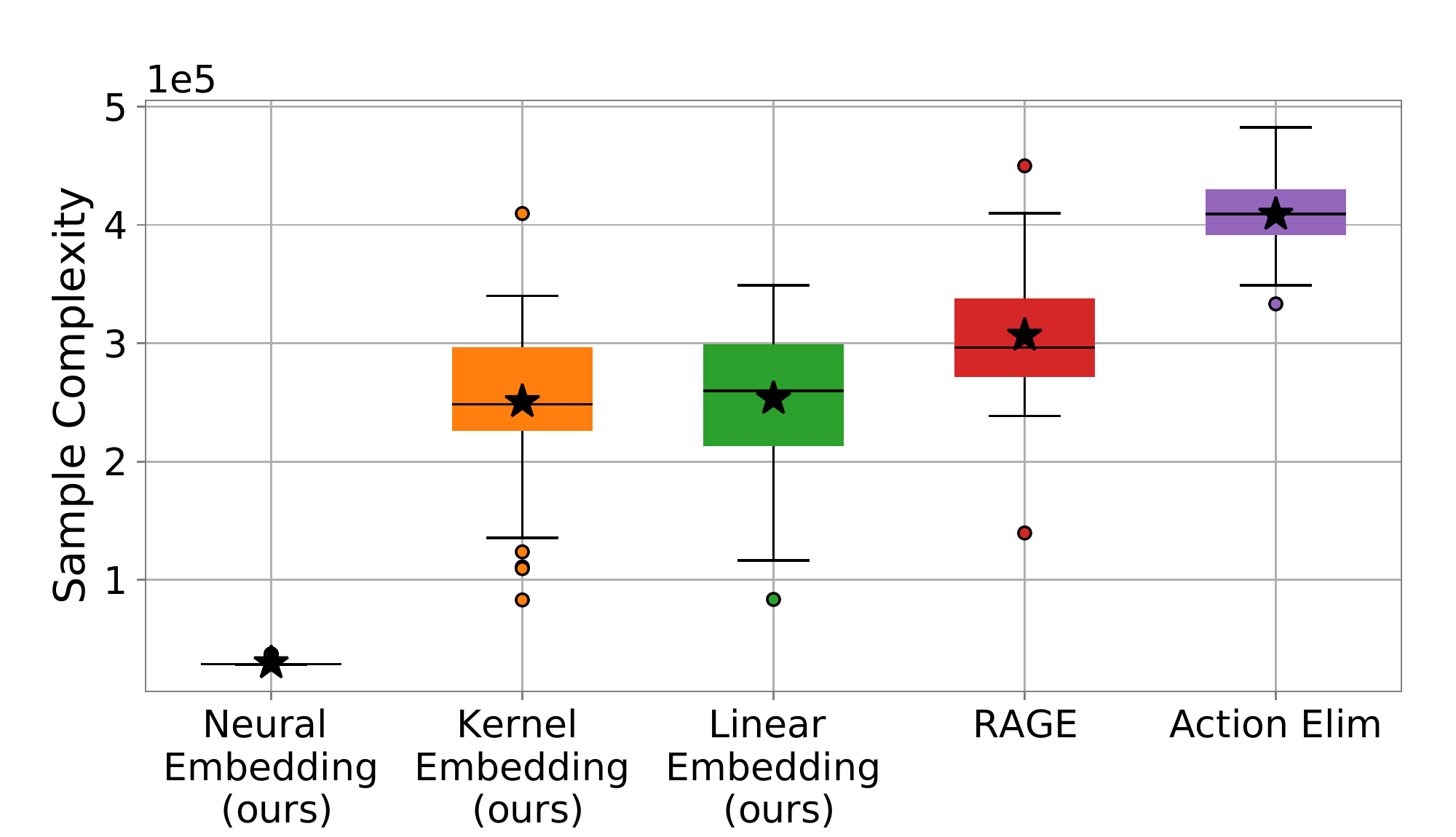}
    \label{fig:MNIST}}
    \subfloat[][Yahoo]{\includegraphics[width=0.5\linewidth]{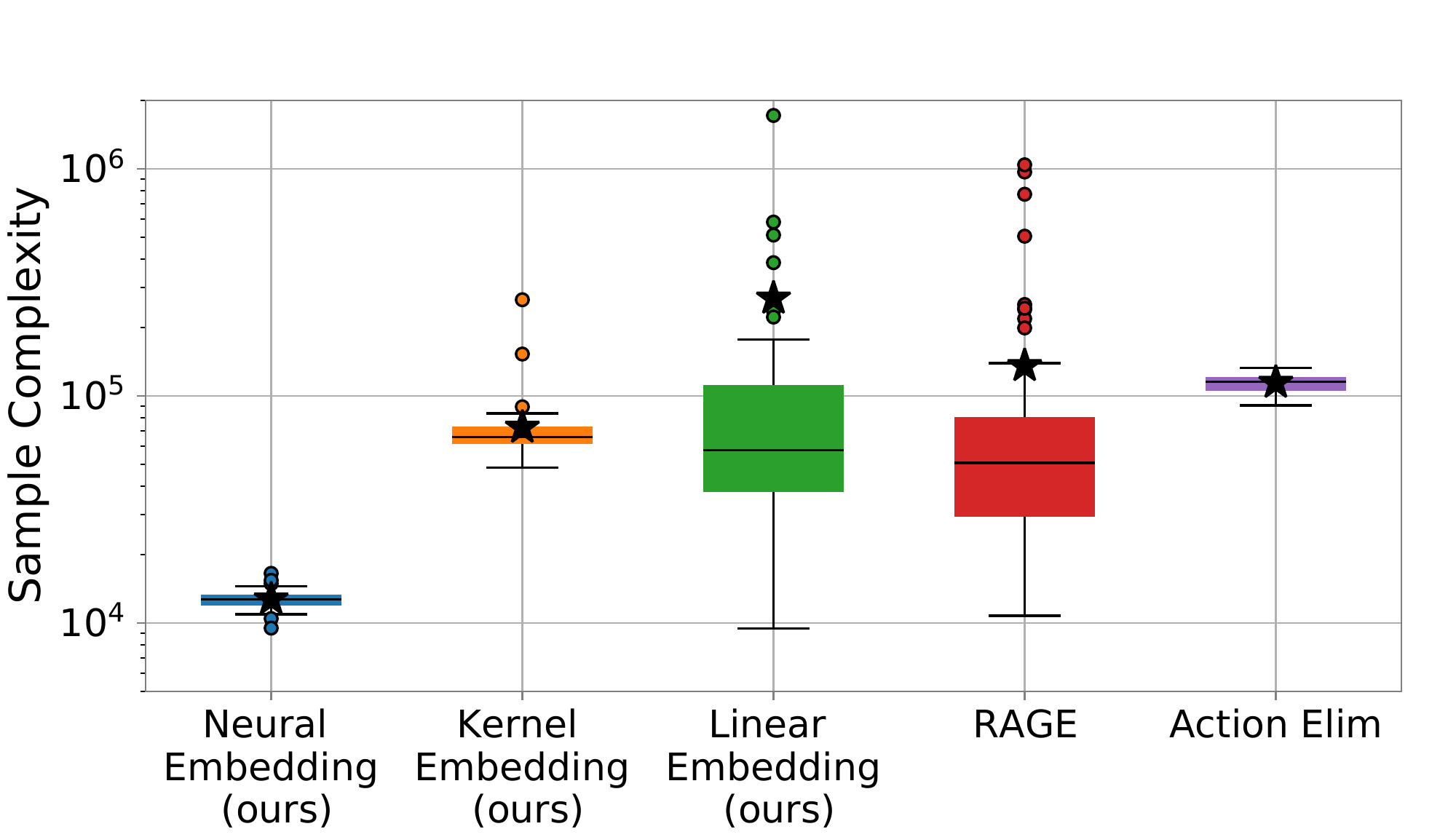}
    \label{fig:Yahoo}}\\
    \caption{Experiments on real-world datasets. The mean sample complexity is represented by a black star. The mean sample complexities of \LE and \rage are heavily affected by outliers in the Yahoo dataset.}
    \label{fig:real-world data}
    \vspace{-8pt}
\end{figure}

Box plots in \cref{fig:real-world data} show the sample complexity of each algorithm on real-world datasets. \NE significantly outperforms all other algorithms thanks to (1) the representation power of neural networks and (2) efficient exploration in low-dimensional spaces. \KE and \LE have competitive performance on the MNIST dataset. \cref{table:success_rate} shows the success rate of each algorithm. Linear methods such as \rage and \LE have relatively low success rates on the Yahoo dataset (with nonlinear rewards). Our \NE and \KE methods have high success rates since they are designed for nonlinear setting.

\vspace{-8 pt}
\begin{table}[H]
  \caption{Success rates on real-world datasets}
  \label{table:success_rate}
  \centering
  \begin{tabular}{lccccc}
    \toprule
           & {\footnotesize \NE  }   & {\footnotesize  \KE} & {\footnotesize \LE} & {\footnotesize \rage} &{\footnotesize  \AE}\\
    \midrule
    MNIST    & $98\%$ & $100\%$ & $100\%$ & $100\%$ & $100\%$\\
    Yahoo    & $100\%$ & $98\%$ & $88\%$ & $90\%$ & $100\%$ \\
    \bottomrule
  \end{tabular}
\end{table}

\section{Conclusion}
\label{sec:conclusion}

We introduce the idea of adaptive embedding in bandit pure exploration. Unlike existing works that passively deal with model misspecification, we adaptively embed high-dimensional feature representations into lower-dimensional spaces to avoid the curse of dimensionality. The induced misspecification are carefully dealt with. We further apply our approach to two under-studied settings with the nonlinearity: (1) pure exploration in an RKHS and (2) pure exploration with neural networks. Our sample complexity guarantees depend on the effective dimension of the feature spaces in the kernel or neural representations. We conduct extensive experiments on both synthetic and real-world datasets, and our algorithms greatly outperform existing ones. 

Our current analysis with neural representations is in the NTK regime, which can only describe a part of the representation of the neural networks. We leave extending our algorithm to more general settings (beyond the NTK regime) as a future direction.

\section*{Acknowledgments and Disclosure of Funding}
We thank the anonymous reviewers and area chair for their helpful comments. YZ and RN are partially supported by AFOSR grant FA9550-18-1-0166. DZ and QG are partially supported by the National Science Foundation CAREER Award 1906169, IIS-1904183 and AWS Machine Learning Research Award. RJ and RW are partially supported by AFOSR FA9550-18-1-0166, NSF OAC-1934637, NSF DMS-2023109, and NSF DGE-2022023. The views and conclusions contained in this paper are those of the authors and should not be interpreted as representing any funding agencies.

\bibliographystyle{plainnat}
\bibliography{refs.bib}

\clearpage

\section*{Checklist}

\begin{enumerate}

\item For all authors...
\begin{enumerate}
  \item Do the main claims made in the abstract and introduction accurately reflect the paper's contributions and scope?
    \answerYes{}
  \item Did you describe the limitations of your work?
    \answerYes{As discussed in \cref{sec:conclusion}.}
  \item Did you discuss any potential negative societal impacts of your work?
    \answerYes{This paper provides algorithms to conduct pure exploration for kernel and neural bandits. It does not lead to negative societal impacts in the foreseeable future.}
  \item Have you read the ethics review guidelines and ensured that your paper conforms to them?
    \answerYes{}
\end{enumerate}

\item If you are including theoretical results...
\begin{enumerate}
  \item Did you state the full set of assumptions of all theoretical results?
    \answerYes{As stated in the paper.}
	\item Did you include complete proofs of all theoretical results?
    \answerYes{All proofs are provided in the appendix in supplementary material (full paper).}
\end{enumerate}

\item If you ran experiments...
\begin{enumerate}
  \item Did you include the code, data, and instructions needed to reproduce the main experimental results (either in the supplemental material or as a URL)?
    \answerYes{Code, data, and instructions can be found in the supplementary material (code).}
  \item Did you specify all the training details (e.g., data splits, hyperparameters, how they were chosen)?
    \answerYes{As stated in \cref{sec:experiment} and \cref{app:experiment}.}
	\item Did you report error bars (e.g., with respect to the random seed after running experiments multiple times)?
    \answerYes{As reported in \cref{sec:experiment}.}
	\item Did you include the total amount of compute and the type of resources used (e.g., type of GPUs, internal cluster, or cloud provider)?
    \answerYes{As stated in the supplementary material (code).}
\end{enumerate}

\item If you are using existing assets (e.g., code, data, models) or curating/releasing new assets...
\begin{enumerate}
  \item If your work uses existing assets, did you cite the creators?
    \answerYes{}
  \item Did you mention the license of the assets?
    \answerYes{We clarified the source of datasets used in \cref{sec:experiment}.}
  \item Did you include any new assets either in the supplemental material or as a URL?
    \answerYes{Code used to produce our experiments is included in the supplementary material (code).}
  \item Did you discuss whether and how consent was obtained from people whose data you're using/curating?
    \answerYes{As stated in \cref{app:experiment}.}
  \item Did you discuss whether the data you are using/curating contains personally identifiable information or offensive content?
    \answerYes{As discussed in \cref{app:experiment}.}
\end{enumerate}

\item If you used crowdsourcing or conducted research with human subjects...
\begin{enumerate}
  \item Did you include the full text of instructions given to participants and screenshots, if applicable?
    \answerNA{}
  \item Did you describe any potential participant risks, with links to Institutional Review Board (IRB) approvals, if applicable?
    \answerNA{}
  \item Did you include the estimated hourly wage paid to participants and the total amount spent on participant compensation?
    \answerNA{}
\end{enumerate}

\end{enumerate}

\clearpage
\appendix 

\section{Supporting materials}
\label{app:supporting}

\subsection{Matrix inversion and rounding}
\label{app:rounding}

For possibly singular $\bA_{\bpsi_d(\lambda)}$, pseudo-inverse is used in $\norm{\by}_{\bA_{\bpsi_d}(\lambda)^{-1}}^2$ if $\by$ belongs to the range of $\bA_{\bpsi_d(\lambda)}$; otherwise, we set $\norm{\by}_{\bA_{\bpsi_d}(\lambda)^{-1}}^2 = \infty$.

Consider any $\cS \subseteq \cX$. With our (slightly abused) definition of matrix inversion, the optimal design 
\begin{align}
    \inf_{\lambda \in \bLambda_{\cX}} \sup_{\by \in \cY(\bpsi_d(\cS))} \norm{\by}^2_{\bA_{\bpsi_d}(\lambda)^{-1}} \label{eq:optimal_design_app}
\end{align}
will select a design $\lambda^\star \in \bLambda_{\cX}$ such that every $\by \in \cY(\bpsi_d(\cS))$ lies in the range of $\bA_{\bpsi_d}(\lambda^\star)$.\footnote{If the infimum is not attained, we can simply take a design $\lambda^{\star \star}$ with associated value $\tau^{\star \star} \leq (1+\zeta_0) \inf_{\lambda \in \bLambda_{\cX}} \sup_{\by \in \cY(\bpsi_d(\cS))} \norm{\by}^2_{\bA_{\bpsi_d}(\lambda)^{-1}}$ for a $\zeta_0 > 0$ arbitrarily small. Our analysis goes through with changes only in constant terms.} If $\spn(\cY(\bpsi_d(\cS))) = \R^d$, then $\bA_{\bpsi_d}(\lambda^\star)$ is positive definite (recall that $\bA_{\bpsi_d}(\lambda^\star) = \sum_{\bx \in \cX} \lambda_{\bx} \bpsi_d(\bx) \bpsi_d(\bx)^\top$ and $\spn(\bpsi_d(\cX))= \R^d$ by assumption). Thus the rounding guarantees (Theorem 2.1 therein, which requires a positive definite design) in \citet{allen2020near} goes through (with additional simple modifications dealt as in Appendix B of \citet{fiez2019sequential}).

We now consider the case when $\bA_{\bpsi_d}(\lambda^\star)$ is singular. Since $\spn(\bpsi_d(\cX)) = \R^d$, we can always find another $\lambda^\prime$ such that $\bA_{\bpsi_d}(\lambda^\prime)$ is invertible. For any $\zeta_1 >0$, let $\widetilde \lambda^\star = (1-\zeta_1) \lambda^\star + \zeta_1 \lambda^\prime $. We know that $\widetilde \lambda^\star$ leads to a positive definite design. With respect to $\zeta_1$, we can find another $\zeta_2 > 0$ small enough (e.g., smaller than the smallest eigenvalue of $\zeta_1 \bA_{\bpsi_d}(\lambda^\prime)$) such that $\bA_{\bpsi_d}(\widetilde \lambda^\star) \succeq \bA_{\bpsi_d}((1-\zeta_1) \lambda^\star) + \zeta_2 \bI$. Since $\bA_{\bpsi_d}((1-\zeta_1) \lambda^\star) + \zeta_2 \bI$ is positive definite, for any $\by \in \cY(\bpsi_d(\cS))$, we have 
\begin{align*}
    \norm{\by}^2_{\bA_{\bpsi_d}(\widetilde \lambda^\star)^{-1}} \leq \norm{\by}^2_{(\bA_{\bpsi_d}((1-\zeta_1) \lambda^\star) + \zeta_2 \bI)^{-1}}.
\end{align*}
Fix any $\by \in \cY(\bpsi_d(\cS))$. Since $\by$ lies in the range of $\bA_{\bpsi_d}(\lambda^\star)$ (according to the objective and the definition of matrix inversion), we clearly have 
\begin{align*}
    \norm{\by}^2_{(\bA_{\bpsi_d}((1-\zeta_1) \lambda^\star) + \zeta_2 \bI)^{-1}} 
    \leq \norm{\by}^2_{(\bA_{\bpsi_d}((1-\zeta_1) \lambda^\star))^{-1}}
    \leq \frac{1}{1-\zeta_1} \norm{\by}^2_{\bA_{\bpsi_d}(\lambda^\star)^{-1}}.
\end{align*}
To summarize, we have 
\begin{align*}
    \norm{\by}^2_{\bA_{\bpsi_d}(\widetilde \lambda^\star)^{-1}} \leq \frac{1}{1-\zeta_1} \norm{\by}^2_{\bA_{\bpsi_d}(\lambda^\star)^{-1}},
\end{align*}
where $\zeta_1$ can be chosen arbitrarily small. We can thus send the positive definite design $\widetilde \lambda^\star$ to the rounding procedure in \citet{allen2020near}. We can incorporate the additional $1/(1-\zeta_1)$ overhead, for $\zeta_1 >0$ chosen sufficiently small, into the sample complexity requirement $r_d(\zeta)$ of the rounding procedure.

\subsection{Supporting lemmas}
\begin{lemma}[Theorem 10 in \citet{katz2020empirical}]
\label{lm:lower_bound_tao}
Consider a linear bandit problem with action set $\cX \subseteq \R^D$. Suppose $\spn(\curly*{\bx_\star - \bx}_{\bx \in \cX}) = \R^D$, we then have 
\begin{align*}
 \inf_{\lambda \in \bLambda_{\cX}} \sup_{\bx\in  \cX \setminus \bx_{\star}} \norm{\bx_{\star} - \bx }^2_{\bA(\lambda)^{-1}} = \Omega(D).
\end{align*}
\end{lemma}

\section{Omitted proofs for \cref{sec:active_compression}}

\subsection{Proof of \cref{thm:active_elim_adaptive}}

\thmActiveElimAdaptive*

\begin{proof}
The proof is decomposed into three steps: (1) prove correctness through induction; (2) bound the error probability; (3) upper bound the sample complexity. We first note that $d_{\eff}(\epsilon)$ is well-defined since we assume that $\gamma(d)$ can be made arbitrarily small for $d$ large enough.

\textbf{Step 1: The induction.} We define event 
\begin{align*}
    \cE_k = \curly*{\bx_\star \in \widehat \cS_k \subseteq \cS_k},
\end{align*}
and prove through induction that 
\begin{align*}
    \P \paren{ \cE_{k+1} \vert \cap_{i \leq k} \cE_i } \geq 1 - \delta_{k},
\end{align*}
where $\delta_0 \coloneqq 0$. Recall that $\bx_\star = \argmax_{\bx \in \cX} h(\bx) $ and $\cS_k = \curly*{x \in \cX: \Delta_x < 4 \cdot 2^{-k}}$ (with $\cS_1 = \cX$). 

The base case $\curly*{ \bx_\star \in \widehat \cS_1 \subseteq \cS_1 }$ holds with probability $1$ by definition of $\widehat \cS_1 = \cS_1 = \cX$. We next analyze event $\cE_{k+1}$ conditioned on $\cap_{i \leq k} \cE_i$. For simplicity, we'll use notations $\bpsi(\cdot) \coloneqq \bpsi_{d_k}(\cdot)$, $\btheta_\star \coloneqq \btheta_{d_k}$ and $\eta(\cdot) = \eta_{d_k} (\cdot)$ in the analysis of the $k$-th iteration.

\textbf{Step 1.1: Concentration under mis-specifications.}
Let $\curly{\bx_1, \bx_2, \dots, \bx_{N_k}}$ be the arms pulled at the $k$-th iteration of \cref{alg:active_elim_adaptive} and $\curly*{y_1, \ldots, y_{N_k}}$ be the corresponding rewards. The least square estimator $\widehat{\btheta}_k = \bA_k^{-1} \bb_k \in \R^{d_k}$ is constructed with $\bA_k = \sum_{i=1}^{N_k} \bpsi(\bx_i) \bpsi(\bx_i)^{\top}$ and $\bb_k = \sum_{i=1}^{N_k} \bpsi(\bx_i) y_i$. Note that $y_i = h(\bx_i) + \xi_i = \bpsi(\bx_i)^\top \btheta_\star + \eta(\bx_i) + \xi_i$ are i.i.d. generated with 1-sub-Gaussian noise $\xi_i$. Fix any $\by \in \cY(\bpsi(\widehat \cS_k))$, we have
\begin{align}
    \abs{ \ang{\by, \widehat \btheta_k - \btheta_\star} }& = \abs{\by^\top \bA_k^{-1} \sum_{j=1}^{N_k} \bpsi(\bx_i) y_i - \by^\top \btheta_\star}  \nonumber \\
    & = \abs{\by^\top \bA_k^{-1} \sum_{j=1}^{N_k} \bpsi(\bx_i) \paren{ \bpsi(\bx_i)^\top \btheta_\star + \eta(\bx_i) + \xi_i } - \by^\top \btheta_\star} \nonumber\\
    & = \abs{\by^\top \bA_k^{-1} \sum_{j=1}^{N_k} \bpsi(\bx_i) \paren{\eta(\bx_i) + \xi_i }} \nonumber \\
    & \leq \abs{\by^\top \bA_k^{-1} \sum_{j=1}^{N_k} \bpsi(\bx_i) \eta(\bx_i)} + \abs{ \sum_{j=1}^{N_k} \by^\top \bA_k^{-1} \bpsi(\bx_i) \xi_i }. \label{eq:kernel_elim_proof_step1_deviation}
\end{align}

We now bound the two terms in \cref{eq:kernel_elim_proof_step1_deviation} separately as follows. Since $\eta(\bx_i) \leq \widetilde \gamma(d)$ by construction, for the first term, we have
\begin{align}
    \abs{\by^\top \bA_k^{-1} \sum_{j=1}^{N_k} \bpsi(\bx_i) \eta(\bx_i)} & \leq \widetilde \gamma(d) \sum_{i=1}^{N_k} \abs{ \by^\top \bA_k^{-1} \bpsi(\bx_i) } \nonumber \\
    & = \widetilde \gamma(d_k) \sum_{i=1}^{N_k} \sqrt{\paren{ \by^\top \bA_k^{-1} \bpsi(\bx_i) }^2 } \nonumber \\
    & \leq \widetilde \gamma(d_k) \sqrt{N_k \sum_{i=1}^{N_k} \paren{ \by^\top \bA_k^{-1} \bpsi(\bx_i) }^2 } \label{eq:kernel_elim_step1_first_term_jensen} \\
    & = \widetilde \gamma(d_k)  \sqrt{N_k \sum_{i=1}^{N_k}  \by^\top \bA^{-1} \bpsi(\bx_i) \bpsi(\bx_i)^\top \bA_k^{-1} \by } \nonumber \\
    & = \widetilde \gamma(d_k)  \sqrt{N_k  \, \norm{\by}_{\bA_k^{-1}}^2  } \nonumber \\
    & \leq \widetilde \gamma(d_k) \sqrt{(1+\zeta) \, \tau(\cY( \bpsi(\widehat \cS_k)))} \label{eq:kernel_elim_step1_first_term_rounding} 
\end{align}
where \cref{eq:kernel_elim_step1_first_term_jensen} comes from Jensen's inequality; \cref{eq:kernel_elim_step1_first_term_rounding} comes from the guarantee of rounding in \cref{eq:rounding}.

The second term in \cref{eq:kernel_elim_proof_step1_deviation} is a (weighted) sum of independent random variables $\xi_i$. Since $\xi_i$s are $1$-sub-Gaussian, then $\paren*{ \sum_{j=1}^{N_k} \by^\top \bA_k^{-1} \bpsi(\bx_i) \xi_i }$ is $\sqrt{\sum_{j=1}^{N_k} (\by^\top \bA_k^{-1} \bpsi(\bx_i))^2}$-sub-Gaussian. Since $\sqrt{\sum_{j=1}^{N_k} (\by^\top \bA_k^{-1} \bpsi(\bx_i))^2} = \norm{\by}_{\bA_k^{-1}}$, a standard Hoeffding's inequality leads to the following bound
\begin{align*}
    \P \paren{ \abs{\sum_{j=1}^{N_k} \by^\top \bA_k^{-1} \bpsi(\bx_i) \xi_i } \geq \norm{\by}_{\bA_k^{-1}} \sqrt{2 \log \paren{ {\abs*{\widehat \cS_k}^2}/{\delta_k}}} } \leq \frac{\delta_k}{\abs*{\widehat \cS_k}^2/2}
\end{align*}
Since there are at most $\abs*{\widehat \cS_k}^2/2$ directions in $\cY(\bpsi(\widehat \cS_k))$, taking a union bound over all possible directions, we have 
\begin{align}
    \P \paren{ \forall \by \in \cY(\bpsi(\widehat \cS_k)): \abs{\sum_{j=1}^{N_k} \by^\top \bA_k^{-1} \bpsi(\bx_i) \xi_i } \geq \norm{\by}_{\bA_k^{-1}} \sqrt{2 \log \paren{ {\abs*{\widehat \cS_k}^2}/{\delta_k}}} } \leq \delta_k. \label{eq:kernel_elim_proof_step1_second_term}
\end{align}
Plugging \cref{eq:kernel_elim_step1_first_term_rounding} and \cref{eq:kernel_elim_proof_step1_second_term} into \cref{eq:kernel_elim_proof_step1_deviation}, we have that for any $\by \in \cY(\bpsi(\widehat \cS_k))$,
\begin{align}
    \abs{ \ang{\by, \widehat \btheta_k - \btheta^\star} } \leq \widetilde \gamma(d_k) \sqrt{(1+\zeta) \, \tau(\cY( \bpsi(\widehat \cS_k)))} + \norm{\by}_{\bA_k^{-1}} \sqrt{2 \log \paren{ { \abs*{\widehat \cS_k}^2}/{\delta_k}}} =: \widetilde \omega_k (\by) \label{eq:kernel_elim_proof_concentration}
\end{align}
with probability at least $1 - \delta_k$.

\textbf{Step 1.2: Correctness.} We prove the result under the event described in \cref{eq:kernel_elim_proof_concentration}, which holds with (conditional) probability at least $1-\delta_k$. We know that $\bx_\star \in \widehat \cS_k \subseteq \cS_k$ holds conditioned on $\cap_{i \leq k} \cE_i$. We now prove it for iteration $k+1$. Recall the elimination criteria is 
\begin{align*}
    \widehat \cS_{k+1} = \widehat \cS_k \setminus \{ \bx \in \widehat \cS_k : \exists \bx^\prime \text{ such that }(\bpsi(\bx^\prime) -\bpsi(\bx))^\top \widehat{\btheta}_k \geq \omega_k(\bpsi(\bx^\prime) -\bpsi(\bx))\}.
\end{align*}

\textbf{Step 1.2.1: $\bx_\star \in \widehat \cS_{k+1}$.} Let $\widehat \bx \in \widehat \cS_k$ be any arm such that $\widehat \bx \neq \bx_\star$. We have
\begin{align*}
    (\bpsi(\widehat\bx ) -\bpsi(\bx_\star))^\top \widehat{\btheta}_k & \leq (\bpsi(\widehat\bx ) -\bpsi(\bx_\star))^\top {\btheta}^\star + \widetilde \omega_k (\bpsi(\widehat \bx) -\bpsi(\bx_\star)) \nonumber \\
    & = h(\widehat \bx) - \eta(\widehat \bx) - (h(\bx_\star) - \eta(\bx_\star)) + \widetilde \omega_k(\bpsi(\widehat \bx) -\bpsi(\bx_\star)) \nonumber \\
    & \leq  - \Delta(\widehat \bx) + \paren{ 2 \widetilde \gamma(d_k) + \widetilde \omega_k(\bpsi(\widehat \bx) -\bpsi(\bx_\star))} \\
    & < \paren{ 2 \widetilde \gamma(d_k) + \widetilde \omega_k(\bpsi(\widehat \bx) -\bpsi(\bx_\star))} \nonumber \\
    & = \omega_k(\bpsi(\widehat \bx) -\bpsi(\bx_\star)).
\end{align*}
As a result, the optimal arm $\bx_\star$ remains in $\widehat \cS_{k+1}$.

\textbf{Step 1.2.2: $\widehat \cS_{k+1} \subseteq \cS_{k+1}$.} For any $\bx \in \cS_{k+1}^c \cap \widehat \cS_k$, we have $\Delta(\bx) = h(\bx_\star) - h(\bx) \geq 4 \cdot 2^{-(k+1)} = 2 \cdot 2^{-k}$ by definition. Since $\bx_\star \in \widehat \cS_k$, we have 
\begin{align}
    (\bpsi(\bx_\star) -\bpsi(\bx))^\top \widehat{\btheta}_k & \geq (\bpsi(\bx_\star) -\bpsi(\bx))^\top {\btheta}^\star - \widetilde \omega_k(\bpsi(\bx_\star) -\bpsi(\bx)) \nonumber \\
    & = h(\bx_\star) - \eta(\bx_\star) - (h(\bx) - \eta(\bx))- \widetilde \omega_k (\bpsi(\bx_\star) -\bpsi(\bx)) \nonumber \\
    & \geq \Delta(\bx) - \paren{ 2 \widetilde \gamma(d_k) + \widetilde \omega_k (\bpsi(\bx_\star) -\bpsi(\bx)) } \nonumber \\
    & = \Delta(\bx) - \omega_k(\bpsi(\bx_\star) -\bpsi(\bx)) \nonumber  \\
    & \geq \omega_k(\bpsi(\bx_\star) -\bpsi(\bx)) \label{eq:kernel_elim_proof_step121},
\end{align}
where \cref{eq:kernel_elim_proof_step121} comes from the fact that $\omega_k(\bpsi(\bx_\star) -\bpsi(\bx)) \leq 2^{-k}$ by the selection of $N_k$ (as discussed later) and the guarantees from the rounding procedure \cref{eq:rounding}.
As a result, any $\bx \in \cS_{k+1}^c \cap \widehat \cS_k$ will be eliminated and $\widehat \cS_{k+1} \subseteq \cS_{k+1}$.

To summarize, we prove the induction at iteration $k+1$, i.e.,
\begin{align*}
    \P \paren{ \cE_{k+1} \vert \cap_{i< k+1} \cE_{i} } \geq 1 - \delta_k.
\end{align*}

To show the selection of $N_k$ guarantees \cref{eq:kernel_elim_proof_step121}, it suffices to show that $2^{-k} - \epsilon_k \geq 2^{-k}/2$. We first notice the following inequalities:
\begin{align}
    (1+\zeta) \, \tau(\cY( \bpsi(\widehat \cS_k))) & = (1+\zeta) \,\inf_{\lambda \in \bLambda_{\cX}} \sup_{\by \in \cY( \bpsi(\widehat \cS_k))} \norm*{\by}^2_{\bA_{\bpsi}(\lambda)^{-1}} \nonumber \\
    & \leq (1+\zeta) \,\inf_{\lambda \in \bLambda_{\cX}} \sup_{\by \in \cY( \bpsi(\cX))} \norm*{\by}^2_{\bA_{\bpsi}(\lambda)^{-1}}  \nonumber \\
    & = g(d, \zeta) \label{eq:kernel_elim_gd}.
\end{align}

Recall that $\gamma(d) = \paren*{16 + 8\sqrt{g(d, \zeta)}} \widetilde  \gamma(d)$. At each the $k$-th iteration with $k \leq n \coloneqq \ceil*{\log_2(2/\epsilon)}$, we have 
\begin{align}
    \epsilon_k & \leq \gamma(d_k)/8 \label{eq:kernel_elim_adaptive_proof_eps_1} \\
    & \leq 2^{-k} / 2 \label{eq:kernel_elim_adaptive_proof_eps_2},
\end{align}
where \cref{eq:kernel_elim_adaptive_proof_eps_1} comes from \cref{eq:kernel_elim_gd}, and \cref{eq:kernel_elim_adaptive_proof_eps_2} comes from the definition of $d_k = d_{\eff}(4 \cdot 2^{-k})$. This implies that $2^{-k} - \epsilon_k \geq 2^{-k}/2$ and thus $N_k$ is well-defined.

\textbf{Step 2: The error probability.}
From the analysis in Step 1, we have 
\begin{align*}
    \P(\cE_{k+1} \vert \cE_{k} \cap \dots \cap \cE_1) \geq 1 - \delta_k.
\end{align*}
Let $\cE = \cap_{k=1}^{n+1} \cE_k$, we then have 
\begin{align}
    \P \paren{ \cE } & = \prod_{k=1}^{n+1} \P \paren{ \cE_k \vert \cE_{k-1} \cap \dots \cap \cE_1 } \nonumber \\
    & = \prod_{k=1}^{n+1} \paren{1 - \delta_{k-1}} \nonumber \\
    & \geq \prod_{k=1}^\infty \paren{1 - \delta/k^2} \nonumber \\
    & = \frac{\sin(\pi \delta)} {\pi \delta} \nonumber \\
    & \geq 1- \delta \label{eq:kernel_elim_proof_step2},
\end{align}
where we use the fact that ${\sin(\pi \delta)}/{\pi \delta} \geq 1-\delta$ for any $\delta \in (0,1)$ in \cref{eq:kernel_elim_proof_step2}. We analyze the following steps under the good event $\cE$.

\textbf{Step 3: Sample complexity upper bound.} As one can see, the total sample complexity $N$ is upper bounded by
\begin{align*}
    N & \leq \sum_{k=1}^n N_k \nonumber  \\
    & = \sum_{k=1}^{\ceil*{\log_2(2/\epsilon)}}  \max \curly{ \ceil{ (2^{-k} - \epsilon_k)^{-2} 2 (1+\zeta)\,\tau_k  \log(\abs*{\widehat \cS_k}^2/\delta_k) }  ,r_{d_k}(\zeta)} \nonumber  \\
    & \leq \sum_{k=1}^{\ceil*{\log_2(2/\epsilon)}} \paren{ \paren{ (2^{-k} - \epsilon_k)^{-2} 2 (1+\zeta)\,\tau_k  \log(\abs*{ \widehat \cS_k}^2/\delta_k)}  + \paren{r_{d_k}(\zeta)+1}  } \nonumber \\
    & \leq \sum_{k=1}^{\ceil*{\log_2(2/\epsilon)}} \paren{ \paren{ (2^{2k}) 10\,\tau_k  \log(k^2 \abs{\cX}^2/\delta)}  + \paren{r_{d_k}(\zeta)+1}  }. \nonumber \\
\end{align*}
We analyze $\rho^\star_{d_k}(2^{2-k})$ in the following to obtain an instance dependent sample complexity upper bound.
\begin{align}
    \rho_{d_k}^\star(2^{2-k}) & = \inf_{\lambda \in \bLambda_{\cX}} \sup_{ \bx \in \cX \setminus \curly*{ \bx_\star }} \frac{\norm*{\bpsi_{d_k}(\bx_{\star})-\bpsi_{d_k}(\bx)}^2_{\bA_{\bpsi_{d_k}}(\lambda)^{-1}}}{ \max \curly{ \Delta(\bx), 4 \cdot 2^{-k}}^2} \nonumber\\
    & = \inf_{\lambda \in \bLambda_{\cX}} \sup_{i \leq k} \sup_{ \bx \in \cS_i \setminus \curly*{ \bx_\star }} \frac{\norm*{\bpsi_{d_k}(\bx_{\star})-\bpsi_{d_k}(\bx)}^2_{\bA_{\bpsi_{d_k}}(\lambda)^{-1}}}{ \max \curly{ \Delta(\bx), 4 \cdot 2^{-k}}^2} \nonumber \\
    & \geq \sup_{i \leq k} \inf_{\lambda \in \bLambda_{\cX}} \sup_{ \bx \in \cS_i \setminus \curly*{ \bx_\star }} \frac{\norm*{\bpsi_{d_k}(\bx_{\star})-\bpsi_{d_k}(\bx)}^2_{\bA_{\bpsi_{d_k}}(\lambda)^{-1}}}{ \max \curly{ \Delta(\bx), 4 \cdot 2^{-k}}^2} \nonumber\\
    & \geq \sup_{i \leq k} \inf_{\lambda \in \bLambda_{\cX}} \sup_{ \bx \in \cS_i \setminus \curly*{ \bx_\star }} \frac{\norm*{\bpsi_{d_k}(\bx_{\star})-\bpsi_{d_k}(\bx)}^2_{\bA_{\bpsi_{d_k}}(\lambda)^{-1}}}{ \max \curly{ 4\cdot 2^{-i}, 4 \cdot 2^{-k}}^2} \nonumber\\
    & \geq \frac{1}{k} \sum_{i=1}^{k} \inf_{\lambda \in \bLambda_{\cX}} \sup_{ \bx \in \cS_i \setminus \curly*{ \bx_\star }} \frac{\norm*{\bpsi_{d_k}(\bx_{\star})-\bpsi_{d_k}(\bx)}^2_{\bA_{\bpsi_{d_k}}(\lambda)^{-1}}}{ 16\cdot 2^{-2i}} \nonumber \\
    & \geq \frac{1}{k} \sum_{i=1}^{k} \inf_{\lambda \in \bLambda_{\cX}} \sup_{ \bx, \bx^\prime \in \cS_i} \frac{\norm*{\bpsi(\bx)-\bpsi(\bx^\prime)}^2_{\bA_{\bpsi}(\lambda)^{-1}}/ 4}{ 16\cdot 2^{-2i}} \label{eq:active_elim_proof_step3_tri} \\
    & = \frac{1}{k} \sum_{i=1}^{k} \paren{2^{2i}} \, \tau(\cY(\bpsi( \cS_i)))/64 \nonumber \\ 
    & \geq \frac{1}{64 k} (2^{2k}) \tau(\cY(\bpsi( \cS_k))) \nonumber,
\end{align}
where \cref{eq:active_elim_proof_step3_tri} comes from the fact that we can write $\bpsi(\bx) - \bpsi(\bx^\prime) = \bpsi(\bx) - \bpsi(\bx_\star) + \bpsi(\bx_\star) - \bpsi(\bx^\prime)$ and a triangle inequality.

We have $\tau_k = \tau(\cY(\bpsi_{d_k}(\widehat \cS_k))) \leq \tau(\cY(\bpsi_{d_k}( \cS_k)))$ since $\widehat \cS_k \subseteq \cS_k$. Thus, we have 
\begin{align*}
    (2^{2k}) \tau_k & \leq 64 k \, \rho_{d_k}^\star(2^{2-k}).
\end{align*}

As a result, we have 
\begin{align}
    N & \leq 640 \sum_{k=1}^{\ceil*{\log_2(2/\epsilon)}} \paren{ \paren{ k \, \rho_{d_k}^\star(2^{2-k})  \log(k^2 \abs{\cX}^2/\delta)}  + \paren{r_{d_k}(\zeta)+1} } \nonumber \\
    & = \widetilde{O} \paren{ d_{\eff}(\epsilon) \cdot \max \curly*{\Delta_{\min}, \epsilon}^{-2} }, \label{eq:kernel_elim_adaptive_proof_worst_case}
\end{align}
where \cref{eq:kernel_elim_adaptive_proof_worst_case} is discussed as follows.

Since we clearly have $2^{2-k} \geq \epsilon$ when $k \leq n \coloneqq \ceil*{\log_2(2/\epsilon)}$, we have 
\begin{align}
    \rho_{d_k}^\star(2^{2-k}) & =  \rho_{d_k}^\star(\epsilon) \nonumber \\
    & =\inf_{\lambda \in \bLambda_{\cX}} \sup_{ \bx \in \cX \setminus \curly*{ \bx_\star }} \frac{\norm*{\bpsi_{d_k}(\bx_{\star})-\bpsi_{d_k}(\bx)}^2_{\bA_{\bpsi_d}(\lambda)^{-1}}}{ \max \curly{ h(\bx_\star) - h(\bx), \epsilon}^2} \nonumber\\
    & \leq  \inf_{\lambda \in \bLambda_{\cX}} \sup_{ \bx \in \cX \setminus \curly*{ \bx_\star }} \norm*{\bpsi_{d_k}(\bx_{\star})-\bpsi_{d_k}(\bx)}^2_{\bA_{\bpsi_d}(\lambda)^{-1}} \cdot \max \curly*{\Delta_{\min}, \epsilon}^{-2} \nonumber \\
    & \leq 4 d_k \cdot \max \curly*{\Delta_{\min}, \epsilon}^{-2} \label{eq:active_elim_step3_kw}\\
    & = 4 d_{\eff}(4 \cdot 2^{-k}) \cdot \max \curly*{\Delta_{\min}, \epsilon}^{-2} \nonumber \\
    & \leq 4 d_{\eff}(\epsilon) \cdot \max \curly*{\Delta_{\min}, \epsilon}^{-2}, \label{eq:active_elim_step3_d_eff}
\end{align}
where \cref{eq:active_elim_step3_kw} comes from Kiefer–Wolfowitz Theorem \citep{kiefer1960equivalence} and \cref{eq:active_elim_step3_d_eff} comes from the fact that $2^{2-k} \geq \epsilon$ and the definition of $d_{\eff}(\epsilon)$. Similarly, we have $r_{d_k}(\zeta) \leq O(d_k/\zeta^2) \leq O(d_{\eff}(\epsilon))$.
\end{proof}

\subsection{Proof of \cref{prop:complexity_linear}}

\propComplexityLinear*

\begin{proof}
Recall that 
\begin{align*}
    \rho_d^\star(\epsilon) = \inf_{\lambda \in \bLambda_{\cX}} \sup_{ \bx \in \cX \setminus \curly*{ \bx_\star }} \frac{\norm*{\bpsi_d(\bx_{\star})-\bpsi_d(\bx)}^2_{\bA_{\bpsi_d}(\lambda)^{-1}}}{ \max \curly{ \Delta_{\bx}, \epsilon}^2},
\end{align*}
\begin{align*}
    \widetilde \rho_d^\star(\epsilon) = \inf_{\lambda \in \bLambda_{\cX}} \sup_{ \bx \in \cX \setminus \curly*{ \bx_\star }} \frac{\norm*{\bpsi_d(\bx_{\star})-\bpsi_d(\bx)}^2_{\bA_{\bpsi_d}(\lambda)^{-1}}}{ \max \curly{ \ang{\bpsi_d(\bx_\star) - \bpsi_d(\bx), \btheta_d}, \epsilon}^2},
\end{align*}
and 
\begin{align}
    \Delta_{\bx} = h(\bx_\star) - h(\bx) = \ang{\bpsi(\bx_\star) - \bpsi(\bx), \btheta_d} + \eta(\bx_\star) - \eta(\bx), \label{eq:lm_linear_proof_difference}
\end{align}
where $\abs{\eta(\bx)} \leq \gamma$ for any $\bx \in \cX$.

To relate $\rho^\star(\epsilon)$ with $\widetilde \rho^\star(\epsilon)$, we only need to relate $\max \curly{ \Delta_{\bx}, \epsilon}$ with $\max \curly{ \ang{\bpsi_d(\bx_\star) - \bpsi_d(\bx), \btheta_d}, \epsilon}$. From \cref{eq:lm_linear_proof_difference} and the fact that $\epsilon \geq \gamma$, we know that 
\begin{align*}
    \ang{\bpsi_d(\bx_\star) - \bpsi_d(\bx), \btheta_d} \leq \Delta_{\bx} + 2 \gamma \leq \Delta_{\bx} + 2 \epsilon \leq 3 \max \curly{ \Delta_{\bx}, \epsilon},
\end{align*}
and thus 
\begin{align*}
    \max \curly{ \ang{\bpsi_d(\bx_\star) - \bpsi_d(\bx), \btheta_d}, \epsilon} \leq 3 \max \curly{ \Delta_{\bx}, \epsilon}.
\end{align*}
As a result, we have $\rho^\star(\epsilon) \leq 9 \widetilde \rho^\star (\epsilon)$.

When $\gamma < \Delta_{\min}/2$, we have 
\begin{align*}
    \ang{\bpsi_d(\bx_\star) - \bpsi_d(\bx), \btheta_d} \geq \Delta_{\bx} - 2 \gamma > 0,
\end{align*}
which implies that $\bx_\star$ is still the optimal arm in the related linear bandit problem in $\R^d$ with reward function $\widetilde \bh(\bx) = \ang*{\bpsi_d(\bx), \btheta_d}$. Thus, $\widetilde \rho^\star(0) \log(1/2.4\delta)$ serves as the lower bound for best arm identification in the related linear bandit problem \citep{soare2014best, fiez2019sequential, degenne2020gamification}. $\widetilde \rho^\star (\epsilon)$ is the $\epsilon$-relaxed version of $\widetilde \rho^{\star}(0)$.
\end{proof}

\subsection{Example for high-dimensional linear bandits}
\label{app:high_dim_linear}

\textbf{Example.} Let $\cX_1 \oplus \cX_2 = \curly{\bx_i + \bx_j: \bx_i \in \cX_1, \bx_j \in \cX_2}$. Consider a high dimensional linear bandit problem in $\R^D$ with an action set $\cX = \curly*{\bx_1, \bx_1 \oplus \curly{\eta \be_i}_{i=1}^D, \bx_2 \oplus \curly*{\eta \be_i}_{i=1}^D}$. For any $\epsilon \in (0, 1/4)$, we select $\bx_1$, $\bx_2$ and $\eta>0$ such that: (1) $\norm{\bx_1}_2 = 1$; (2) $\ang*{\bx_1, \bx_2} = 1 - 2\epsilon$; (3) $\ang*{\bx_1, \be_i} < 0$; and (4) $8 \sqrt{2} (2 + \sqrt{5D})D \eta \leq \epsilon$. We select $\btheta_\star = \bx_1$ (thus $\norm{\btheta_\star}_{2} = 1$), which implies that arm $\bx_\star = \bx_1$ is the optimal arm, arms in $\curly*{\bx_1 \oplus \curly*{\eta \be_i}_{i=1}^D}$ are $\epsilon$-optimal and arms in $\curly*{\bx_2 \oplus \curly*{\eta \be_i}_{i=1}^D}$ have sub-optimality gap $> \epsilon$ but $< 3 \epsilon$ (note that we must have $\eta < \epsilon$ be construction). As an example, one can select $\bx_1 = - \bm{1}/\sqrt{D}$ ($\bm{1} \in \R^D$ represents the all 1 vector) and $\bx_2 = (1-2\epsilon) \bx_1$ with a sufficiently small $\eta$.

\textbf{Analysis.} Let $\bX \in \R^{(2D+1) \times D}$ be the data matrix whose rows are $\bx \in \cX$; and let $\widetilde \bX \in \R^{(2D+1) \times D}$ be the matrix such that its first $(D+1)$ rows are all $\bx_1$ and its last $D$ rows are all $\bx_2$. One can clearly see that $\widetilde \bX$ is (at most) rank 2, and $\bX = \widetilde \bX + \bE$ where $\bE = [\bm{0}^\top; \eta \bI_D;\eta \bI_D]$ (and thus $\norm{\bE}_2 = \sqrt{2} \eta$). Let $\curly*{\sigma_i}_{i=1}^{D}$ represents the singular values of matrix $\bX$ and let $\curly*{\widetilde \sigma_i}_{i=1}^{D}$ represents the singular values of matrix $\widetilde \bX$.   Perturbation theory on SVD \citep{stewart1998perturbation} implies that
\begin{align*}
    \abs*{\sigma_i - \widetilde \sigma_i} \leq \norm{\bE}_2 = \sqrt{2} \eta, \, \forall i \in [D].
\end{align*}
Since $\widetilde \sigma_i = 0$ for $i > 2$, we have 
\begin{align*}
    \widetilde \gamma(2) = \norm{\btheta_\star}_{2} \cdot \paren{\sum_{i >2} \sigma_i} < \sqrt{2} D \eta.
\end{align*}
The fourth requirement in the construction implies that $d_{\eff}(\epsilon) \leq 2$. To identify a $\epsilon$-optimal arm, \cref{alg:active_elim} will first project the action set into a two-dimensional subspace and then perform arm elimination. Invoking \cref{thm:active_elim}, we know that the sample complexity is upper bounded by $\widetilde{O} (1/\epsilon^2)$, which is independent of $D$ (besides logarithmic dependence on $\abs*{\cX}$). 

We now analyze the sample complexity of \rage \citep{fiez2019sequential}. To identifying an $\epsilon$-optimal arm, we would run \rage for $n = \ceil*{\log_2(2/\epsilon)}$ iterations in $\R^D$, and its sample complexity upper bound scales as
\begin{align*}
   \sum_{k=1}^n \max \curly{ 2(1+\zeta) 2^{2k} \inf_{\lambda \in \bLambda_{\cX}} \sup_{\by \in \cY(\cS_{k})} \norm{\by}^2_{\bA(\lambda)^{-1}} \log(k^2 \abs*{\cX}^2/\delta) , r_D(\zeta)}.
\end{align*}
Consider the round $k = \floor*{\log_2(1/\epsilon)}$. The sample complexity bound at that round is larger than
\begin{align*}
    2^{2\floor*{\log_2(1/\epsilon)}} \inf_{\lambda \in \bLambda_{\cX}} \sup_{\by \in \cY(\cS_{\floor*{\log_2(1/\epsilon)}})} \norm{\by}^2_{\bA(\lambda)^{-1}} \geq \frac{1}{4 \epsilon^2} \inf_{\lambda \in \bLambda_{\cX}} \sup_{\by \in \cY(\cX)} \norm{\by}^2_{\bA(\lambda)^{-1}},
\end{align*}
since $\cS_{\floor*{\log_2(1/\epsilon)}} = \cX$ by construction (note that all arms in $\cX$ are $4\epsilon$-optimal). Since $\spn(\curly*{\bx_\star - \bx}_{\bx \in \cX}) = \R^D$ by construction, applying \cref{lm:lower_bound_tao} in \cref{app:supporting} shows that 
\begin{align*}
    \frac{1}{4 \epsilon^2} \inf_{\lambda \in \bLambda_{\cX}} \sup_{\by \in \cY(\cX)} \norm{\by}^2_{\bA(\lambda)^{-1}} = \Omega(D/\epsilon^2).
\end{align*}

\section{Arm elimination with a fixed embedding}
\label{app:fixed_embedding}

We only consider the case when $\widetilde \gamma(d) >0$. Otherwise, one can simply apply \rage to identify the best arm (when it is unique) or to identify an $\epsilon$-optimal arm (by stopping \rage after $O(\log(1/\epsilon))$ rounds).

\subsection{The algorithm}

\begin{algorithm}[H]
	\caption{Arm Elimination with Fixed Embedding and Induced Misspecification}
	\label{alg:active_elim} 
	\renewcommand{\algorithmicrequire}{\textbf{Input:}}
	\renewcommand{\algorithmicensure}{\textbf{Output:}}
	\begin{algorithmic}[1]
		\REQUIRE Action set $\cX$, confidence parameter $\delta$, embedding dimension $d$ and rounding approximation factor $\zeta$.
		\STATE Construct the compressed feature representation $\bpsi(\cX) = \bpsi_d(\cX)$ with the induced misspecification level $\widetilde \gamma(d)$.
		\STATE Set $\widehat \cS_1 = \cS_1$, $\gamma(d) = \paren*{16 + 8\sqrt{ g(d , \zeta)}} \, \widetilde  \gamma(d)$, and $n = \ceil*{\log_2(2/\gamma(d))}$.

		\FOR {$k = 1, 2, \dots , n$}
		\STATE Set $\delta_k = \delta/k^2$.
		\STATE Set $\lambda_k$ and $\tau_k$ be the design and the value of the following optimization problem 
		$$\inf_{\lambda \in \bLambda_{\cX}} \sup_{\by \in \cY( \bpsi(\widehat \cS_k))} \norm*{\by}^2_{\bA_{\bpsi}(\lambda)^{-1}}.$$
		\STATE Set $\epsilon_k = 2 \widetilde \gamma(d) + \widetilde \gamma(d) \sqrt{(1+\zeta) \, \tau_k }$,\\
		and set $ N_k = \max \curly*{ \ceil*{ (2^{-k} - \epsilon_k)^{-2} 2 (1+\zeta)\,\tau_k  \log(\abs*{\widehat \cS_k}^2/\delta_k) }  ,r_d(\zeta)}.$
		\STATE Get $\curly{\bx_1, \bx_2, \dots, \bx_{N_k}} = \round (\lambda_k,N_k,d,\zeta)$.
		\STATE Pull arms $\curly{\bx_1, \bx_2, \dots, \bx_{N_k}}$ and receive rewards $\curly*{y_1, \ldots, y_{N_k}}$. 
        \STATE Set $\widehat{\btheta}_k = \bA_k^{-1} \bb_k$, where $\bA_k = \sum_{i=1}^{N_k} \bpsi(\bx_i) \bpsi(\bx_i)^{\top}$ and $\bb_k = \sum_{i=1}^{N_k} \bpsi(\bx_i) y_i$.
        \STATE 
        Eliminate arms with respect to criteria 
        \begin{align*}
            \widehat \cS_{k+1} = \widehat \cS_k \setminus \{ \bpsi(\bx) \in \cS_k : \exists \bx^\prime \text{ such that }(\bpsi(\bx^\prime) -\bpsi(\bx))^\top \widehat{\btheta}_k \geq \omega_k(\bpsi(\bx^\prime) -\bpsi(\bx))\},
        \end{align*}
         where $\omega_k(\by) = \epsilon_k + \norm{\by}_{\bA_k^{-1}} \sqrt{2 \log \paren*{ { \abs*{\widehat \cS_k}^2}/{\delta_k}}}$.
		\ENDFOR 
		\ENSURE Output any arm in $\widehat \cS_{n+1}$.
	\end{algorithmic}
\end{algorithm}

\subsection{Sample complexity analysis }

\begin{restatable}{theorem}{thmActiveElim}
\label{thm:active_elim}
With probability of at least $1-\delta$, \cref{alg:active_elim} correctly outputs an $\gamma (d)$-optimal arm with sample complexity upper bounded by 
\begin{align}
    N \leq 640 \, \rho_d^\star(\gamma(d)) \log_2 \paren{\frac{4}{ \gamma(d)}}  \log \paren{ \frac{\log_2 \paren{4/ \gamma(d)} \abs{\cX}^2 }{\delta} } + (r_d(\zeta) + 1) \log_2 \paren{\frac{4}{ \gamma(d)}} . \label{eq:kernel_elim_complexity}
\end{align}
\end{restatable}

\begin{remark}
\label{rm:fixed_embedding}
\cref{alg:active_elim} can be viewed as a modified version of \cref{alg:active_elim_adaptive} with a fixed embedding. When we choose $d = d_{\eff}(\epsilon)$ as the input, \cref{alg:active_elim} will output an $\epsilon$-optimal arm.
\end{remark}

\begin{proof}

The proof is very similar to the proof of \cref{thm:active_elim_adaptive} (we re-define $\bpsi(\cdot) \coloneqq \bpsi_d(\cdot)$ with the fixed $d$). 

We first notice that
\begin{align*}
    \epsilon_k & \leq (2 + \sqrt{g(d,\zeta)})\widetilde \gamma(d) \nonumber\\
    & = \gamma(d)/8 \nonumber \\
    & \leq 2^{-{\ceil*{\log_2(2/\gamma(d))}}} /2 \nonumber \\
    & \leq 2^{-k}/2, 
\end{align*}
for any $1 \leq k \leq n = {\ceil*{\log_2(2/\gamma(d))}}$. Thus, $N_k$ is well-defined and Step 1 and 2 from the proof of \cref{thm:active_elim_adaptive} hold true.

We now analyze the sample complexity. As one can see, the total sample complexity $N$ is upper bounded by
\begin{align}
    N & \leq \sum_{k=1}^n N_k \nonumber  \\
    & = \sum_{k=1}^{\ceil*{\log_2(2/\gamma(d))}}  \max \curly{ \ceil{ (2^{-k} - \epsilon_k)^{-2} 2 (1+\zeta)\,\tau(\cY(\bpsi(\widehat \cS_k)))  \log(\abs*{\widehat \cS_k}^2/\delta_k) }  ,r_d(\zeta)} \nonumber  \\
    & \leq \sum_{k=1}^{\ceil*{\log_2(2/\gamma(d))}}  \paren{ (2^{-k} - \epsilon_k)^{-2} 2 (1+\zeta)\,\tau(\cY(\bpsi(\widehat \cS_k)))  \log(\abs*{ \widehat \cS_k}^2/\delta_k)}  + \paren{r_d(\zeta)+1} \log_2(4/\gamma(d)) \nonumber \\
    & \leq \sum_{k=1}^{\ceil*{\log_2(2/\gamma(d))}} \paren{2^{2k}} 10\,\tau(\cY(\bpsi( \cS_k)))  \log(k^2\abs{\cX}^2/\delta)  + \paren{r_d(\zeta)+1} \log_2(4/\gamma(d))  \label{eq:kernel_elim_proof_step3_immediate_bound},
\end{align}
where we use the fact that $\epsilon_k \leq 2^{-k}/2$, $\widehat \cS_k \subseteq \cS_k \subseteq \cX$, $\delta_k = \delta/k^2$ and $\zeta \leq 1/4$ in \cref{eq:kernel_elim_proof_step3_immediate_bound}.

We analyze the instance dependent sample complexity $\rho^\star(\gamma(d))$ next.
\begin{align}
    \rho^\star(\gamma(d)) & = \inf_{\lambda \in \bLambda_{\cX}} \sup_{ \bx \in \cX \setminus \curly*{ \bx_\star }} \frac{\norm*{\bpsi(\bx_{\star})-\bpsi(\bx)}^2_{\bA_{\bpsi}(\lambda)^{-1}}}{ \max \curly{ \Delta(\bx), \gamma(d)}^2} \nonumber\\
    & = \inf_{\lambda \in \bLambda_{\cX}} \sup_{k \leq \ceil*{\log_2(2/\gamma(d))}} \sup_{ \bx \in \cS_k \setminus \curly*{ \bx_\star }} \frac{\norm*{\bpsi(\bx_{\star})-\bpsi(\bx)}^2_{\bA_{\bpsi}(\lambda)^{-1}}}{ \max \curly{ \Delta(\bx), \gamma(d)}^2} \nonumber \\
    & \geq \sup_{k \leq \ceil*{\log_2(2/\gamma(d))}} \inf_{\lambda \in \bLambda_{\cX}} \sup_{ \bx \in \cS_k \setminus \curly*{ \bx_\star }} \frac{\norm*{\bpsi(\bx_{\star})-\bpsi(\bx)}^2_{\bA_{\bpsi}(\lambda)^{-1}}}{ \max \curly{ \Delta(\bx), \gamma(d)}^2} \nonumber\\
    & \geq \sup_{k \leq \ceil*{\log_2(2/\gamma(d))}} \inf_{\lambda \in \bLambda_{\cX}} \sup_{ \bx \in \cS_k \setminus \curly*{ \bx_\star }} \frac{\norm*{\bpsi(\bx_{\star})-\bpsi(\bx)}^2_{\bA_{\bpsi}(\lambda)^{-1}}}{ \max \curly{ 4\cdot 2^{-k}, \gamma(d)}^2} \nonumber\\
    & \geq \frac{1}{\ceil*{\log_2(2/\gamma(d))}} \sum_{k=1}^{\ceil*{\log_2(2/\gamma(d))}} \inf_{\lambda \in \bLambda_{\cX}} \sup_{ \bx \in \cS_k \setminus \curly*{ \bx_\star }} \frac{\norm*{\bpsi(\bx_{\star})-\bpsi(\bx)}^2_{\bA_{\bpsi}(\lambda)^{-1}}}{ \max \curly{ 4\cdot 2^{-k}, \gamma(d)}^2} \nonumber \\
    & \geq \frac{1}{\ceil*{\log_2(2/\gamma(d))}} \sum_{k=1}^{\ceil*{\log_2(2/\gamma(d))}} \inf_{\lambda \in \bLambda_{\cX}} \sup_{ \bx, \bx^\prime \in \cS_k} \frac{\norm*{\bpsi(\bx)-\bpsi(\bx^\prime)}^2_{\bA_{\bpsi}(\lambda)^{-1}}/ 4}{ \max \curly{ 4\cdot 2^{-k}, \gamma(d)}^2} \label{eq:kernel_elim_proof_step3_tri} \\
    & = \frac{1}{\ceil*{\log_2(2/\gamma(d))}} \sum_{k=1}^{\ceil*{\log_2(2/\gamma(d))}} \paren{2^{2k}} \, \tau(\cY(\bpsi( \cS_k)))/64, \label{eq:kernel_elim_proof_step3_rho}
\end{align}
where \cref{eq:kernel_elim_proof_step3_tri} comes from the fact that we can write $\bpsi(\bx) - \bpsi(\bx^\prime) = \bpsi(\bx) - \bpsi(\bx_\star) + \bpsi(\bx_\star) - \bpsi(\bx^\prime)$ and apply triangle inequality; \cref{eq:kernel_elim_proof_step3_rho} comes from the definition of $\tau(\cY(\bpsi( \cS_k)))$ and the fact that $4 \cdot 2^{-k} \geq \gamma(d)$ when $1 \leq k \leq \ceil*{\log_2(2/\gamma(d))}$.

By comparing terms in \cref{eq:kernel_elim_proof_step3_immediate_bound} and terms in \cref{eq:kernel_elim_proof_step3_rho}, we now relate the complexity bound $N$ with the instance-dependent complexity $\rho^\star(\gamma(d))$ as follows:
\begin{align*}
     N \leq 640 \, \rho^\star(\gamma(d)) \log_2 \paren{\frac{4}{\gamma(d)}}  \log \paren{ \frac{\log_2 \paren{4/\gamma(d)} \abs{\cX}^2 }{\delta} } + (r_d(\zeta) + 1) \log_2  \paren{\frac{4}{\gamma(d)}} .
\end{align*}
\end{proof}

\section{Arm elimination with unknown misspecification}
\label{app:unknown_misspecification}

We consider of version of \cref{alg:active_elim} with unknown mis-specification level $\widetilde \gamma(\cdot)$. As explained in \cref{app:fixed_embedding}, we only consider the case when $\widetilde \gamma(d) > 0$. We (re-)define some notations for this section as follows.

We set $n(d)$ as a critical value for the number of iteration
\begin{align*}
    n(d) = \max \curly{ \widetilde k \in \N: \forall k \leq \widetilde k,  \paren{2 + \sqrt{(1+\zeta) \tau(\cY(\bpsi_d (\cS_k)))}} \widetilde \gamma(d) \leq 2^{-k}/2 },
\end{align*}  
which captures the largest number of iterations such that the induced misspecification error is still well-controlled. Note that $n(d)$ crucially depends on the action set $\bpsi_d(\cX)$ and cannot be calculated due to unknown identities of $\curly*{\cS_k}$.
We then define $\overline \gamma(d)$, which quantifies the optimality of the identified arm as $\overline \gamma(d) = 2 \cdot 2^{-n(d)}$.
We also introduce the notation $g_k(d, \zeta) = (1+\zeta) \tau \paren{ \cY(\bpsi_d(\cS_k)) }$.

\begin{restatable}{proposition}{propRoundAndError}
\label{prop:round_and_error}
The following inequalities hold.
\begin{align*}
        n(d) \geq \floor*{\log_2 \paren*{ 1/\paren*{ 2 (2 + \sqrt{g(d, \zeta)}) \widetilde \gamma(d) } }} \quad \text{ and } \quad \overline \gamma(d) \leq 8 (2 + \sqrt{g(d, \zeta)}) \widetilde \gamma(d) = \gamma(d).
\end{align*}
\end{restatable}
\begin{proof}
Since $\tau(\cY(\bpsi_d(\cS_k))) \leq \tau(\cY(\bpsi_d(\cX)))$ according to Kiefer-Wolfowitz theorem \citep{kiefer1960equivalence} and $\zeta \in [1/10,1/4]$, we have
\begin{align*}
   2 + \sqrt{(1+\zeta) \tau(\cY(\bpsi_d (\cS_k)))} \leq 2 + \sqrt{g(d, \zeta)} ,
\end{align*}
which implies that 
\begin{align*}
 n(d) \geq \floor*{\log_2 \paren*{ 1/\paren*{ 2 (2 + \sqrt{g(d, \zeta)}) \widetilde \gamma(d) } }}.
\end{align*}
The upper bound on $\overline \gamma(d)$ immediately follows.
\end{proof}

\begin{algorithm}[H]
	\caption{Arm Elimination with Fixed Embedding and Unknown Misspecification}
	\label{alg:active_elim_unknown} 
	\renewcommand{\algorithmicrequire}{\textbf{Input:}}
	\renewcommand{\algorithmicensure}{\textbf{Output:}}
	\begin{algorithmic}[1]
		\REQUIRE Action set $\cX$, confidence parameter $\delta$, embedding dimension $d$ and rounding approximation factor $\zeta$.
		\STATE Construct the compressed feature representation $\bpsi(\cX) = \bpsi_d(\cX)$ with the induced mis-specification level $\widetilde \gamma(d)$.
		\STATE Set $\widehat \cS_1 = \cX$, and randomly select $\widehat \bx_\star \in \widehat \cS_1$ as the recommendation.
		\FOR {$k = 1, 2, \dots$}
		\STATE Set $\delta_k = \delta/k^2$.
		\STATE Set $\lambda_k$ and $\tau_k$ be the design and the value of the following optimization problem 
		$$\inf_{\lambda \in \bLambda_{\cX}} \sup_{\by \in \cY( \bpsi(\widehat \cS_k))} \norm*{\by}^2_{\bA_{\bpsi}(\lambda)^{-1}}.$$

		\STATE Set $ N_k = \max \curly*{ \ceil*{ 2^{2k} \, 8 (1+\zeta)\,\tau_k  \log(\abs*{\widehat \cS_k}^2/\delta_k) }  ,r_d(\zeta)}$.
		\STATE Set $\curly{\bx_1, \bx_2, \dots, \bx_{N_k}} = \round (\lambda_k,N_k,d,\zeta)$.
		\STATE Pull arms $\curly{\bx_1, \bx_2, \dots, \bx_{N_k}}$ and receive rewards $\curly*{y_1, \ldots, y_{N_k}}$.
        \STATE Set $\widehat{\btheta}_k = \bA_k^{-1} \bb_k$, where $\bA_k = \sum_{i=1}^{N_k} \bpsi(\bx_i) \bpsi(\bx_i)^{\top}$ and $\bb_k = \sum_{i=1}^{N_k} \bpsi(\bx_i) y_i$.
        \STATE 
        Update $\widehat \bx_\star = \argmax_{\bx \in \widehat \cS_k} \ang*{ \bpsi(\bx), \widehat \btheta_k}$, and eliminate arms with respect to criteria 
        \begin{align*}
            \widehat \cS_{k+1} = \widehat \cS_k \setminus \curly*{ \bpsi(\bx) \in \widehat \cS_k : (\bpsi(\widehat \bx_\star) -\bpsi(\bx))^\top \widehat{\btheta}_k \geq 2^{-k}}.
        \end{align*}
		\ENDFOR 
	\end{algorithmic}
\end{algorithm}

\begin{restatable}{theorem}{thmKernelElimUnknown}
\label{thm:active_elim_unknown}
Suppose $\widetilde \gamma(d)$ is unknown. With probability of at least $1-\delta$, \cref{alg:active_elim_unknown} starts to output a $\overline  \gamma(d)$-optimal arm after at most $N$ samples where
\begin{align*}
N \leq 640 \, \rho^\star(\overline \gamma(d)) \log_2 \paren{n(d)}  \log \paren{ \frac{(n(d) \abs{\cX})^2 }{\delta} } + (r(\zeta) + 1) \, n(d) .
\end{align*}
\end{restatable}

\begin{remark}
An arm with slightly smaller optimality gap $\overline \gamma(d)$ is identified, with slightly larger sample complexity, in the unverifiable way by \cref{alg:active_elim_unknown}. In fact, the same can be achieved by \cref{alg:active_elim} if an upper bound on the number of iteration $k$ is not set. The optimality gap $\overline \gamma(d)$, however, cannot be quantified without knowing $h(\cdot)$. That's the reason for us to set an upper bound on the number of iterations in \cref{alg:active_elim}.
\end{remark}

\begin{proof}
The proof is similar to the proof of \cref{thm:active_elim_adaptive} and \cref{thm:active_elim}, we provide it here for completeness. We decompose the proof into three steps: (1) prove correctness through induction; (2) bound the error probability; and (3) upper bound the sample complexity. For simplicity, we'll use notations $\bpsi(\cdot) \coloneqq \bpsi_{d}(\cdot)$, $\btheta_\star \coloneqq \btheta_{d}$ and $\eta(\cdot) = \eta_{d} (\cdot)$ in the proof.

\textbf{Step 1: The induction.} We define event 
\begin{align*}
    \cE_k = \curly*{\bx_\star \in \widehat \cS_k \subseteq \cS_k},
\end{align*}
and prove through induction that 
\begin{align*}
    \P \paren{ \cE_{k+1} \vert \cap_{i \leq k} \cE_i } \geq 1 - \delta_{k},
\end{align*}
where $\delta_0 \coloneqq 0$ \emph{and} $k \leq n(d)$. Recall that $\cS_k = \curly*{x \in \cX: \Delta_x < 4 \cdot 2^{-k}}$ (with $\cS_1 = \cX$), and thus $\cS_{n(d)+1}$ contains the set of $\overline  \gamma(d)$-optimal arms. Note that as long as $\widehat \cS_{k+1} \subseteq \cS_{k+1}$ holds true, \cref{alg:active_elim_unknown} will have $\widehat \cS_{n} \subseteq \widehat \cS_{k+1}$ for $n > k+1$ due to the nature of the elimination-styled algorithm.

The base case $\curly*{ \bx_\star \in \widehat \cS_1 \subseteq \cS_1 }$ holds by definition of $\widehat \cS_1= \cS_1 = \cX $. We next analyze event $\cE_{k+1}$ conditioned on $\cap_{i \leq  k} \cE_i$.

\textbf{Step 1.1: Concentration under mis-specifications.} This step is almost the same as the Step 1.1 as analyzed in the proof of \cref{thm:active_elim_adaptive} since the unknown mis-specification level $\widetilde \gamma(d)$ is only used in the analysis. As a result, with probability at least $1 - \delta_k$, for any $\by \in \cY(\bpsi(\widehat \cS_k))$,
\begin{align}
    \abs{ \ang{\by, \widehat \btheta_k - \btheta_\star} } & \leq \widetilde \gamma(d) \sqrt{(1+\zeta) \, \tau(\cY( \bpsi(\widehat \cS_k)))} + \norm{\by}_{\bA_k^{-1}} \sqrt{2 \log \paren{ { \abs*{\widehat \cS_k}^2}/{\delta_k}}} \nonumber \\
    & \leq \sqrt{ g_k(d,\zeta)} \widetilde \gamma(d) + \norm{\by}_{\bA_k^{-1}} \sqrt{2 \log \paren{ { \abs*{\widehat \cS_k}^2}/{\delta_k}}} \label{eq:kernel_elim_unknown_concentration},
\end{align}
where \cref{eq:kernel_elim_unknown_concentration} comes from the fact that $\tau(\cY( \bpsi(\widehat \cS_k))) \leq \tau(\cY( \bpsi(\cS_k)))$. In the following, we define
\begin{align*}
   \omega_k (\by) \coloneqq \norm{\by}_{\bA_k^{-1}} \sqrt{2 \log \paren{ { \abs*{\widehat \cS_k}^2}/{\delta_k}}}.
\end{align*}

\textbf{Step 1.2: Correctness.} We prove the result under the event described in \cref{eq:kernel_elim_unknown_concentration}, which holds with (conditional) probability at least $1-\delta_k$. We know that $\bx_\star \in \widehat \cS_k \subseteq \cS_k$ holds conditioned on $\cap_{i \leq k} \cE_i$. We now prove it for iteration $k+1$. Recall the elimination criteria is \begin{align*}
    \widehat \cS_{k+1} = \widehat \cS_k \setminus \curly*{ \bpsi(\bx) \in \widehat \cS_k : (\bpsi(\widehat \bx_\star) -\bpsi(\bx))^\top \widehat{\btheta}_k \geq 2^{-k}},
\end{align*}
where $\widehat \bx_\star = \argmax_{\bx \in \widehat \cS_k} \ang*{ \bpsi(\bx), \widehat \btheta_k}$.

\textbf{Step 1.2.1: $\bx_\star \in \widehat \cS_{k+1}$.} We trivially have $\bx_\star \in \widehat \cS_{k+1}$ if $\widehat \bx_\star =\bx_\star $. Suppose $\widehat \bx_\star \neq \bx_\star $, we have
\begin{align}
    (\bpsi(\widehat\bx_\star ) -\bpsi(\bx_\star))^\top \widehat{\btheta}_k & \leq (\bpsi(\widehat\bx_\star ) -\bpsi(\bx_\star))^\top {\btheta}^\star + \sqrt{ g_k(d,\zeta)} \widetilde \gamma(d) + \omega_k (\bpsi(\widehat \bx_\star) -\bpsi(\bx_\star)) \nonumber \\
    & = h(\widehat\bx_\star) - \eta(\widehat\bx_\star) - (h(\bx_\star) - \eta(\bx_\star)) + \sqrt{ g_k(d,\zeta)} \widetilde \gamma(d) + \omega_k (\bpsi(\widehat \bx_\star) -\bpsi(\bx_\star)) \nonumber \\
    & \leq  - \Delta(\widehat\bx_\star) + \paren{ 2 \widetilde \gamma(d) + \sqrt{ g_k(d,\zeta)} \widetilde \gamma(d) + \omega_k (\bpsi(\widehat \bx_\star) -\bpsi(\bx_\star))} \\
    & < 2^{-k}/2 + 2^{-k}/2 \label{eq:kernel_elim_unknown_correctness} \\
    & = 2^{-k}, \nonumber 
\end{align}
where \cref{eq:kernel_elim_unknown_correctness} comes from the definition of $n(d)$ and the fact that $\omega_k(\bpsi(\widehat \bx_\star) - \bpsi(\bx_\star)) \leq 2^{-k}/2$ (due to the selection of $N_k$ and the guarantee of the rounding procedure in \cref{eq:rounding}). As a result, the optimal arm $\bx_\star$ remains in $\widehat \cS_{k+1}$.

\textbf{Step 1.2.2: $\widehat \cS_{k+1} \subseteq \cS_{k+1}$.} For any $\bx \in \cS_{k+1}^c \cap \widehat \cS_k$, we have $\Delta(\bx) = h(\bx_\star) - h(\bx) \geq 4 \cdot 2^{-(k+1)} = 2 \cdot 2^{-k}$ by definition. Since $(\bpsi(\widehat\bx_\star ) -\bpsi(\bx_\star))^\top \widehat{\btheta}_k \geq 0$, we only need to lower bound $(\bpsi(\bx_\star) -\bpsi(\bx))^\top \widehat{\btheta}_k$.
\begin{align}
    (\bpsi(\bx_\star) -\bpsi(\bx))^\top \widehat{\btheta}_k & \geq (\bpsi(\bx_\star) -\bpsi(\bx))^\top {\btheta}^\star - \sqrt{ g_k(d,\zeta)} \widetilde \gamma(d) - \omega_k(\bpsi(\bx_\star) -\bpsi(\bx)) \nonumber \\
    & = h(\bx_\star) - \eta(\bx_\star) - (h(\bx) - \eta(\bx))- \sqrt{ g_k(d,\zeta)} \widetilde \gamma(d) - \omega_k (\bpsi(\bx_\star) -\bpsi(\bx)) \nonumber \\
    & \geq 2 \cdot 2^{-k} - (2 + \sqrt{ g_k(d,\zeta)}) \widetilde \gamma(d) - \omega_k (\bpsi(\bx_\star) -\bpsi(\bx))  \nonumber \\
    & \geq  2 \cdot 2^{-k} - 2^{-k}/2 - 2^{-k}/2 \label{eq:kernel_elim_unknown_bad_arms} \\
    & \geq 2^{-k} \nonumber ,
\end{align}
where \cref{eq:kernel_elim_unknown_bad_arms} comes from a similar analysis as in \cref{eq:kernel_elim_unknown_correctness}. As a result, any $\bx \in \cS_{k+1}^c \cap \widehat \cS_k$ will be eliminated and we have $\widehat \cS_{k+1} \subseteq \widetilde \cS_{k+1}$.

To summarize, we prove the induction at iteration $k+1$, i.e.,
\begin{align*}
    \P \paren{ \cE_{k+1} \vert \cap_{i< k+1} \cE_{i} } \geq 1 - \delta_k.
\end{align*}

\textbf{Step 2: The error probability.} Using exactly the same argument as appeared in Step 2 in the proof of \cref{thm:active_elim}, we have 
\begin{align*}
    \P \paren{\cap_{i=1}^{k+1} \cE_i} \geq 1 -\delta.
\end{align*}

\textbf{Step 3: Sample complexity upper bound.} Up to iteration $k \leq n(d)$, the sample complexity of \cref{alg:active_elim_unknown} can be upper bounded by
\begin{align}
    N & = \sum_{k=1}^{n(d)} N_k \nonumber  \\
    & = \sum_{k=1}^{n(d)}   \max \curly{ \ceil{ 2^{2k} \, 8 (1+\zeta)\,\tau_k  \log(\abs*{\widehat \cS_k}^2/\delta_k) }  ,r_d(\zeta)} \nonumber  \\
    & \leq \sum_{k=1}^{n(d)}  \paren{ 2^{2k} 10\,\tau(\cY(\bpsi( \cS_k)))  \log(k^2 \abs{\cX}^2/\delta)}  + \paren{r_{d}(\zeta)+1} \, n(d) . \label{eq:kernel_elim_unknown_immediate} 
\end{align}

Following a similar analysis as in \cref{eq:kernel_elim_proof_step3_rho} (note that $4 \cdot 2^{-k} \geq  \overline \gamma(d)$ for $k \leq n(d)$ by definition), we bound the instance dependent sample complexity as follows.
\begin{align*}
    N & \leq   640 \, \rho^\star(\overline \gamma(d)) \log_2 \paren{n(d)}  \log \paren{ \frac{(n(d) \abs{\cX})^2 }{\delta} } + (r(\zeta) + 1) \, n(d) .
\end{align*}
\end{proof}

\section{Omitted materials for \cref{sec:kernel}}

\subsection{Mercer's theorem and corollary}
\label{app:mercer}

Let $\P$ be a non-negative measure over the compact metric space $\cZ$, we define the kernel integral operator as follows
\begin{align*}
    T_{\cK}(f)(\bx) = \int_{\cZ} \cK(\bx, \bz) f(\bz) d \P(\bz).
\end{align*}
\begin{theorem}[Mercer's theorem, see \citet{wainwright2019high}]
\label{thm:mercer}
Suppose $\cZ$ is compact, the kernel function $\cK$ is continuous,  positive semi-definite and satisfies the Hilbert-Schmidt condition. Then there exist a sequence of eigenfunctions $\curly{\phi_j}_{j=1}^\infty$ that form an orthornormal basis of $L^2(\cZ;\P)$, and non-negative eigenvalues $\curly{\mu_j}_{j=1}^\infty$ such that
\begin{align*}
    T_{\cK}(\phi_j) = \mu_j \phi_j \quad \text{ for } j = 1, 2, \dots
\end{align*}
Moreover, the kernel function has the expansion 
\begin{align*}
    \cK(\bx, \bz) = \sum_{j=1}^\infty \mu_j \phi_j(\bx) \phi_j(\bz).
\end{align*}
\end{theorem}

\begin{corollary}[Mercer's corollary, see \citet{wainwright2019high}]
\label{cor:mercer}
Consider a kernel satisfying the conditions of Mercer's theorem with associated eigenfunctions $\curly{\phi_j}_{j=1}^\infty$ and non-negative eigenvalues $\curly{\mu_j}_{j=1}^\infty$. It induces the reproducing kernel Hilbert space 
\begin{align*}
    \cH \coloneqq \curly{h = \sum_{j=1}^\infty \theta_j \phi_j: \text{ for some } \curly*{\theta_j}_{j=1}^\infty \in \ell^2(\N) \text{ with } \sum_{j=1}^\infty \theta_j^2/\mu_j < \infty },
\end{align*}
along with inner product
\begin{align*}
    \ang*{h_1, h_2}_{\cH} \coloneqq \sum_{j=1}^\infty \frac{\ang*{h_1, \phi_j} \ang*{h_2, \phi_j}}{\mu_j},
\end{align*}
where $\ang*{\cdot, \cdot}$ denotes the inner product in $L^2(\cZ;\P)$.
\end{corollary}
\begin{remark}
The RKHS induced by any kernel is unique.
\end{remark}

\subsection{Proof of \cref{thm:kernel_elim}}

\thmKernelElim*

\begin{proof}
We only need calculate the $d_{\eff}(\epsilon)$ under different assumptions, and the sample complexity bounds follow from \cref{thm:active_elim_adaptive}. Since $g(d, \zeta) \leq 5d$, we have
\begin{align*}
    d_{\eff}(\epsilon) \leq  \min \curly{ d \geq 1: (16 + 8\sqrt{5d}) C_{\phi}  \sqrt{\sum\nolimits_{j > d} \mu_j} \leq \epsilon }.
\end{align*}
We analyze an upper bound of $d_{\eff}(\epsilon)$ in the following.

\textbf{Case 1: Polynomial eigenvalue decay.} When $\cK$ has $(C_k, \beta)$-polynomial eigenvalue decay, we have, with $\beta>3$, $\mu_j \leq C_k j^{-\beta}$ for all $j \geq 1$. Since
\begin{align*}
    \sqrt{\sum\nolimits_{j > d} \mu_j} \leq C_k^{1/2} \sqrt{ \int_{x=d}^\infty x^{-\beta} dx } \leq \frac{C_k^{1/2}}{\beta -1} d^{1-\beta},
\end{align*}
we then have 
\begin{align}
    (16 + 8\sqrt{5d}) C_{\phi} \sqrt{\sum\nolimits_{j > d} \mu_j} & \leq  \frac{(16 + 8\sqrt{5d}) C_{\phi} C_k^{1/2} \, d^{1-\beta}}{\beta -1} \nonumber \\
    & \leq \frac{ 16\sqrt{5} C_{\phi} C_k^{1/2} \, d^{3/2-\beta}}{\beta -1}, \label{eq:cor_kernel_elim_proof_poly}
\end{align}
where we use the fact that $2 \leq \sqrt{5d}$ for any $d \geq 1$ in \cref{eq:cor_kernel_elim_proof_poly}. Since \cref{eq:cor_kernel_elim_proof_poly} is decaying in $d$ (recall $\beta > 3/2$), it suffices to set 
\begin{align*}
    d = \ceil{ \paren{ \frac{16 \sqrt{5} C_{\phi} C_k^{1/2} }{ \epsilon (\beta/2 -1) } }^{2/(2\beta - 3)} } = O \paren{{\epsilon}^{-2/(2\beta -3)}}
\end{align*}
to make it smaller than or equal to $\epsilon$. This also gives an upper bound of $d_{\eff}(\epsilon)$.

\textbf{Case 2: Exponential eigenvalue decay.} When $\cK$ has $(C_k, \beta)$-exponential eigenvalue decay, we have, with $\beta>0$, $\mu_j \leq C_k \exp(-\beta j)$ for all $j \geq 1$. Since 
\begin{align*}
    \sqrt{\sum_{j>d} \mu_j} \leq \frac{C_k^{1/2} e^{-\beta d}}{\beta },
\end{align*}
we then have (with $2 \leq \sqrt{5d}$ for any $d \geq 1$)
\begin{align*}
    (16 + 8\sqrt{5d}) C_{\phi}  \sqrt{\sum_{j > d} \mu_j} & \leq \frac{(16 + 8\sqrt{5d})C_{\phi} C_k^{1/2} e^{-\beta d}}{\beta }\\
    & \leq \frac{ 16 \sqrt{5} C_{\phi} C_k^{1/2} \sqrt{d} e^{-\beta d} }{\beta}\\
    & = \frac{ 16 \sqrt{5} C_{\phi} C_k^{1/2}   }{\beta} e^{1/2 \cdot \log d-\beta d}
\end{align*}
To have $(16 + 8\sqrt{5d}) C_{\phi}  \sqrt{\sum_{j > d} \mu_j} \leq \epsilon$, it suffices to have 
\begin{align*}
    e^{ \beta d -  1/2 \cdot \log d} \geq \frac{1}{\tau} \coloneqq \frac{16 \sqrt{5} C_{\phi} C_k^{1/2} }{\epsilon \beta}.
\end{align*}
Since $x \geq 2a \log a \implies x \geq a \log x$ for any $a >0$ \citep{shalev2014understanding}. We know that $ \beta d - 1/2 \cdot \log d \geq \beta d /2 $ when $d \geq 2/\beta \cdot \log (1/\beta)$. Thus, it suffices to have \begin{align*}
    d & \geq \ceil{\max \curly{ \frac{2 \log(1/\beta)}{\beta},  \frac{2 \log \paren{ \frac{16 \sqrt{5} C_\phi C_k^{1/2}}{\epsilon \beta} } }{\beta}} } \\
    & = \ceil{ 2 \log \paren{ \max \curly{ \frac{16 \sqrt{5} C_{\phi} \sqrt{C_k}}{\epsilon \beta}, \frac{1}{\beta} } }/\beta } = O \paren{\log(1/\epsilon)}.
\end{align*}
This gives an upper bound of $d_{\eff}(\epsilon)$.
\end{proof}

\section{Details and proofs for \cref{sec:neural}}\label{app:neural}

\subsection{Omitted details for Section~\ref{sec:neural}}
\noindent\textbf{Details of Gradient Descent}
For the completeness, we present the gradient descent algorithm in Algorithm \ref{algorithm:GD}, which is originally used in \cite{zhou2019neural} as a subroutine for the NeuralUCB algorithm.
\begin{algorithm}
	\caption{Gradient Descent} \label{algorithm:GD} 
	\begin{algorithmic}[1]
	\STATE \textbf{Input:} Regularization parameter $\alpha$, step size $\eta$, number of gradient descent steps $J$, network width $m$, contexts $\{\xb_{i}\}_{i=1}^N$, rewards $\{r_{i}\}_{i=1}^N$, initial parameter $\btheta^{(0)}$.
    \STATE Define $\cL (\btheta) = \sum_{i=1}^N ( f(\xb_{i}; \btheta) - r_{i})^2/2 + m\alpha\|\btheta - \btheta^{(0)}\|_2^2/2$.
    \FOR{$j = 0, \dots, J-1$}
    \STATE $\btheta^{(j+1)} = \btheta^{(j)} - \eta\nabla \cL(\btheta^{(j)})$
    \ENDFOR 
    \STATE \textbf{Return} $\btheta^{(J)}$.
	\end{algorithmic}
\end{algorithm}

\noindent\textbf{Definition of Neural Tangent Kernel} We define the Neural Tangent Kernel (NTK) gram matrix as follows. 
\begin{definition}[\citep{jacot2018neural, zhou2019neural}]\label{def:ntk}
Given arm set $\cX$, define $\Hb^{(l)}, \widetilde\Hb^{(l)}, \bSigma^{(l)} \in \RR^{|\cX|\times |\cX|}$ as follows:
\vspace{-2mm}
\begin{align*}
    \widetilde \Hb^{(1)}(\xb, \xb') &= \bSigma^{(1)}(\xb, \xb') = \la \xb, \xb'\ra,  \Ab^{(l)}(\xb, \xb') = 
    \begin{pmatrix}
    \bSigma^{(l)}(\xb, \xb) & \bSigma^{(l)}(\xb, \xb') \\
    \bSigma^{(l)}(\xb, \xb') & \bSigma^{(l)}(\xb', \xb')
    \end{pmatrix},\notag \\
    \bSigma^{(l+1)}(\xb, \xb') &= 2\EE_{(u, v)\sim N(\zero, \Ab^{(l)}(\xb, \xb'))} \max\{u, 0\}\max\{v, 0\},\notag \\
    \widetilde \Hb^{(l+1)}(\xb, \xb') &= 2\widetilde \Hb^{(l)}(\xb, \xb')\EE_{(u, v)\sim N(\zero, \Ab^{(l)}(\xb, \xb'))} \ind(u \ge 0)\ind(v \ge 0) + \bSigma^{(l+1)}(\xb, \xb').\notag
    \vspace{-1mm}
\end{align*}
Then, $\Hb = (\widetilde \Hb^{(L)} + \bSigma^{(L)})/2$ is called the neural tangent kernel matrix on the arm set $\cX$. 
\end{definition}
The gram matrix $\Hb$ is defined over all arms $\cX$, and $\Hb(\xb, \xb')$ can be regarded as the limitation of $\gb(\xb; \btheta_0)^\top \gb(\xb'; \btheta_0)/m$ when $m$ goes infinity \citep{jacot2018neural}. $\Hb$ plays an important role in our theoretical analysis of Algorithm \ref{alg:neural}.

\noindent\textbf{Formal Version of Theorem~\ref{thm:neural_inf}} Next we present the formal version of Theorem~\ref{thm:neural_inf}.
\begin{theorem}\label{thm:neural}
Under Assumption \ref{assumption:context}, let $\hb: = [h(\xb)] \in \RR^{|\cX|}$ denote the vector consisting of all rewards from arm set $\cX$. Let $S$ be some constant satisfying $S \geq \sqrt{\hb^\top \Hb^{-1}\hb}$. Then there exist constants $C_1, C_2 >0$ such that for any $\epsilon$, if we set the parameters in Algorithm \ref{alg:neural} as follows:
\begin{align}
\alpha =& \min\{1, \log(|\cX|^2)/S\},\ n = \ceil*{\log(1/\epsilon)},\notag\\ 
\bar \epsilon^2 =& \min\{\alpha^2/(r_{|\cX|}^2(\zeta)L),\ \epsilon^2/((1+\zeta)|\cX|S),\ \epsilon^7 \alpha^3/((1+\zeta)^3|\cX|^3\log^3 \paren{{|\cX|^2}/{\delta_n}}L^2)\},\notag \\
A=&d_{\eff}(\bar \epsilon^2/|\cX|),\ \eta_k = C_1(m\alpha + N_kmL)^{-1},\ J_k = \log(C_2\epsilon^2\alpha/(N_k L))/(\eta_k m L)\notag,
\end{align}
then when $m = \text{poly}(|\cX|, L, \lambda_0^{-1}, \log(|\cX|/\delta_k), N_k, \alpha, \bar \epsilon^{-1})$, with probability at least $1-\delta$, $\widehat{\cS}_{n+1}$ only includes arm $\xb$ satisfying $\Delta_{\xb} \leq \epsilon$, and the total sample complexity of Algorithm \ref{alg:neural} is bounded by 
\begin{align}
    N = \widetilde O\bigg((1+\zeta)d_{\eff}(\bar \epsilon^2/|\cX|)/\epsilon^2 + r_{d_{\eff}(\bar \epsilon^2/|\cX|)}(\zeta) \bigg) = \widetilde O\Big(d_{\eff}(\bar \epsilon^2/|\cX|) \epsilon^{-2}\Big).\notag
\end{align}
\end{theorem}
\begin{remark}
Note that $N_k \sim 4^k$, therefore $m$ actually depends on $\epsilon$ polynomially given the selection of $m$.
\end{remark}
\begin{remark}
It is worth noting that the parameter selection depends on $S$, which is the upper bound of the quantity $\sqrt{\hb^\top \Hb^{-1}\hb}$. Since $\sqrt{\hb^\top \Hb^{-1}\hb} \leq \|\hb\|_2\lambda_{\min}(\Hb)^{-1/2} \leq \sqrt{|\cX|/\lambda_0}$, $S = \sqrt{|\cX|/\lambda_0}$ is always a valid upper bound. Moreover, if the reward function $h$ belongs to the RKHS space spanned by NTK with norm $\|h\|_{\cH}$, then we can set a tighter upper bound $S = \|h\|_{\cH}$ which ensures $S \geq \sqrt{\hb^\top \Hb^{-1}\hb}$ \citep{zhou2019neural}. 
\end{remark}
\subsection{Proof of Theorem \ref{thm:neural} }\label{app:neural_222}
To prove Theorem \ref{thm:neural}, we need the following lemmas. For simplicity, we first denote $\bar\Zb_k, \Zb_k \in \RR^{p \times p} \bar\bbb_k \in \RR^p, \Ab_k \in \RR^{d_k \times d_k}, \bbb_k, \widehat\btheta_k \in \RR^{d_k}$ as follows.
\begin{align}
        &\Zb_k = \alpha \Ib + \sum_{i=1}^{N_k}\gb(\xb_i; \btheta_{k-1})\gb(\xb_i; \btheta_{k-1})^\top/m,\ \bar\Zb_k = \alpha \Ib + \sum_{i=1}^{N_k}\gb(\xb_i; \btheta_{0})\gb(\xb_i; \btheta_{0})^\top/m,\notag \\
        &\bar\bbb_k = \sum_{i=1}^{N_k} y_i \gb(\xb_i; \btheta_{0})/\sqrt{m},\notag\\
    &\Ab_k = \alpha \Ib + \sum_{i=1}^{N_k} \bpsi_{d_k}(\bx_i) \bpsi_{d_k}(\bx_i)^{\top},\ \bbb_k = \sum_{i=1}^{N_k} \bpsi_{d_k}(\xb_i) y_i,\ \widehat\btheta_k = \Ab_k^{-1}\bbb_k,
\end{align}
where $\{\xb_i\}_{i=1}^{N_k}$ are arms selected at round $k$. We first present several lemmas which provide error bounds between neural network functions and their linear estimations.

\begin{lemma}[Theorem 3.1, \citep{arora2019exact}]\label{lemma:ntkconcen}
Fix $\epsilon>0$ and $\delta \in (0,1)$. Suppose that 
\begin{align}
    m  = \Omega\bigg(\frac{L^6|\cX|^{16}\log (|\cX|^2L/\delta)}{\epsilon^8}\bigg),\notag
\end{align}
then with probability at least $1-\delta$, for any $\xb, \xb' \in \cX$, we have
\begin{align}
    |\la \gb(\xb; \btheta_0), \gb(\xb';\btheta_0)\ra/m - \Hb(\xb, \xb')| \leq \epsilon^2/(2|\cX|^4).\label{lemma:ntkconcen_1}
\end{align}
\end{lemma}

\begin{lemma}[Lemma 4.1, \citep{cao2019generalization2}]\label{lemma:cao_functionvalue}
There exist positive constants $\bar C_1, \bar C_2, C_f$ such that for any $\delta > 0$, if $\tau$ satisfies that
\begin{align}
    \bar C_1m^{-3/2}L^{-3/2}[\log(|\cX|L^2/\delta)]^{3/2}\leq\tau \leq \bar C_2 L^{-6}[\log m]^{-3/2},\notag
\end{align}
then with probability at least $1-\delta$,
for all $\widetilde\btheta, \widehat\btheta$ satisfying $ \|\widetilde\btheta - \btheta_0\|_2 \leq \tau, \|\widehat\btheta - \btheta_0\|_2 \leq \tau$ and $\xb\in \cX$ we have
\begin{align}
    \Big|f(\xb; \widetilde\btheta) - f(\xb;  \widehat\btheta) - \la \gb(\xb; 
    \widehat\btheta),\widetilde\btheta - \widehat\btheta\ra\Big| \leq C_f\tau^{4/3}L^3\sqrt{m \log m}.\notag
\end{align}
\end{lemma}

\begin{lemma}[Lemma B.3, \citep{cao2019generalization2}]\label{lemma:cao_boundgradient}
There exist positive constants $\bar C_1, \bar C_2, C_{g, 1}, C_{g,2}$ such that for any $\delta > 0$, if $\tau$ satisfies that
\begin{align}
    \bar C_1m^{-3/2}L^{-3/2}[\log(|\cX|L^2/\delta)]^{3/2}\leq\tau \leq \bar C_2 L^{-6}[\log m]^{-3/2},\notag
\end{align}
then with probability at least $1-\delta$,
for any $\|\btheta - \btheta_0\|_2 \leq \tau$ and $\xb \in \cX$ we have 
\begin{align}
    \|\gb(\xb; \btheta)\|_2\leq C_{g, 1}\sqrt{mL},\  \|\gb(\xb; \btheta) - \gb(\xb; \btheta_0)\|_2 \leq C_{g,2}\tau^{1/3}\sqrt{m \log m}L^{7/2}/2.\notag
\end{align}
\end{lemma}

\begin{lemma}[Lemma 5.1, \citep{zhou2019neural}]\label{lemma:equal}
There exists a positive constant $\bar C$ such that for any $\delta \in (0,1)$, if $m \geq \bar C|\cX|^4L^6\log(|\cX|^2L/\delta)/\lambda_0^4$, then with probability at least $1-\delta$, there exists a $\btheta_{\star} \in \RR^p$ such that for any $\xb \in \cX$, 
\begin{align}
    &h(\xb) = \la \gb(\xb; \btheta_0), \btheta_\star - \btheta_0\ra,\ \sqrt{m}\|\btheta_\star - \btheta_0\|_2 \leq \sqrt{2}S. \label{lemma:equal_0}
\end{align}
\end{lemma}

\begin{lemma}[Lemma B.2, \citep{zhou2019neural}]\label{lemma:newboundreference}
There exist constants $\{\bar C_i\}_{i=1}^5>0$ such that for any $\delta > 0$, if for all $N_k$, $\eta, m$ satisfy
\begin{align}
    &2\sqrt{N_k/(m\alpha)} \geq \bar C_1m^{-3/2}L^{-3/2}[\log(|\cX|L^2/\delta)]^{3/2},\notag \\
    &2\sqrt{N_k/(m\alpha)} \leq \bar C_2\min\big\{ L^{-6}[\log m]^{-3/2},\big(m(\alpha\eta)^2L^{-6}N_k^{-1}(\log m)^{-1}\big)^{3/8} \big\},\notag \\
    &\eta \leq \bar C_3(m\alpha + N_kmL)^{-1},\notag \\
    &m^{1/6}\geq \bar C_4\sqrt{\log m}L^{7/2}N_k^{7/6}\alpha^{-7/6}(1+\sqrt{N_k/\alpha}),\notag
\end{align}
then with probability at least $1-\delta$,
we have that $\|\btheta_k -\btheta_0\|_2 \leq  2\sqrt{N_k/(m\alpha )}$ and
\begin{align}
   &\|\btheta_k - \btheta_0 - \bar \Zb_k^{-1}\bar\bbb_t/\sqrt{m}\|_2  \notag \\
   &\leq (1- \eta m \alpha)^{J/2} \sqrt{N_k/(m\alpha)} + \bar C_5m^{-2/3}\sqrt{\log m}L^{7/2}N_k^{5/3}\alpha^{-5/3}(1+\sqrt{N_k/\alpha}).\notag
\end{align}
\end{lemma}

\begin{lemma}[Lemma B.3, \citep{zhou2019neural}]\label{lemma:newboundz}
There exist constants $\{\bar C_i\}_{i=1}^5>0$ such that for any $\delta > 0$, if $m$ satisfies that
\begin{align}
    \bar C_1m^{-3/2}L^{-3/2}[\log(|\cX|L^2/\delta)]^{3/2}\leq2\sqrt{N_k/(m\lambda)} \leq \bar C_2 L^{-6}[\log m]^{-3/2},\notag
\end{align}
then with probability at least $1-\delta$, we have
\begin{align}
    &\|\bar\Zb_k -  \Zb_k\|_F \leq \bar C_3m^{-1/6}\sqrt{\log m}L^4N_k^{7/6}\alpha^{-1/6},\notag
\end{align}
\end{lemma}

Based on previous lemmas, we can define the random event $\cE^{\text{neural}}_k$ as follows:
\begin{align}
    &\cE^{\text{neural}}_k: = \bigg\{\exists\text{ positive constants }\{C_i\}_{i=1}^5,\ \forall \xb, \xb' \in \cX, \notag \\
    & |\la \gb(\xb; \btheta_0), \gb(\xb';\btheta_0)\ra/m - \Hb(\xb, \xb')| \leq \epsilon^2/(2|\cX|^4),\notag \\
    & \Big|f(\xb; \btheta_k) - \la \gb(\xb; 
    \btheta_0),\btheta_k - \btheta_0\ra\Big|
    \leq C_1N_k^{2/3}m^{-1/6}\alpha^{-2/3}L^3\sqrt{\log m},\notag \\
    &\|\gb(\xb; \btheta_0)\|_2,\ \|\gb(\xb; \btheta_{k-1})\|_2 \leq C_2\sqrt{mL},\notag \\
    &\|\gb(\xb; \btheta_{k-1}) - \gb(\xb; \btheta_0)\|_2 \leq C_3N_{k-1}^{1/6}\alpha^{-1/6}m^{1/6}\sqrt{\log m}L^{7/2},\label{gradientbound} \\
    &  \exists \btheta_\star \in \RR^p,\ h(\xb) = \la \gb(\xb; \btheta_0), \btheta_\star - \btheta_0\ra,\ \sqrt{m}\|\btheta_\star - \btheta_0\|_2 \leq \sqrt{2}S,\notag \\
    & \|\btheta_k - \btheta_0 - \bar \Zb_k^{-1}\bar\bbb_t/\sqrt{m}\|_2  \notag \\
    &\qquad \leq (1- \eta m \alpha)^{J/2} \sqrt{N_k/(m\alpha)} + C_4m^{-2/3}\sqrt{\log m}L^{7/2}N_k^{5/3}\alpha^{-5/3}(1+\sqrt{N_k/\alpha}),\notag \\
    & \|\bar\Zb_k -  \Zb_k\|_F \leq C_5m^{-1/6}\sqrt{\log m}L^4N_k^{7/6}\alpha^{-1/6}.\bigg\}\notag
\end{align}
It can be verified that when $m = \text{poly}(|\cX|, L, \lambda_0^{-1}, \log(|\cX|/\delta_k), N_k, \alpha, \bar\epsilon^{-1})$, $\PP(\cE^{\text{neural}}_k) \geq 1-\delta_k$.

Suppose event $\cE^{\text{neural}}_k$ holds, then we have the following lemmas.
\begin{lemma}\label{lemma:neural_main}
Let $m = \text{poly}(|\cX|, L, \lambda_0^{-1}, \log(|\cX|/\delta_k), N_k, \alpha, \bar\epsilon^{-1})$, then under event $\cE^{\text{neural}}_k$, we have
\begin{align}
    &|f(\xb; \btheta_k) - \la \bpsi_{d_k}(\xb), \widehat\btheta_k\ra| \notag \\
    &\leq C[(1 - \eta m \alpha)^{J/2} \sqrt{N_kL/\alpha} + 4N_k^3\bar\epsilon^2 L^2/\alpha^3 + N_k^2 L\bar\epsilon^2/\alpha^2  + \bar\epsilon^2N_k/\alpha].
\end{align}
\end{lemma}
\begin{proof}
See Appendix \ref{proof:lemma:neural_main}.
\end{proof}

\begin{lemma}\label{lemma:lowandh}
Let $m = \text{poly}(|\cX|, L, \lambda_0^{-1}, \log(|\cX|/\delta_k), N_k, \alpha, \bar\epsilon^{-1})$, then under event $\cE^{\text{neural}}_k$, there exists an $\widehat\btheta_k^\star \in \RR^{d_k}$ such that $|h(\xb) - \la \bpsi_{d_k}(\xb), \widehat \btheta^\star_k\ra| \leq 2S\bar\epsilon,\ \|\widehat\btheta^\star_k\|_2 \leq \sqrt{2}S$.
\end{lemma}
\begin{proof}
See Appendix \ref{proof:lemma:lowandh}.
\end{proof}
We use $\widehat\btheta_k^\star$ to denote the underlying vector, and we avoid using double subscripts for the simplicity of notions.

\begin{lemma}\label{lemma:boundd}
Let $m = \text{poly}(|\cX|, L, \lambda_0^{-1}, \log(|\cX|/\delta_k), N_k, \alpha, \bar\epsilon^{-1})$, then under event $\cE^{\text{neural}}_k$,  we have $d_k \leq d_{\eff}(\bar\epsilon^2/|\cX|)$.
\end{lemma}
\begin{proof}
See Appendix \ref{proof:lemma:boundd}.
\end{proof}

Now we begin to prove Theorem \ref{thm:neural}. 
\begin{proof}[Proof of Theorem \ref{thm:neural}]
We follow the main steps of the proof of Theorem \ref{thm:active_elim}. 

\noindent\textbf{Step 1.1: }
By Lemma \ref{lemma:lowandh}, we know that there exists $\widehat\btheta_k^\star$ such that for any $\xb \in \cX$, $h(\xb) = \la \bpsi_{d_k}(\xb), \widehat \btheta_k^\star\ra + \eta(\xb)$, where $|\eta(\xb)| \leq 2S\bar\epsilon$. Therefore, for any $\yb \in \cY(\bpsi_{d_k}(\widehat\cS_k))$, we have
\begin{align}
    \abs{ \ang{\yb, \widehat \btheta_k - \widehat\btheta_k^\star} }& = \abs{\yb^\top \Ab_k^{-1} \sum_{j=1}^{N_k} \bpsi_{d_k}(\bx_i) y_i - \yb^\top \widehat\btheta_k^\star}  \nonumber \\
    & = \abs{\yb^\top \Ab_k^{-1} \sum_{j=1}^{N_k} \bpsi_{d_k}(\bx_i) \paren{ \bpsi_{d_k}(\bx_i)^\top \widehat\btheta_k^\star + \eta(\bx_i) + \xi_i } - \yb^\top \widehat\btheta_k^\star} \nonumber\\
    & = \abs{\yb^\top \Ab_k^{-1} \sum_{j=1}^{N_k} \bpsi_{d_k}(\bx_i) \paren{\eta(\bx_i) + \xi_i } - \alpha \yb^\top \Ab_k^{-1}\widehat\btheta_k^\star} \nonumber \\
    & \leq \abs{\yb^\top \Ab_k^{-1} \sum_{j=1}^{N_k} \bpsi_{d_k}(\bx_i) \eta(\bx_i)} + \abs{ \sum_{j=1}^{N_k} \yb^\top \Ab_k^{-1} \bpsi_{d_k}(\bx_i) \xi_i } + \alpha\abs{\yb^\top \Ab_k^{-1}\widehat\btheta_k^\star} . \label{eq:551}
\end{align}
Following the proof of Theorem \ref{thm:active_elim}, the first and second terms in \eqref{eq:551} can be bounded as follows, with probability at least $1-\delta_k$ for any $\yb \in \cY(\bpsi_{d_k}(\widehat\cS_k))$:
\begin{align}
    &\abs{\yb^\top \Ab_k^{-1} \sum_{j=1}^{N_k} \bpsi_{d_k}(\bx_i) \eta(\bx_i)} \leq 2S\bar\epsilon\sqrt{(1+\zeta) \, \tau(\cY( \bpsi_{d_k}(\widehat \cS_k)))}\notag \\
    &\abs{ \sum_{j=1}^{N_k} \yb^\top \Ab_k^{-1} \bpsi_{d_k}(\bx_i) \xi_i } \leq \|\yb\|_{\Ab_k^{-1}}\sqrt{2 \log \paren{ {\abs*{\widehat \cS_k}^2}/{\delta_k}}} 
\end{align}
The last term in \eqref{eq:551} can be bounded as
\begin{align}
    \alpha\abs{\yb^\top \Ab_k^{-1}\widehat\btheta_k^\star} \leq \sqrt{\alpha}\|\widehat\btheta_k^\star\|_2 \|\yb\|_{\Ab_k^{-1}} \leq\sqrt{2\alpha}S\|\yb\|_{\Ab_k^{-1}} .\notag
\end{align}
Therefore, we have
\begin{align}
    &\big|\big\la \yb, \widehat\btheta_k - \widehat\btheta_k^\star\big\ra\big|\notag \\
    &\leq 2S\bar\epsilon \sqrt{(1+\zeta) \, \tau(\cY( \bpsi_{d_k}(\widehat \cS_k)))} + \norm{\yb}_{\Ab_k^{-1}} \Big(\sqrt{2 \log \paren{ { \abs*{\widehat \cS_k}^2}/{\delta_k}}} + \sqrt{2\alpha}S\Big)\notag \\
    &: = \widetilde\omega_k(\bpsi_{d_k}(\xb) - \bpsi_{d_k}(\xb')).\label{eq:551.1}
\end{align}
Similar to the proof of Theorem \ref{thm:active_elim}, since by the property of rounding procedure, we have
\begin{align}
    \norm{\yb}_{\Ab_k^{-1}} \leq \max_{\yb' \in \cY(\bpsi_{d_k}(\widehat\cS_k))} \norm{\yb'}_{\Ab_k^{-1}} \leq \sqrt{(1+\zeta)/N_k}\|\yb'\|_{\Ab_{\bpsi_{d_k}}(\lambda_k)^{-1}} \leq 2\sqrt{(1+\zeta)d_k/N_k}.\label{eq:551.2}
\end{align}
Meanwhile, according to Kiefer–Wolfowitz Theorem, we have
\begin{align}
    \tau(\cY( \bpsi_{d_k}(\widehat \cS_k))) = \inf_{\lambda \in \bLambda_{\cX}} \sup_{\yb \in \cY(\bpsi_{d_k}(\widehat\cS_k))}\|\yb\|_{\Ab_{\bpsi_{d_k}}(\lambda)^{-1}}^2 \leq \inf_{\lambda \in \bLambda_{\cX}} \sup_{\yb \in \cY(\bpsi_{d_k}(\cX))}\|\yb\|_{\Ab_{\bpsi_{d_k}}(\lambda)^{-1}}^2 \leq 4d_k.\label{eq:551.3}
\end{align}
Submitting \eqref{eq:551.2} and \eqref{eq:551.3} into \eqref{eq:551.1}, we have
\begin{align}
    \big|\big\la \yb, \widehat\btheta_k - \widehat\btheta_k^\star\big\ra\big| &\leq 4S\bar\epsilon \sqrt{(1+\zeta) d_k} +2\sqrt{(1+\zeta)d_k/N_k} \Big(\sqrt{2 \log \paren{ { |\cX|^2}/{\delta_k}}} + \sqrt{2\alpha}S\Big)\label{eq:551.4},
\end{align}

\noindent\textbf{Step 1.2: }
We show that $\xb_\star \in \widehat\cS_{k+1}$ for any $0 \leq k \leq n-1$ first. For any $\widehat\xb \in \widehat\cS_k$, we have
\begin{align}
    &f(\widehat\xb) - f(\xb_\star) \notag \\
    &\leq \la \bpsi_{d_k}(\widehat\xb), \widehat\btheta_k\ra - \la \bpsi_{d_k}(\xb_\star), \widehat\btheta_k\ra + \underbrace{|f(\widehat\xb) - \la \bpsi_{d_k}(\widehat\xb), \widehat\btheta_k\ra| + |f(\xb_\star) - \la \bpsi_{d_k}(\xb_\star), \widehat\btheta_k\ra|}_{I_1}\notag \\
    & \leq \widetilde\omega_k(\bpsi_{d_k}(\widehat\xb) - \bpsi_{d_k}(\xb_\star)) + h(\widehat\xb) - \eta(\widehat\xb) - h(\xb_\star) + \eta(\xb_\star) + I_1\notag \\
    & \leq 4S\bar\epsilon + \widetilde\omega_k(\bpsi_{d_k}(\widehat\xb) - \bpsi_{d_k}(\xb_\star)) + I_1, \label{eq:551.401}
\end{align}
To further bound \eqref{eq:551.401}, we have
\begin{align}
    &4S\bar\epsilon + \widetilde\omega_k(\bpsi_{d_k}(\widehat\xb) - \bpsi_{d_k}(\xb_\star)) + I_1\notag \\
    & \leq 4S\bar\epsilon \notag \\
    &\qquad + 4S\bar\epsilon \sqrt{(1+\zeta) d} +2\sqrt{(1+\zeta)d/N_k} \Big(\sqrt{2 \log \paren{ { |\cX|^2}/{\delta_k}}} + \sqrt{2\alpha}S\Big) \notag \\
    &\qquad + 2C[(1 - \eta m \alpha)^{J/2} \sqrt{N_kL/\alpha} + 4N_k^3\bar\epsilon^2 L^2/\alpha^3 + N_k^2 L\bar\epsilon^2/\alpha^2  + \bar\epsilon^2N_k/\alpha]\notag \\
    & \leq 8S\bar\epsilon \sqrt{(1+\zeta) d_{\eff}(\bar\epsilon^2/|\cX|)} + 4\sqrt{(1+\zeta)d_{\eff}(\bar\epsilon^2/|\cX|)/N_k} \sqrt{ \log \paren{ { |\cX|^2}/{\delta_k}}}\notag \\
    &\qquad + 2C[(1 - \eta m \alpha)^{J/2} \sqrt{N_kL/\alpha} + 6N_k^3\bar\epsilon^2 L^2/\alpha^3]\label{eq:551.42}
\end{align}
where the first inequality holds due to Lemma \ref{lemma:neural_main} and \eqref{eq:551.4}, the second inequality holds since $\alpha \leq \min\{1, \log(|\cX|^2)/S\}$. Now, according to the choice of $N_k$, we have
\begin{align}
    4\sqrt{(1+\zeta)\log \paren{ { |\cX|^2}/{\delta_k}}d_{\eff}(\bar\epsilon^2/|\cX|)/N_k} \leq 2^{-k}/8,
\end{align}
where the inequality holds due to the selection of $N_k$. 
According to the choice of $\bar\epsilon$, we have
\begin{align}
    8S\bar\epsilon \sqrt{(1+\zeta) d_{\eff}(\bar\epsilon^2/|\cX|)} \leq \epsilon/8,\notag
\end{align}
According to the choice of $N_k, \bar\epsilon$, we have
\begin{align}
    &12CN_k^3\bar\epsilon^2 L^2/\alpha^3  \notag \\
    &\leq   12C\cdot 1024^3(1+\zeta)^3\log^3 \paren{{|\cX|^2}/{\delta_k}}  d^3 (\bar\epsilon^2/|\cX|) \epsilon^{-6}\bar\epsilon^2 L^2/\alpha^3 + 12Cr_d(\zeta)^3 \bar\epsilon^2 L^2/\alpha^3   \notag \\
    &\leq \epsilon/8, 
\end{align}
According to the choice of $J$, we have
\begin{align}
    2C(1 - \eta m \alpha)^{J/2} \sqrt{N_kL/\alpha} \leq \epsilon/8. 
\end{align}
Therefore, substituting these bounds into \eqref{eq:551.42}, we have 
\begin{align}
    4S\bar\epsilon + \widetilde\omega_k(\bpsi_{d_k}(\widehat\xb) - \bpsi_{d_k}(\xb_\star)) + I_1 \leq 2^{-k}/8 + 3\epsilon/8, \label{end:1}
\end{align}
Therefore, $\xb_\star \in \widehat\cS_{k+1}$.

\noindent\textbf{Step 1.3: }
Next, we show that for any $k = n$, any $\xb \in \widehat\cS_k$ satisfying $\Delta(\xb) \geq \epsilon$ will be eliminated. In other words, all the arms $\xb'$ remaining in $\widehat\cS_{k+1}$ satisfy $\Delta(\xb') \leq \epsilon$. Suppose $\Delta(\xb) \geq \epsilon$ and $\xb \in \widehat\cS_k$, then we have $\Delta(\xb) = \epsilon \geq 4\cdot 2^{-(k+1)} = 2\cdot 2^{-k}$ by the selection of $k$. Now we have
\begin{align}
    f(\xb_\star) - f(\xb) &\geq \la \bpsi_{d_k}(\xb_\star), \widehat\btheta_k\ra - \la \bpsi_{d_k}(\xb), \widehat\btheta_k\ra - \underbrace{|f(\xb) - \la \bpsi_{d_k}(\xb), \widehat\btheta_k\ra| + |f(\xb_\star) - \la \bpsi_{d_k}(\xb_\star), \widehat\btheta_k\ra|}_{I_2}\notag \\
    & \geq h(\xb_\star) - h(\xb) - \eta(\xb_\star) + \eta(\xb) - \widetilde\omega_k(\bpsi_{d_k}(\xb_\star) - \bpsi_{d_k}(\xb)) - I_2\notag \\
    & \geq \Delta(\xb) - (4S\bar\epsilon + \widetilde\omega_k(\bpsi_{d_k}(\xb) - \bpsi_{d_k}(\xb_\star)) + I_2),\label{eq:551.41} 
\end{align}
Similar to \eqref{end:1}, \eqref{eq:551.41} can be bounded by $2^{-k}/8 + 3\epsilon/8$. Therefore, when $k > n = \log(1/\epsilon)$, we have $f(\xb_\star) - f(\xb) \geq \epsilon - \epsilon/2 = \epsilon/2 > 2^{-k}/8 + 3\epsilon/8$, which suggests that $\xb$ will be eliminated. 

\noindent\textbf{Step 1.4: }
The probability for all the events hold simultaneously in round $k = 1,\dots, n$ can be bounded by a union bound, the same as \eqref{eq:kernel_elim_proof_step2}.

\noindent\textbf{Step 1.5: }
Finally we bound the total sample complexity, which is
\begin{align}
    N &= \sum_{k=1}^n N_k \notag \\
    &= \widetilde O\bigg((1+\zeta)d_{\eff}(\bar\epsilon^2/|\cX|)\sum_{k=1}^n (4^k + r_{d_k}(\zeta)) \bigg) \notag \\
    &= \widetilde O\bigg((1+\zeta)d_{\eff}(\bar\epsilon^2/|\cX|)/\epsilon^2 + r_{d_{\eff}(\bar\epsilon^2/|\cX|)}(\zeta) \bigg).\notag
\end{align}

\end{proof}

\subsection{Proof of Lemmas}\label{sec:neurallemmaproof}
In this section we propose the proofs of lemmas in Appendix~\ref{app:neural_222}. We first propose the following proposition, which can be directly derived by the construction rule of $\bpsi_{d_k}$ introduced in Algorithm~\ref{alg:neural}. 
\begin{proposition}\label{prop:neural}
For any $\xb \in \cX$, we have
\begin{align}
    \bpsi_{d_k}(\xb) = \sum_{j=1}^{d} a_{\xb, j}e_j \eb_j,\ \gb(\xb; \btheta_{k-1})/\sqrt{m} = \sum_{j=1}^{|\cX|} a_{\xb, j}e_j \bbb_j = \Bb_{\cX}^\top[\bpsi(\xb); \bepsilon_{\xb}],\ \|\bepsilon_{\xb}\|_2 \leq \bar\epsilon,\notag
\end{align}
where $\bbb_j \in \RR^p$ are orthogonal unit vectors in $p$-dimensional space, $\Bb_{\cX} = [\bbb_1^\top;,\dots, \bbb_{|\cX|}^\top] \in \RR^{|\cX|\times p}$. 
\end{proposition}
In the following proofs, for simplicity, let $\bpsi(\cdot)$ denote $\bpsi_{d_k}(\cdot)$.

\subsubsection{Proof of Lemma \ref{lemma:neural_main}}\label{proof:lemma:neural_main}

In detail, to prove Lemma \ref{lemma:neural_main}, we need the following lemmas. 
\begin{lemma}\label{lemma:linearize}
Let $m = \text{poly}(|\cX|, L, \lambda_0^{-1}, \log(|\cX|/\delta_k), N_k, \alpha, \bar\epsilon^{-1})$, then under event $\cE^{\text{neural}}_k$, we have
\begin{align}
     |f(\xb; \btheta_k) - \la \gb(\xb; \btheta_0), \bar \Zb_k^{-1}\bar\bbb_t/\sqrt{m}\ra| \leq C(1 - \eta m \lambda)^{J/2} \sqrt{N_kL/\alpha}.\label{lemma:linearize_0}
\end{align}
\end{lemma}
\begin{proof}[Proof of Lemma \ref{lemma:linearize}]
By Lemma \ref{lemma:cao_functionvalue}, we have for all $\xb \in \cX$, 
\begin{align}
\Big|f(\xb; \btheta_k) - \la \gb(\xb; 
    \btheta_0),\btheta_k - \btheta_0\ra\Big|
    & = \Big|f(\xb; \btheta_k) - f(\xb;  \btheta_0) - \la \gb(\xb; 
    \btheta_0),\btheta_k - \btheta_0\ra\Big| \notag \\
    &\leq 3C_fN_k^{2/3}m^{-1/6}\alpha^{-2/3}L^3\sqrt{\log m},\notag
\end{align}
where we use the fact that $f(\xb_j;  \btheta_0) = 0$ due to the initialization scheme adapted by Algorithm \ref{alg:neural}, and $\|\btheta_k - \btheta_0\|_2 \leq 2\sqrt{N_k/(m\alpha)}$ from Lemma \ref{lemma:newboundreference}. Therefore, we have
\begin{align}
&|f(\xb; \btheta_k) - \la \gb(\xb; \btheta_0), \bar \Zb_k^{-1}\bar\bbb_t/\sqrt{m}\ra|\notag \\
& \leq|f(\xb; \btheta_k) - \la \gb(\xb; \btheta_0), \btheta_k - \btheta_0\ra| +   \|\btheta_k - \btheta_0 - \bar \Zb_k^{-1}\bar\bbb_t/\sqrt{m}\|_2\| \gb(\xb; \btheta_0)\|_2\notag \\
&\leq 3C_fN_k^{2/3}m^{-1/6}\alpha^{-2/3}L^3\sqrt{\log m} \notag \\
&\qquad +[(1- \eta m \alpha)^{J/2} \sqrt{N_k/(m\alpha)} + \bar C_5m^{-2/3}\sqrt{\log m}L^{7/2}N_k^{5/3}\alpha^{-5/3}(1+\sqrt{N_k/\alpha})]C_{g, 1}\sqrt{mL} \notag \\
    & \leq C(1 - \eta m \alpha)^{J/2} \sqrt{N_kL/\alpha},\notag
\end{align}
where the first inequality holds due to triangle inequality, the second one holds due to Lemma \ref{lemma:cao_boundgradient} and Lemma \ref{lemma:newboundreference}, the last one holds since $m = \text{poly}(L, N_{k-1}, \alpha, \log(1/\delta))$. 
\end{proof}

We also need the following lemma.

\begin{lemma}\label{lemma:neural_truncate}
Suppose $\alpha \geq 2C_{g,1}N_k\bar\epsilon\sqrt{L}$ and $m = \text{poly}(L, N_{k-1}, \alpha, \log(1/\delta))$, then with probability at least $1-\delta$, for all $\xb \in \cX$, we have
\begin{align}
    &\la \bpsi(\xb), \widehat\btheta_k\ra - \la \bar \Zb_k^{-1}\bar\bbb_t/\sqrt{m}, \gb(\xb; \btheta_0)\ra  \leq C_{\bpsi} (N_kL\alpha/\lambda_k^2 + \bar\epsilon^2/\lambda_k^2N_k^2 L  + \bar\epsilon^2N_k/\alpha).\label{lemma:neural_truncate_0}
\end{align}
\end{lemma}
\begin{proof}[Proof of Lemma \ref{lemma:neural_truncate}]
For each $\xb \in \cX$, we have
\begin{align}
    &\la \bpsi(\xb), \widehat\btheta_k\ra - \la \bar \Zb_k^{-1}\bar\bbb_t/\sqrt{m}, \gb(\xb; \btheta_0)\ra \notag \\
    &= \bigg|\bigg\la \bpsi(\xb), \Ab_k^{-1}\sum_{i=1}^{N_k} \bpsi(\xb_i) y_i\bigg\ra - \bigg\la \frac{\gb(\xb; \btheta_0)}{\sqrt{m}}, \bar \Zb_k^{-1}\sum_{i=1}^{N_k} \sum_{i=1}^{N_k} y_i \gb(\xb_i; \btheta_{0})/\sqrt{m}\bigg\ra\bigg|\notag \\
    & \leq \sum_{i=1}^{N_k} \bigg|\bigg\la \bpsi(\xb), \Ab_k^{-1} \bpsi(\xb_i)\bigg\ra - \bigg\la \frac{\gb(\xb; \btheta_0)}{\sqrt{m}}, \bar \Zb_k^{-1}   \gb(\xb_i; \btheta_{0})/\sqrt{m}\bigg\ra\bigg|,\label{eq:111}
\end{align}
where the last line holds due to triangle inequality. 
Then, to bound \eqref{eq:111}, we first decompose it to two terms by triangle inequality:
\begin{align}
    &\bigg|\bigg\la \bpsi(\xb), \Ab_k^{-1} \bpsi(\xb_i)\bigg\ra - \bigg\la \frac{\gb(\xb; \btheta_0)}{\sqrt{m}}, \bar \Zb_k^{-1}   \gb(\xb_i; \btheta_{0})/\sqrt{m}\bigg\ra\bigg|\notag \\
    &\leq \underbrace{\bigg|\bigg\la \bpsi(\xb), \Ab_k^{-1} \bpsi(\xb_i)\bigg\ra - \bigg\la \frac{\gb(\xb; \btheta_{k-1})}{\sqrt{m}},  \Zb_k^{-1}   \gb(\xb_i; \btheta_{k-1})/\sqrt{m}\bigg\ra\bigg|}_{I_1}\notag \\
    &\qquad + \underbrace{\bigg|\bigg\la \frac{\gb(\xb; \btheta_0)}{\sqrt{m}}, \bar \Zb_k^{-1}   \gb(\xb_i; \btheta_{0})/\sqrt{m}\bigg\ra - \bigg\la \frac{\gb(\xb; \btheta_{k-1})}{\sqrt{m}},  \Zb_k^{-1}   \gb(\xb_i; \btheta_{k-1})/\sqrt{m}\bigg\ra\bigg|}_{I_2}.
\end{align}
\noindent\textbf{To bound $I_1$: }
Let $\xb'$ denote $\xb_i$. First we write $\big\la \gb(\xb; \btheta_{k-1}),  \Zb_k^{-1}   \gb(\xb'; \btheta_{k-1})\big\ra/m$ as products of block matrices, as follows:
\begin{align}
    &\bigg\la \frac{\gb(\xb; \btheta_{k-1})}{\sqrt{m}},  \Zb_k^{-1}   \frac{\gb(\xb'; \btheta_{k-1})}{\sqrt{m}}\bigg\ra \notag \\
    &= \bigg\la \Bb_{\cX}^\top[\bpsi(\xb); \bepsilon_{\xb}], \bigg(\alpha \Ib+ \Bb_{\cX}^\top\sum_{j=1}^{N_k}[\bpsi(\xb_j); \bepsilon_{\xb_j}][\bpsi(\xb_j); \bepsilon_{\xb_j}]^\top \Bb_{\cX}\bigg)^{-1} \Bb_{\cX}^\top[\bpsi(\xb'); \bepsilon_{\xb'}]\bigg\ra\notag \\
    & = \bigg\la [\bpsi(\xb); \bepsilon_{\xb}], \bigg(\alpha \Ib+ \sum_{j=1}^{N_k}[\bpsi(\xb_j); \bepsilon_{\xb_j}][\bpsi(\xb_j); \bepsilon_{\xb_j}]^\top \bigg)^{-1} [\bpsi(\xb'); \bepsilon_{\xb'}]\bigg\ra\label{eq:112},
\end{align}
where the first equality holds due to Proposition~\ref{prop:neural}. 
To further bound \eqref{eq:112}, we give the matrix inverse an explicit expression. Denote $\Mb, \Nb, \Pb, \Qb$ as follows:
\begin{align}
    &\Mb = \sum_{j=1}^{N_k} \bpsi(\xb_j)\bepsilon_{\xb_j}^\top,\ \Nb = \sum_{j=1}^{N_k} \bepsilon_{\xb_j}\bepsilon_{\xb_j}^\top,\notag \\
    &\Pb = \Ab_k - \Mb (\Nb + \alpha \Ib)^{-1}\Mb^\top,\ \Qb = (\Nb + \alpha \Ib)^{-1} + (\Nb + \alpha \Ib)^{-1}\Mb^\top \Pb^{-1}\Mb(\Nb + \alpha \Ib)^{-1},\notag
\end{align}
then we have
\begin{align}
    \bigg(\alpha \Ib+ \sum_{j=1}^{N_k}[\bpsi(\xb_j); \bepsilon_{\xb_j}][\bpsi(\xb_j); \bepsilon_{\xb_j}]^\top \bigg)^{-1} &= \begin{pmatrix}
    \Ab_k, &\Mb\notag \\
    \Mb^\top, &\Nb + \alpha \Ib\notag
    \end{pmatrix}
    ^{-1}\notag \\
    & = \begin{pmatrix}
    \Pb^{-1}
    &- \Ab_k^{-1} \Mb\Pb^{-1} \\
    \Pb^{-1}\Mb^\top\Ab_k^{-1}, 
    &\Qb
    \end{pmatrix},\label{eq:112.5}
\end{align}
where the second equality holds due to the block matrix inverse formula. Then substituting \eqref{eq:112.5} into \eqref{eq:112} and considering the difference between $\big\la \gb(\xb; \btheta_{k-1}),  \Zb_k^{-1}   \gb(\xb'; \btheta_{k-1})\big\ra/m$ and $\big\la \bpsi(\xb), \Ab_k^{-1} \bpsi(\xb')\big\ra$, we have
\begin{align}
    &\big|\big\la \gb(\xb; \btheta_{k-1}),  \Zb_k^{-1}   \gb(\xb'; \btheta_{k-1})\big\ra/m - \big\la \bpsi(\xb), \Ab_k^{-1} \bpsi(\xb')\big\ra\big|\notag \\
    & = \big|\bpsi(\xb)^\top(\Pb^{-1} - \Ab_k^{-1})\bpsi(\xb') + \bepsilon_{\xb}^\top \Pb^{-1}\Mb^\top\Ab_k^{-1}\bpsi(\xb')\notag \\
    &\qquad + \bpsi(\xb)\Ab_k^{-1} \Mb\Pb^{-1} \bepsilon_{\xb'} + 
    \bepsilon_{\xb}^\top \Qb \bepsilon_{\xb'}^\top\big|\notag \\
    & \leq \|\bpsi(\xb')\|_2\|\bpsi(\xb)\|_2\|\Pb^{-1} - \Ab_k^{-1}\|_2 + \|\Qb\|_2\|\bepsilon_{\xb}\|_2\|\bepsilon_{\xb'}\|_2 \notag \\
    &\qquad + \|\Pb^{-1}\|_2\|\Mb\|_2\|\Ab_k^{-1}\|_2(\|\bepsilon_{\xb}\|_2 \|\bpsi(\xb')\|_2 + \|\bepsilon_{\xb'}\|_2 \|\bpsi(\xb)\|_2),\label{eq:112.6}
\end{align}
where we use triangle inequality in the inequality. To bound \eqref{eq:112.6}, we have the following inequalities:
\begin{align}
    \forall \xb \in \cX,\ \|\bpsi(\xb)\|_2 \leq \|\gb(\xb; \btheta_{k-1})\|_2/\sqrt{m} \leq C_{g,1} \sqrt{L},\label{eq:112.71}
\end{align}
where the first inequality holds since $\bpsi(\xb)$ is the truncation of $\gb(\xb; \btheta_{k-1})/\sqrt{m}$, the second one holds due to Lemma \ref{lemma:cao_boundgradient}. Use the fact that $\alpha \geq 2C_{g,1}N_k\bar\epsilon\sqrt{L}$, we also have
\begin{align}
    &\Ab_k \succeq \Pb \succeq \Ab_k - \|\Mb\|_2^2/\alpha \succeq \Ab_k - C_{g,1}^2N_k^2\bar\epsilon^2 L/\alpha \succeq \Ab_k -\alpha/2\cdot \Ib,\label{eq:112.72}
\end{align}
and
\begin{align}
    \|\Qb\|_2 \leq \|(\Nb + \alpha \Ib)^{-1}\|_2 + \|(\Nb + \alpha \Ib)^{-1}\|_2^2\|\Mb\|_2^2\|\Pb^{-1}\|_2 \leq \alpha^{-1} + \alpha^{-3} N_k^2C_{g,1}^2L\bar\epsilon^2 \leq 2/\alpha.\label{eq:112.73}
\end{align}
Lastly, we have
\begin{align}
    \|\Ab_k^{-1} - \Pb^{-1}\|_2 = \|\Pb^{-1}(\Pb - \Ab_k)\Ab_k^{-1}\|_2 &\leq \|\Ab_k^{-1}\|_2\|\Pb^{-1}\|_2\|\Ab_k - \Pb\|_2 \leq C_{g,1}^2N_k^2\bar\epsilon^2 L/\alpha^3, \label{eq:112.9}
\end{align}
where the bounds of $\|\Ab_k^{-1}\|_2$, $\|\Pb^{-1}\|_2$ and $\|\Ab_k - \Pb\|_2$ come from \eqref{eq:112.72} and \eqref{eq:112.73}. Finally, substituting \eqref{eq:112.71}, \eqref{eq:112.72}, \eqref{eq:112.73}, \eqref{eq:112.9} into \eqref{eq:112.6}, we can bound $I_1$ as
\begin{align}
    I_1 \leq C_{g,1}^4N_k^2\bar\epsilon^2 L^2/\alpha^3 + 4\bar\epsilon^2/\alpha^2N_k L  + 2\bar\epsilon^2/\alpha.\label{eq:988}
\end{align}

\noindent\textbf{To bound $I_2$: }
To bound $I_2$, we have
\begin{align}
    I_2 &\leq \bigg|\bigg\la \frac{\gb(\xb; \btheta_0)}{\sqrt{m}}, \bar \Zb_k^{-1}   \gb(\xb_i; \btheta_{0})/\sqrt{m}\bigg\ra - \bigg\la \frac{\gb(\xb; \btheta_{k-1})}{\sqrt{m}}, \bar \Zb_k^{-1}   \gb(\xb_i; \btheta_{0})/\sqrt{m}\bigg\ra\bigg|\notag \\
    &\qquad + \bigg|\bigg\la \frac{\gb(\xb; \btheta_{k-1})}{\sqrt{m}}, \bar \Zb_k^{-1}   \gb(\xb_i; \btheta_{0})/\sqrt{m}\bigg\ra - \bigg\la \frac{\gb(\xb; \btheta_{k-1})}{\sqrt{m}},  \Zb_k^{-1}   \gb(\xb_i; \btheta_{0})/\sqrt{m}\bigg\ra\bigg|\notag \\
    &\qquad + \bigg|\bigg\la \frac{\gb(\xb; \btheta_{k-1})}{\sqrt{m}},  \Zb_k^{-1}   \gb(\xb_i; \btheta_{0})/\sqrt{m}\bigg\ra - \bigg\la \frac{\gb(\xb; \btheta_{k-1})}{\sqrt{m}},  \Zb_k^{-1}   \gb(\xb_i; \btheta_{k-1})/\sqrt{m}\bigg\ra\bigg|\notag \\
    & \leq \|\gb(\xb; \btheta_0) - \gb(\xb; \btheta_{k-1})\|_2/\sqrt{m}\|\bar\Zb_k^{-1}\|_2 \|\gb(\xb_i; \btheta_0)\|_2/\sqrt{m} \notag \\
    &\qquad + \|\gb(\xb_i; \btheta_{k-1})\|_2/\sqrt{m}\|\bar \Zb_k^{-1} - \Zb_k^{-1}\|_2 \|\gb(\xb_i; \btheta_0)\|_2/\sqrt{m}\notag \\
    &\qquad + \|\gb(\xb_i; \btheta_{k-1})\|_2/\sqrt{m}\|\Zb_k^{-1}\|_2\|\gb(\xb; \btheta_0) - \gb(\xb; \btheta_{k-1})\|_2/\sqrt{m}\notag \\
    & \leq 2C_{g,1}C_{g,2}N_{k-1}^{1/6}\alpha^{-7/6}m^{-1/3}\sqrt{\log m}L^4 + C_{g,1}^2 L^5\cdot \bar C_4m^{-1/6}\sqrt{\log m}N_k^{7/6}\alpha^{-13/6},\label{eq:989}
\end{align}
where the first inequality holds due to triangle inequality, the third one holds due to Lemma \ref{lemma:cao_boundgradient} and the fact 
\begin{align}
    \|\bar \Zb_k^{-1} - \Zb_k^{-1}\|_2 \leq \|\bar \Zb_k^{-1}\|_2\|\Zb_k^{-1}\|_2\|\bar \Zb_k - \Zb_k\|_2 \leq \bar C_4 L^4m^{-1/6}\sqrt{\log m}N_k^{7/6}\alpha^{-13/6},
\end{align}
where the last inequality holds due to the fact $\bar\Zb_k, \Zb_k \succeq \alpha \Ib$ and Lemma \ref{lemma:newboundz}. 

\noindent\textbf{Final bound: }
Substituting \eqref{eq:988} and \eqref{eq:989} into \eqref{eq:111} ends our proof. 
\end{proof}
Then we begin our proof, which is a direct combination of Lemma \ref{lemma:linearize} and Lemma \ref{lemma:neural_truncate}. 
\begin{proof}[Proof of Lemma \ref{lemma:neural_main}]
Adding \eqref{lemma:linearize_0} with \eqref{lemma:neural_truncate_0} finishes our proof. 
\end{proof}

\subsubsection{Proof of Lemma \ref{lemma:lowandh}}\label{proof:lemma:lowandh}
\begin{proof}[Proof of Lemma \ref{lemma:lowandh}]
First, by Lemma \ref{lemma:equal}, we know that there exists $\btheta_\star \in \RR^p$ such that
\begin{align}
    &h(\xb) = \la \gb(\xb; \btheta_0), \btheta_\star - \btheta_0\ra,\ \sqrt{m}\|\btheta_\star - \btheta_0\|_2 \leq \sqrt{2}S. \label{yy1}
\end{align}
By Proposition~\ref{prop:neural}, we have
\begin{align}
    \bpsi(\xb) = \sum_{j=1}^{d} u_{\xb, j}e_j \eb_j,\ \gb(\xb; \btheta_{k-1})/\sqrt{m} = \sum_{j=1}^{|\cX|} u_{\xb, j}e_j \bbb_j = \Bb_{\cX}^\top[\bpsi(\xb); \bepsilon_{\xb}],\ \|\bepsilon_{\xb}\|_2 \leq \bar\epsilon.
\end{align}
Therefore, suppose $\Bb_{\cX}\sqrt{m}(\btheta_\star - \btheta_0) = [\rb; \tb], \rb \in \RR^{d}, \tb \in \RR^{|\cX| - d}$, set $\widehat\btheta_\star = \rb$, we have $\|\widehat\btheta_\star\|_2 \leq \sqrt{m}\|\btheta_\star - \btheta_0\|_2 \leq \sqrt{2}S$, and
\begin{align}
    |\la \bpsi(\xb), \widehat\btheta_\star\ra - \la \gb(\xb; \btheta_{k-1}), \btheta_\star - \btheta_0\ra| &= |\bpsi(\xb)^\top \rb - [\bpsi(\xb)^\top, \bepsilon_{\xb}^\top]\Bb_{\cX}\sqrt{m}(\btheta_\star - \btheta_0)|\notag \\
    & = |\bpsi(\xb)^\top \rb - \bpsi(\xb)^\top \rb - \bepsilon_{\xb}^\top\tb|\notag \\
    & \leq \bar\epsilon \|\sqrt{m}(\btheta_\star - \btheta_0)\|_2\notag \\
    & \leq \sqrt{2}S\bar\epsilon, 
\end{align}
where the last inequality holds due to \eqref{yy1}. Therefore, we have
\begin{align}
    &|h(\xb) - \la \bphi(\xb), \widehat \btheta_\star\ra|\notag \\
    &= |\la \gb(\xb; \btheta_0), \btheta_\star - \btheta_0\ra - \la \bphi(\xb), \widehat \btheta_\star\ra|\notag \\
    & \leq \|\btheta_\star - \btheta_0\|_2\|\gb(\xb; \btheta_0) - \gb(\xb; \btheta_{k-1})\|_2 + |\la \bpsi(\xb), \widehat\btheta_\star\ra - \la \gb(\xb; \btheta_{k-1}), \btheta_\star - \btheta_0\ra|\notag \\
    & \leq \sqrt{2}SC_{g,2}N_{k-1}^{1/6}\alpha^{-1/6}m^{-1/3}\sqrt{\log m}L^{7/2} + \sqrt{2}S\bar\epsilon\notag \\
    & \leq 2S\bar\epsilon,\notag
\end{align}
where at the last line we use the fact that $m = \text{poly}(L, N_{k-1}, \alpha)$. 
\end{proof}

\subsubsection{Proof of Lemma \ref{lemma:boundd}}\label{proof:lemma:boundd}
\begin{proof}[Proof of Lemma \ref{lemma:boundd}]
First we bound $|\la \gb(\xb; \btheta_0), \gb(\xb';\btheta_0)\ra/m - \la \gb(\xb; \btheta_0), \gb(\xb';\btheta_0)\ra/m|$. We have
\begin{align}
    &|\la \gb(\xb; \btheta_0), \gb(\xb';\btheta_0)\ra/m - \la \gb(\xb; \btheta_{k-1}), \gb(\xb';\btheta_{k-1})\ra/m|\notag \\
    & \leq |\la \gb(\xb; \btheta_0), \gb(\xb';\btheta_0)\ra/m - \la \gb(\xb; \btheta_0), \gb(\xb';\btheta_{k-1})\ra/m| \notag \\
    &\qquad + |\la \gb(\xb; \btheta_0), \gb(\xb';\btheta_{k-1})\ra/m - \la \gb(\xb; \btheta_0), \gb(\xb';\btheta_{k-1})\ra/m|\notag \\
    & \leq 2C_{g,1}C_{g,2}N_{k-1}^{1/6}\alpha^{-1/6}m^{-1/3}\sqrt{\log m}L^4,
\end{align}
where the first inequality holds due to triangle inequality, the second one holds due to \eqref{gradientbound}. Therefore, when $m = \text{poly}(|\cX|, L, \lambda_0^{-1}, \log(|\cX|/\delta_k), N_k, \alpha, \bar\epsilon^{-1})$, we have
\begin{align}
    &|\la \gb(\xb; \btheta_{k-1}), \gb(\xb';\btheta_{k-1})\ra/m - \Hb(\xb, \xb')| \notag \\
    &\leq  |\la \gb(\xb; \btheta_{0}), \gb(\xb';\btheta_{0})\ra/m - \Hb(\xb, \xb')|\notag \\
    &\qquad +|\la \gb(\xb; \btheta_0), \gb(\xb';\btheta_0)\ra/m - \la \gb(\xb; \btheta_{k-1}), \gb(\xb';\btheta_{k-1})\ra/m|\notag \\
    & \leq |\la \gb(\xb; \btheta_{0}), \gb(\xb';\btheta_{0})\ra/m - \Hb(\xb, \xb')| + 2C_{g,1}C_{g,2}N_{k-1}^{1/6}\alpha^{-1/6}m^{-1/3}\sqrt{\log m}L^4\notag \\
    & \leq \bar\epsilon^2/(2|\cX|^4).
\end{align}
Next, by the definition of $d_{\eff}(\epsilon)$, we know that
\begin{align}
    \sum_{i= d_{\eff}(\bar\epsilon^2/|\cX|)+1}^{|\cX|} \lambda_i(\Hb) \leq \bar\epsilon^2/|\cX|, 
\end{align}
which suggests that
\begin{align}
    \sum_{i= d_{\eff}(\bar\epsilon^2/|\cX|)+1}^{|\cX|} \lambda_i(\Gb\Gb^\top) &\leq \sum_{i= d_{\eff}(\bar\epsilon^2/|\cX|)+1}^{|\cX|} \lambda_i(\Hb) + \sum_{i= d_{\eff}(\bar\epsilon^2/|\cX|)+1}^{|\cX|} \|\Gb\Gb^\top - \Hb\|_2\notag \\
    & \leq \bar\epsilon^2/(2|\cX|) + |\cX|^3\cdot\max_{\xb, \xb'} |\la \gb(\xb; \btheta_{k-1}), \gb(\xb';\btheta_{k-1})\ra/m - \Hb(\xb, \xb')|\notag \\
    & \leq \bar\epsilon^2/(2|\cX|) + \bar\epsilon^2/(2|\cX|) \notag \\
    & = \bar\epsilon^2/|\cX|.
\end{align}
Finally, we have
\begin{align}
    \sum_{i= d_{\eff}(\bar\epsilon^2/|\cX|)+1}^{|\cX|} \lambda_i(\Gb) = \sum_{i= d_{\eff}(\bar\epsilon^2/|\cX|)+1}^{|\cX|} \sqrt{\lambda_i(\Gb\Gb^\top)} \leq \sqrt{|\cX|\sum_{i= d_{\eff}(\bar\epsilon^2/|\cX|)+1}^{|\cX|} \lambda_i(\Gb\Gb^\top)} \leq \bar\epsilon. 
\end{align}
\end{proof}

\section{Implementation details}
\label{app:experiment}
We provide further details of experiments presented in \cref{sec:experiment}. We ran all experiments on a Xeon Gold 6130 CPU, except for {\tt NeuralEmbedding} where we used an Nvidia GeForce RTX2080Ti GPU. All the datasets we are using are either synthetic or publicly available and do not contain personally identifiable information or offensive consent.  

\subsection{Experimental setups}
\label{sec_app:exper_setup}

All algorithms are run at confidence level of $\delta = 0.05$, and a misspecification level of $\epsilon = 0.1$. For each algorithm that requires computing a design $\lambda$ (or the associated value $\tau$ of the minimax optimization problem), we employed a Frank-Wolfe algorithm and set the step-size and convergence rate in the same way as in \citep{fiez2019sequential}. 
Since our algorithms are all elimination-styled, the (empirical) sample complexity is calculated as the number of samples pulled when the set of uneliminated arms only consists of $\epsilon$-optimal arms. All algorithms were forced to stop when the number of pulls reaches $10^7$ (and such a case was reported as a failure). We choose $10^7$ as the stopping criteria since continuing running an algorithm after pulling a large number of arms ($10^7$) is likely to raise memory issues on the running machine. For a fair comparison, each experiment is repeated 50 times and we report sample complexity results based on successful instances runs.

\KE was instantiated with Gaussian kernel $\cK(\bx_i, \bx_j) = \exp (-\gamma\Vert \bx_i - \bx_j \Vert_2^2)$ over action set $\cX$. To decide the value of $\gamma$ in {\tt KernelEmbedding }as well as $C$ (in \cref{sec:high_dim_linear}) in \KE and {\tt LinearEmbedding}, for each dataset, we ran the method with several different seeds, used a grid search over $\{10^{-5}, 10^{-4}, 10^{-3}, 10^{-2}, 10^{-1}, 10^{0}, 10^{1}, 10^2, 10^3 \}$, and monitored the interested metrics (empirical sample complexity and the success rate). 
For \KE and {\tt LinearEmbedding}, we implemented \cref{alg:active_elim_adaptive} for synthetic datasets; 
for Yahoo and MNIST datasets, due to the difficulty of the data, e.g., nonlinearity, we found that \cref{alg:active_elim_adaptive} with adaptive embedding might compress arms into undesired inseparable subspaces in earlier rounds, which worsens the empirical performance. As a result, we implemented a slightly modified version of \cref{alg:active_elim_adaptive} (i.e., \cref{alg:active_elim}; see details in \cref{rm:fixed_embedding}) for Yahoo and MNIST datasets for better empirical performance (sample complexity and success rate). To save computation, we use $4(1+\zeta)d$ as an approximation of $g(d,\zeta)$ when computing $d_k$ for \KE and \LE (recall that we always have $g(d,\zeta) \leq 4(1+\zeta)d$).\footnote{One can also determine $d_k$ in a binary search manner if monotonicity of (an upper bound of) $\gamma(d)$ can be guaranteed, e.g., see examples in the proof of \cref{thm:kernel_elim}.}

The model used in {\tt Neural Embedding }is a three-layer (i.e., two hidden layers) fully connected neural network, where each hidden layer has 128 nodes.
We used a Rectified Linear Unit (ReLU) as our activation function. To overcome overfitting on real datasets, we add one dropout layer with $50\%$ dropout rate. 
The learning rate is set as $10^{-4}$ and the maximum training iteration number is set as $6,000$. A grid search over $\{10^{-5}, 10^{-4}, 10^{-3},10^{-2}, 10^{-1}\}$ is used to determine the $\bar \epsilon$ parameter (in \cref{alg:neural}).

\subsection{Detailed descriptions of datasets} 

\textbf{Synthetic dataset with linear rewards.} We tested all methods with varying the number of arms ($K$) and a fixed dimension ($D = 20$). 
To construct a set of arms $\cX$, we first randomly generate $\bx_1 = \btheta_\star$ so that $\ang*{\btheta_\star, \bx_1} = 0.8$; we then generate a $\bx_2$ such that $\ang*{\btheta_\star, \bx_2} \approx 0.4$. After generating the two principal arms $\bx_1$ and $\bx_2$, we constructed an action set $\cX$ by adding randomly perturbated arms, i.e., $\cX = \curly*{\bx_1, \bx_2, \bx_1 \oplus \curly{\eta_i \be_{i_j}}_{i=1}^{\frac{K}{2}-1}, \bx_2 \oplus \curly{\eta_i \be_{i_j}}_{i=1}^{\frac{K}{2}-1}}$, where $i_j$ is randomly chosen from $[D]$ and $\eta_i \sim \mathcal{N}(0, (10^{-5})^2)$. We ensure $\spn(\cX) = \R^D$ in the construction. We used Bernoulli reward with $\bx^\top \btheta_\star$ success probability in this experiment.

\textbf{Synthetic dataset with nonlinear rewards.} We tested all algorithms with increasing dimensions ($D$) and a fixed number of arms ($K = 200$). The reward of an arm is set as the $2$-norm of its feature representation. We randomly generate $\bx_1$ and $\bx_2$ such that $h(\bx_1) = 0.8$ and $h(\bx_2) = 0.4$. We construction our action set $\cX$ by adding randomly perturbated arms, i.e., $\cX = \curly*{\bx_1, \bx_2, \bx_1 \oplus \curly{\eta_i \be_{i_j}}_{i=1}^{\frac{K}{2}-1}, \bx_2 \oplus \curly{\eta_i \be_{i_j}}_{i=1}^{\frac{K}{2}-1}}$, where $i_j$ is randomly chosen from $[D]$ and $\eta_i \sim \mathcal{N}(0, (10^{-5})^2)$. We ensure $\spn(\cX) = \R^D$ in the construction. We used Bernoulli reward with $h(\bx)$ success probability in this experiment.

\textbf{Yahoo dataset.} We used Yahoo! User Click Log Dataset R6A in this experiment.\footnote{\url{https://webscope.sandbox.yahoo.com}} 
The dataset contains the users' click-through records from the Today news module on Yahoo! front page between May 1st. 2009 and May 10th. 2009. 
Each user click log record consists of 6 article features and 6 user features, along with a binary variable stating whether or not a user clicked on the article. 
To process the data, we considered the records collected in the 1st day and randomly selected 200 records from it.
We construct an arm set $\bX \in \mathbb{R}^{200 \times 36}$ by taking the flattened outer product of the user and the corresponding article feature vector.
And in our further examination, we found that the rank of $\bX$ is 35 instead of 36 (full rank). 
To ensure full rank, we preprocess $\bX$ by projecting it into a lower-dimensional space $\widetilde{\bX} \in \mathbb{R}^{200 \times 35}$. The reward of each arm is determined as in \cref{sec:experiment}. Noise in the observed rewards was generated from a standard normal distribution.

 \textbf{MNIST dataset.} The dataset is described as in \cref{sec:experiment}. Noise in the observed rewards was generated from a standard normal distribution.

\subsection{Empirical effective dimension $d_k$}

For implementations with a fixed embedding (i.e., Algorithm \ref{alg:active_elim}), $d_k$ is fixed over all iterations. 
For implementations with an adaptive embedding (i.e., Algorithm \ref{alg:active_elim_adaptive} and \ref{alg:neural}), $d_k$ is increasing with iteration index $k$ since we decrease tolerance to misspecification levels in later iterations. 
For experimental results of synthetic dataset in Figure \ref{fig:synthetic}, due to the simplicity of the data, {\tt NeuralEmbdding} is able to compress arms in $\mathbb{R}^2$ and completed the elimination in one round, and {\tt Alg.1-KernelEmbedding} and {\tt Alg.1-LinearEmbedding} are able to compress arms in $\mathbb{R}^2$ and completed the elimination in two rounds.
For experimental results of real-world datasets in Figure \ref{fig:real-world data}, our {\tt KernelEmbedding} and {\tt LinearEmbedding} are implemented with fxied embedding (as discussed in Appendix \ref{app:experiment}). For {\tt NeuralEmbedding}, at each iteration $k$, we calculated the averaged empirical value of $d_k$; we also use $s_k$ to denote the percentage of runs that have successfully completed the elimination process at round $k$. We summarized the pair value $(d_k, s_k)$ in \cref{table:effective_dimension}. 

\begin{table}[H]
  \caption{Empirical effective dimension and success rate on real-world datasets}
  \label{table:effective_dimension}
  \centering
  \begin{tabular}{lccccc}
    \toprule
            {\footnotesize \NE  }& {MNIST}   & {Yahoo}  \\ \midrule
      $(d_1, s_1)$ & $(190.22, 82\%)$ & $(72.70, 86\%)$ \\
     $(d_2, s_2)$ & $(192.11, 98\%)$ & $(167.57, 100\%)$ \\
    \bottomrule
  \end{tabular}
\end{table}

\end{document}